\documentclass[a4paper,11pt]{amsart}%
\usepackage{hyperref}
\usepackage{pdfsync}
\usepackage{dsfont}
\usepackage{color}
\usepackage{esint}
\usepackage{amsthm}
\usepackage{enumerate}
\usepackage{amsfonts}
\usepackage{amssymb}%
\usepackage{mathtools}
\usepackage{braket}
\usepackage{bm}
\usepackage{amsmath}
\usepackage{graphicx}
\usepackage{caption}
\usepackage{subcaption}
\usepackage{fullpage}
\usepackage{longtable}
\usepackage{mathrsfs}
\usepackage{mathtools}
\numberwithin{equation}{section}
\graphicspath{ {./images/} }

\usepackage{amsfonts}
\usepackage{graphicx}%
\providecommand{\U}[1]{\protect\rule{.1in}{.1in}}
\newtheorem{theorem}{Theorem}

\newtheorem{assumption}{Assumption}

\newtheorem{definition}[theorem]{Definition}

\newtheorem{lemma}{Lemma}

\newtheorem{proposition}{Proposition}
\newtheorem{remark}{Remark}

\newcommand{\veps}{\varepsilon}

\renewcommand{\H}{\mathcal{H}}

\newcommand{\R}{\mathbb{R}}
\newcommand{\Y}{\mathcal{Y}}

\DeclareMathOperator{\Prob}{\mathbb{P}}
\newcommand{\vol}{\text{vol}_{\mathcal{M}}}
\newcommand{\M}{\mathcal{M}}

\newcommand{\length}{length}

\definecolor{mygreen}{rgb}{0.1,0.75,0.2}

\newcommand{\nc}{\normalcolor}
\newcommand{\E}{\mathbb{E}}

\title{Local Regularization of Noisy Point Clouds: Improved Global Geometric Estimates and Data Analysis}
\author{Nicol\'{a}s Garc\'{i}a Trillos, Daniel Sanz-Alonso, and Ruiyi Yang}
\begin{document}
\maketitle	

\begin{abstract}
Several data analysis techniques employ similarity relationships between data points to uncover the intrinsic dimension and geometric structure of the underlying data-generating mechanism. In this paper we work under the model assumption that the data is made of random perturbations of feature vectors lying on a low-dimensional manifold. We study two questions: how to define the similarity relationship over noisy data points, and what is the resulting impact of the choice of similarity in the extraction of global geometric information from the underlying manifold. We provide concrete mathematical evidence that using a local regularization of the noisy data to define the similarity improves the approximation of the hidden Euclidean distance between unperturbed points. Furthermore, graph-based objects constructed with the locally regularized similarity function satisfy better error bounds in their recovery of global geometric ones. Our theory is supported by numerical experiments that demonstrate that the gain in geometric understanding facilitated by local regularization translates into a gain in classification accuracy in simulated and real data.
\end{abstract}

\section{Introduction}
\subsection{Aim and Background} Several techniques for the analysis of high dimensional data exploit the fact that data-generating mechanisms can be often described by few degrees of freedom. The focus of this paper is on graph-based methods that employ similarity relationships between points to uncover the low intrinsic dimensionality and geometric structure of a data set. Graph-based learning provides a well-balanced compromise between accuracy and interpretability, and is popular in a variety of unsupervised and semi-supervised tasks \cite{zhu2005semi,von2007tutorial}. These methods have been analyzed in the idealized setting where the data is sampled from a low-dimensional manifold and similarities are computed using the ambient Euclidean distance or the geodesic distance, see e.g.   \cite{coifman2006diffusion,singer2006graph,burago2013graph,SpecRatesTrillos}. The manifold setting is truthful in spirit to the presupposition that data arising from structured systems may be described by few degrees of freedom, but it is not so in that the data are typically noisy. The aim of this paper is to provide new mathematical theory under the more general model assumption that the data consists of random perturbations of low-dimensional features lying on a manifold. 

By relaxing the manifold assumption we bring forward two fundamental questions that are at the heart of graph-based learning which have not been accounted for by previous theory. First, how should one define the inter-point similarities between noisy data points in order to recover Euclidean distances between unperturbed data-points more faithfully? Second, is it possible to recover global geometric features of the manifold from suitably-defined similarities between noisy data? We will show by rigorous mathematical reasoning that: 
\begin{enumerate}[(i)]
\item  Denoising inter-point distances leads to an improved approximation of the hidden Euclidean distance between unperturbed points. We illustrate this general idea by analyzing a simple, easily-computable similarity defined in terms of a local-regularization of the noisy data set. 
\item Graph-based objects defined via locally regularized similarities can be guaranteed to satisfy improved error bounds in the recovery of global geometric properties. We illustrate this general idea by showing the spectral approximation of an unnormalized $\veps$-graph Laplacian to a Laplace operator defined on the underlying manifold.
\end{enumerate}

In addition to giving theoretical support for the regularization of noisy point clouds, we study the practial use of local regularization in classification problems. Our analytically tractable local-regularization depends on a parameter that modulates the amount of localization, and our analysis suggests the scaling of the localization parameter in terms of the level of noise in the data. In our numerical experiments we show that in semi-supervised classification problems this parameter may be chosen by cross-validation, ultimately  producing classification rules with improved accuracy.  Finally, we propose two alternative denoising methods with similar empirical performance that are sometimes easier to implement with. In short, the improved geometric understanding facilitated by (local) regularization translates into improved graph-based data analysis, and the results seem to be robust to the choice of methodology.

\subsection{Framework}
We assume a data model 
\begin{equation}\label{eq:datamodel}
y_i = x_i + z_i, 
\end{equation}
where the \emph{unobserved} points $x_i$ are sampled from an unknown $m$-dimensional manifold $\M,$ the vectors $z_i\in \R^d$ represent noise, and $\Y_n = \{y_1, \dots, y_n \} \subseteq \R^d$ is the observed data. Further geometric and probabilistic structure will be imposed to prove our main results --see Assumptions \ref{asp:1} and \ref{asp:2} below.  Our analysis is motivated by the case, often found in applications,  where the number $n$ of data points and the ambient space dimension $d$ are large, but the underlying intrinsic dimension $m$ is small or moderate. Thus, the data-generating mechanism is described (up to a noisy perturbation) by $m\ll d$ degrees of freedom. We aim to uncover geometric properties of the underlying manifold $\M$ from the observed data $\Y_n$ by using \emph{similarity} graphs. The set of vertices of these graphs will be identified with the set $[n]:=\{1,\ldots,n\}$ ---so that the $i$-th node corresponds to the $i$-th data point--- and the weight $W(i,j)$ between the $i$-th and $j$-th data-point will be defined in terms of a similarity function $\delta: [n] \times [n] \to [0,\infty)$. 

The first question that we consider is how to choose the similarity function so that $\delta(i,j)$ approximates the hidden Euclidean distance $\delta_{\mathcal{X}_n}(i,j):= |x_i - x_j|.$ 
Full knowledge of the Euclidean distance between the latent variables $x_i$ would allow to recover, in the large $n$ limit, global geometric features of the underlying manifold. This motivates the idea of denoising the observed point cloud $\Y_n$ to approximate the hidden similarity function $\delta_{\mathcal{X}_n}.$  Here we study a family of similarity functions based on the Euclidean distance between local averages of points in $\Y_n$. We define a denoised data set $\bar{\Y}_n = \{ \overline{y}_1, \ldots, \overline{y}_n\}$ by locally averaging the original data set, and we then define an associated similarity function
\[  \delta_{\bar \Y_n}(i, j):= |\overline{y}_i - \overline{y}_j|. \]
In its simplest form, $\overline{y}_i$ is defined by averaging all points in $\Y_n $ that are inside the ball of radius $r>0$ centered around $y_i, $ that is, 
\begin{equation}\label{eq:localavg}
\overline{y}_i := \frac{1}{N_i} \sum_{j \in \mathcal{A}_i} y_j,
\end{equation}
where $N_i$ is the cardinality of $\mathcal{A}_i := \{ j \in [n] \: : \:  y_j \in B(y_i , r) \}.$
Note that $\bar{\Y}_n$ (and the associated similarity function $\delta_{\bar{\Y}_n}$) depends on $r,$ but we do not include said dependence in our notation for simplicity. Other possible local and non-local averaging approaches may be considered. We will only analyze the choice made in \eqref{eq:localavg} and we will explore other constructions numerically. Introducing the notation
\[ \delta_{\mathcal{X}_n}(i,j)=|x_i-x_j|,   \quad
\delta_{\Y_n}(i,j)=|y_i-y_j|, \]
the first question that we study may be formalized as understanding when, and to what extent, the similarity function $\delta_{\bar{\Y}_n}$ is a better approximation than $\delta_{\Y_n}$ (the standard choice) to the hidden similarity function $\delta_{\mathcal{X}_n}.$ An answer is given in Theorem \ref{thm1} below. 

The second question that we investigate is how an improvement in the approximation of the hidden similarity affects the approximation of global geometric quantities of the underlying manifold. Specifically, we study how the  spectral convergence of graph-Laplacians constructed with noisy data may be improved by local regularization of the point cloud. For concreteness, our theoretical analysis is focused on $\veps$-graphs and unnormalized graph-Laplacians, but we expect our results to generalize to other graphs and graph-Laplacians --evidence to support this claim will be given through numerical experiments.  We now summarize  the necessary background to formalize this question. For a given similarity $\delta : [n] \times [n] \rightarrow [0,\infty)$ and a parameter $\veps>0$,  we define a weighted graph $\Gamma_{\delta,\veps}= ([n], W)$ by setting the weight between the $i$-th and $j$-th node to be
\begin{equation} \label{eq:weights}
W(i,j):= \frac{2(m+2)}{ \alpha_m\veps^{m+2}n}\mathds{1}\{ \delta(i, j)< \veps\}, 
\end{equation}
where $\alpha_m$ is the volume of the $m$-dimensional Euclidean unit ball.  Associated to the graph $\Gamma_{\delta,\veps}$ we define the unnormalized graph Laplacian matrix
\begin{equation}\label{eq:graphlaplacian}
\Delta_{\delta,\veps} := D - W \in \R^{n\times n}, 
\end{equation}  
where $D$ is a diagonal matrix with diagonal entries 
\[ D(i,i):= \sum_{j=1}^n W(i,j). \]
For the rest of the paper we shall denote $\Gamma_{\mathcal{X}_n,\veps}:= \Gamma_{\delta_{\mathcal{X}_n},\veps}$ and $\Delta_{\mathcal{X}_n,\veps}: = \Delta_{\delta_{\mathcal{X}_n},\veps}.$ We use analogous notation for $\Y_n$ and $\bar{\Y}_n$. The second question that we consider may be formalized as understanding when, and to what extent, $\Delta_{\bar{\Y}_n}$ provides a better approximation (in the spectral sense) than $\Delta_{\Y_n}$ to a Laplace operator on the manifold $\M.$ An answer is given in Theorem \ref{thm2} below.

\subsection{Main results}
\label{sec:main}
In this subsection we state our main theoretical results. We first impose some geometric conditions on the underlying manifold $\M.$

\begin{assumption} \label{asp:1}
$\M$ is a smooth, oriented, compact manifold with no boundary and intrinsic dimension $m$, embedded in $\R^d.$  Moreover, $\M$ has injectivity radius $i_0$, maximum of the absolute value of sectional curvature $K,$ and reach $R.$
\end{assumption}
Loosely speaking, the injectivity radius determines the range of the exponential map (which will be an important tool in our analysis and will be reviewed in the next section) and the sectional curvature controls the metric distortion by the exponential map, and thereby its Jacobian. The reach $R$ can be thought of as an (inverse) conditioning number of the manifold and controls its second fundamental form; it can also be interpreted as a measure of extrinsic curvature. The significance of these geometric quantities and their role in our analysis will be further discussed in Section \ref{sec:them1}. 

Next we impose further probabilistic structure into the data model \eqref{eq:datamodel}. We assume that the pairs $(x_i,z_i)$ are i.i.d. samples of the random vector $(X,Z) \sim {\boldsymbol{\mu}} \in \mathcal{P}(\mathcal{M} \times \mathbb{R}^d)$. Let $\mu$ and $\mu_x$ be, respectively, the marginal distribution of $X$ and the conditional distribution of $Z$ given $X=x$. We assume that $\mu$ is absolutely continuous with respect to the Riemannian volume form of $\mathcal{M}$ with density $p(x)$, i.e.,
\begin{equation}\label{eq:pdensity}
d\mu(x)=p(x) d\text{vol}_{\mathcal{M}}(x).
\end{equation}
Furthermore, we assume that $\mu_x$ is supported on $T_x \mathcal{M}^{\perp}$ (the orthogonal complement of the tangent space $T_x \mathcal{M})$ and that it is absolutely continuous with respect to the $(d-m)$-dimensional Hausdorff measure $\mathcal{H}^{d-m}$ restricted to $T_x \M^\perp$ with density $p(z|x)$, i.e.,
\[  d \mu_{x}(z) =  p(z|x) d\H^{d-m}(z).  \]
To ease the notation we will write $dz$ instead of $\H^{d-m}(dz)$. We make the following assumptions on these densities. 
\begin{assumption} \label{asp:2}
It holds that: 
\begin{enumerate}[(i)]
\item The density $p(x)$ is of class $C^2(\M)$ and is bounded above and below by positive constants:
\[  0 < p_{min} \leq p(x) \leq  p_{max}, \quad \forall x \in \M. \]
\item  For all $x \in \M$, 
\[ \int z p(z|x) dz  = 0.   \]
Moreover, there is $\sigma <R$ such that  $p(z|x) = 0$ for all $z$ with $|z| \geq \sigma.$ 
\end{enumerate}
\end{assumption}
Note that the assumption on $p(z|x)$ ensures that the noise is centered and bounded by a constant $\sigma. $

 In our first main theorem we study the approximation of the similarity function $\delta_{\mathcal{X}_n}$ by $\delta_{\bar{\Y}_n}$. We consider points $x_i$ and $x_j$ that are close with respect to the geodesic distance $d_\M$ on the manifold, and show that local regularization improves the approximation of the hidden similarity provided that $n$ is large and the noise level $\sigma$ is small. The local regularity parameter $r$ needs to be suitably scaled with $\sigma.$ We make the following standing assumption linking both parameters; we refer to Remark \ref{rmk:betterdist} below for a discussion on the optimal scaling of $r$ with $\sigma$, and to our numerical experiments for practical guidelines.
\begin{assumption} \label{asp:3}
The localization parameter $r$ and the noise level $\sigma$ satisfy 
\begin{equation}
\sigma \leq \frac{R}{16m}, \; r\leq \operatorname{min} \left\{ i_0, \frac{1}{\sqrt{K}}, \sqrt{\frac{\alpha_m}{2CmK}},\sqrt{\frac{R}{32}}  \right\}, \; and \; \sigma\leq \frac{1}{3}r,
\end{equation}
where $C$ is a universal constant, $\alpha_m$ denotes the volume of the Euclidean unit ball in $\R^m,$ and $i_0,$ $R,$ and $K$ are as in Assumption \ref{asp:1}.
\end{assumption}
In words, Assumption \ref{asp:3} requires both $r$ and $\sigma$ to be sufficiently small, and $r$ to be larger than $\sigma$.

Now we are ready to state the first main result.
\begin{theorem} 
\label{thm1}
Under Assumptions \ref{asp:1}, \ref{asp:2} and \ref{asp:3}, with  probability at least  $1-n4e^{-cnr^{\operatorname{max}\{2m,m+4 \}}}$, for all $x_i$ and $x_j$ with $d_{\mathcal{M}}(x_i,x_j)\leq r$ we have 
\begin{align} \label{eq:distbound}
    \big|\delta_{\mathcal{X}_n}(i,j)-\delta_{\bar{\Y}_n}(i,j) \big|\leq  C_{\mathcal{M}} \left( r^3+r\sigma+\frac{\sigma^2}{r} \right),
\end{align} 
where $c=\operatorname{min}\left\{ \frac{\alpha_m^2p_{min}^2}{4^{m+2}},  \frac{1}{16}\right\} $ and $C_{\mathcal{M}}$ is a constant depending on $m,K,R$, a uniform bound on the change in second fundamental form of $\M$, and on the regularity of the density $p.$ 
\end{theorem}
\begin{remark} \label{rmk:betterdist}
Theorem \ref{thm1} gives concrete evidence of the importance of the choice of similarity function.
For the usual Euclidean distance between observed data, $\delta_{\Y_n},$ one can only guarantee that
\[ \big| \delta_{\mathcal{X}_n}(i,j)-\delta_{\Y_n}(i,j) \big| \leq 2\sigma,  \]
which follows from 
\[ \big| |  x_i - x_j | - | y_i - y_j  |\big| \leq |z_i -z_j| \leq 2 \sigma. \]
However, if we choose $r\propto\sigma^{1/2}$, then the error in \eqref{eq:distbound} is of order $\sigma^{3/2}$, which is a considerably smaller quantity in the small noise limit. 
\end{remark}

\begin{remark}\label{rem:choiceofsigma}
It will become evident from our analysis that for points $x_i, x_j$ that are sufficiently close, the quantity $\big| |\overline y_i - \overline y_j |- |x_i-x_j| \big| $ is much smaller than the terms  $|\overline{y}_i-x_i|.$ This is an important observation, since our purpose is not to estimate the location of the point $x_i$ using $\overline y_i$, but rather to estimate \textit{distances}. In other words, our interest is in the intrinsic geometry of the point cloud and not in its actual representation. 
\end{remark}
Our second main result translates the local similarity bound from Theorem \ref{thm1} into a global geometric result concerning the spectral convergence of the graph Laplacian to the Laplace operator formally defined by
\begin{equation}\label{eq:laplacianmanifold}
\Delta_\M f = - \frac{1}{p} \text{div}\bigl( p^2 \nabla f \bigr),
\end{equation}
where $\text{div}$ and $\nabla$ denote the divergence and gradient operators on the manifold and $p$ is the sampling density of the hidden point cloud $\mathcal{X}_n$, as introduced in Equation \eqref{eq:pdensity}. It is intuitively clear that the spectral approximation of the discrete graph-Laplacian to the continuum operator $\Delta_\M$ necessarily rests upon having a sufficient number of samples from $\mu$ (defined in \eqref{eq:pdensity}). In other words, the empirical measure  $\mu_n = \frac{1}{n}\sum_{i=1}^n \delta_{x_i} $ needs to be close to $\mu,$ the sampling density of the hidden data set. We characterize the closeness between $\mu_n$ and $\mu$ by the $\infty$-OT transport distance, defined as     
\[d_{\infty}(\mu_n,\mu) := \underset{T:T_{\sharp} \mu =\mu_n}{\operatorname{min}} \underset{x\in\mathcal{M}}{\operatorname{esssup}} \quad d_{\mathcal{M}} \bigl(x,T(x)\bigr), \]
where $T_{\sharp}\mu$ denotes the push-forward of $\mu$ by $T$, that is,  $T_{\sharp}\mu =\mu\bigl(T^{-1}(U)\bigr)$ for any Borel subset $U$ of $\mathcal{M}$. In \cite{SpecRatesTrillos} it is shown that for every $\beta>1$, with probability at least $1- C_{\beta, \M}n^{-\beta} $, 
\[ d_\infty(\mu_n, \mu) \leq C_{\M} \frac{\log(n)^{p_m}}{n^{1/m}},  \]
where $p_m =3/4$ if $m=2$ and $p_m= 1/m$ for $m \geq 3$. This is the high probability scaling of $d_\infty(\mu_n, \mu)$ in terms of $n$.

We introduce some notation before stating our second main result. Let $\lambda_{\ell}(\Gamma_{\delta,\veps})$ be the $\ell$-th eigenvalue of the unnormalized graph-Laplacian $\Delta_{\delta,\veps}$ defined in Equation \eqref{eq:graphlaplacian}, and let  $\lambda_{\ell}(\mathcal{M})$ be the $\ell$-th eigenvalue of the continuum Laplace operator defined in Equation \eqref{eq:laplacianmanifold}. 
\begin{theorem} 
\label{thm2}
Suppose that Assumptions \ref{asp:1}, \ref{asp:2}, and \ref{asp:3} hold. 
Suppose further that $\veps$ is small enough (but not too small) so that
\begin{align}\label{conditionseps}
\begin{split}
\operatorname{max} \Big\{(m+5) &d_{\infty}(\mu_n, \mu),  2C m \eta \Big\} <\veps<\operatorname{min} \Big\{  1,\frac{i_0}{10}, \frac{1}{\sqrt{mK}},\frac{R}{\sqrt{27m}}\Big\},   \\ 
    &\big(\sqrt{\lambda_{\ell}(\mathcal{M})}+1\big)\veps +\frac{d_\infty(\mu_n, \mu)}{\veps}<\tilde c_p,
    \end{split}
\end{align}
    where $\tilde c_p$ is a constant that only depends on $m$ and the regularity of the density $p$, $C$ is a universal constant, and 
    \begin{equation*}
    \eta=C_{\mathcal{M}} \left( r^3+r\sigma+\frac{\sigma^2}{r} \right)
    \end{equation*}
    is the bound in \eqref{eq:distbound}. Then, with probability at least $1-4ne^{-cnr^{\operatorname{max}\{2m,m+4 \}}}$, 
\begin{align*}
        \frac{|\lambda_{\ell}(\Gamma_{\bar{\Y}_n,\veps})-\lambda_{\ell}(\mathcal{M})|}{\lambda_{\ell}(\mathcal{M})}  
        \leq \tilde C \left(\frac{\eta}{\veps}+\frac{d_\infty(\mu_n, \mu)}{\veps}+\big(1+\sqrt{\lambda_{\ell}(\mathcal{M})}\big)\veps+\Big( K+\frac{1}{R^2}\Big)\veps^2 \right),
\end{align*}
where $\tilde C$ only depends on $m$ and the regularity of $p$,  and
$c=\operatorname{min}\left\{ \frac{\alpha_m^2p_{min}^2}{4^{m+2}},  \frac{1}{16} \right\}.$
\end{theorem}

\begin{remark}
As described in Remark \ref{rmk:betterdist}, local regularization enables a smaller $\eta$ than if no regularization is performed. This in turn allows one to choose, for a given error tolerance, a smaller  connectivity $\veps$, leading to a sparser graph that is computationally more efficient. Note also that the bound in Theorem \ref{thm2} does not depend on the ambient space dimension $d,$ but only on the intrinsic dimension $m$ of the data. 
\end{remark}

\begin{remark}
Theorem \ref{thm2} concretely shows how an improvement in metric approximation translates into an improved estimation of global geometric quantities. We have restricted our attention to analyzing eigenvalues of a Laplacian operator, but we remark that the idea goes beyond this particular choice. For example, one can conduct an asymptotic analysis illustrating the effect of changing the similarity function in the approximation of other geometric quantities of interest like Cheeger cuts. Such analysis could be carried out using the variational convergence approach from \cite{trillos}.  

We would also like to mention that it is possible to make statements about convergence of eigenvectors of graph Laplacians following the results in \cite{SpecRatesTrillos}. We have omitted the details for brevity. 
\end{remark}

\begin{remark}
Theorem \ref{thm2} puts together Lemma \ref{lemma:specbound} below and a convergence result from \cite{SpecRatesTrillos} which we present for the convenience of the reader in Theorem \ref{thm:speccv}. We remark that any improvement of Theorem \ref{thm:speccv} would immediately translate into an improvement of our Theorem \ref{thm2}.
\end{remark}
\nc

\subsection{Related and Future Work}

Graph-based learning algorithms include spectral clustering, total variation clustering, graph-Laplacian regularization for semi-supervised learning, graph based Bayesian semi-supervised learning. A brief and incomplete summary of methodological and review papers is
  \cite{shi2000normalized,ng2002spectral,belkin2004semi,zhou2005regularization,spielman2007spectral,von2007tutorial,zhu2005semi,bertozzi2018uncertainty}. These algorithms involve either a graph Laplacian, the graph total variation, or Sobolev norms involving the graph structure. The large sample $n \rightarrow \infty$ theory studying the behavior of some of the above methodologies has been analyzed without reference to the intrinsic dimension of the data \cite{von2008consistency} and in the case of points laying \emph{on} a low dimensional manifold, see e.g. \cite{belkin2006manifold,garcia2018continuum,trillos2017consistency} and references therein. Some papers that account for \emph{both} the noisy and low intrinsic dimensional structure of data are \cite{niyogi2008finding,little2017multiscale,agapiou2015importance,weed2017sharp}. For example,  \cite{niyogi2008finding} studies the recovery of the homology groups of submanifolds from noisy samples. We use the techniques for the analysis of spectral convergence of graph-Laplacians introduced in  \cite{burago2013graph} and further developed in  \cite{SpecRatesTrillos}. The results in the latter reference would allow to extend our analysis to other graph Laplacians, but we do not pursue this here for conciseness. 
 
We highlight that the denoising by local regularization occurs at the level of the data set. That is, rather than denoising each of the observed features individually, we analyze denoising by averaging different data points. In practice combining both forms of denoising may be advantageous. For instance, when each of the data points corresponds to an image, one can first  denoise each image at the pixel level and then do regularization at the level of the data set as proposed here. In this regard, our regularization at the level of the data-set is similar to applying a filter at the level of individual pixels \cite{tukey1988computer}. The success of non-local filter image denoising algorithms suggests that non-local methods may be also of interest at the level of the data set, but we expect this to be application-dependent. Finally, while in this paper we only consider first-order regularization based on averages, a topic for further research is the analysis of local PCA regularization \cite{little2017multiscale}, incorporating covariance information.

It is worth noting the parallel between the local regularization that we study here and mean-shift and mode seeking methods \cite{chen2016comprehensive,fukunaga1975estimation}. Indeed, a side benefit of local averaging in classification and clustering applications is that the data-points are pushed to regions of higher density. This paralellism with mean-shift techniques also suggests the idea of doing local averaging iteratively. Local regularization may be also interpreted as a form of dictionary learning, where each data-point is represented in terms of its neighbors. For specific applications it may be of interest to restrict (or extend) the dictionary used to represent each data point \cite{haddad2014texture}.

\subsection{Outline}
The  paper is organized as follows. In Section \ref{sec:them1} we formalize the geometric setup and prove Theorem \ref{thm1}. Section \ref{sec:them2} contains the proof of Theorem \ref{thm2} and a lemma that may be of independent interest. Finally, Section \ref{sec:numerics} includes several numerical experiments.  In the Appendix we prove a technical lemma that serves as a key ingredient in proving Theorem \ref{thm1}.

\section{Distance Approximation}
\label{sec:them1}
In this section we prove Theorem \ref{thm1}. We start with Subsection
\ref{ssec:geometricpreliminaries} by giving some intuition on the geometric conditions imposed in Assumption \ref{asp:1} and introducing the main geometric tools in our analysis. In Subsection  \ref{ssec:localdist} we decompose the approximation error between the similarity functions $\delta_{\bar{\mathcal{Y}}_n}$ and $\delta_{\mathcal{X}_n}$ into three terms, which are bounded in Subsections \ref{ssec:boundingnoise}, \ref{ssec:geombias}, and \ref{sec:samplerror}.
\label{eqn:difdistance}

\subsection{Geometric Preliminaries}\label{ssec:geometricpreliminaries}
\subsubsection{Basic Notation}
For each $x \in \M$ we let $T_{x} \M$  be the tangent plane of $\M$ at $x$ centered at the origin. In particular, $T_{x}\M$ is a $m$-dimensional subspace of $\R^d$, and we denote by $T_x \M^\perp$ its orthogonal complement. We will use $\vol $ to denote the Riemannian volume form of $\M$. We will denote by $|x-\tilde x|$ the Euclidean distance between  arbitrary points in $\R^d$ and denote by $d_\M(x,\tilde x)$ the geodesic distance between points in $\M$. We denote by $B_{x}$ balls in $T_{x}\M$ and by $B_\M$ balls in the manifold $\M$ (with respect to the geodesic distance). Also, unless otherwise specified $B$, without subscripts will be used to denote balls in $\R^d$. We denote by $\alpha_m$ the volume of the unit Euclidean ball in $\R^m$. Throughout the rest of the paper we use  $R, i_0$  and $K$ to denote the reach, injectivity radius, and maximum absolute curvature of $\M,$ as in Assumption \ref{asp:1}. We now describe at an intuitive level the role that these quantities play in our analysis.

\subsubsection{The Reach}
The reach of a manifold $\M$ is defined as the largest value $t \in (0, \infty]$ for which the projection map  
\[  \{ x \in \R^d \: : \:  \inf_{ \tilde x \in \M} |x- \tilde x| < t   \} \longmapsto  \M \]
is well defined, i.e., every point in the tubular neighborhood around $\M$ of width $t$ has a unique closest point in $\M$. Our assumption that the noise level satisfies $\sigma<R$ guarantees that $x_i$ is the (well-defined) projection of $y_i$ onto the manifold. The reach can be thought of as an inverse conditioning number for the manifold \cite{niyogi2008finding}. We will use that the inverse of the reach provides a uniform upper bound on the second fundamental form (see Lemma \ref{lem:boundaccelerations}).

\subsubsection{Exponential Map, Injectivity Radius and Sectional Curvature}
We will make use of the exponential map $\exp$, which for every $x \in \M$ is a map  
\[ \exp_{x} : B_{x}(0, i_0) \rightarrow B_\M (x, i_0) \]
where $i_0$ is the injectivity radius for the manifold $\M$. We recall that the exponential map $\exp_x$ takes a vector $v \in T_{x} \M$ and maps it to the point $\exp_x(v)\in\M$ that is at geodesic distance $|v|$ from $x$ along the unit speed geodesic that at time $t=0$ passes through $x$ with velocity $v/|v|$. The injectivity radius $i_0$ is precisely the maximum radius of a ball in $T_x \M$ centered at the origin for which the exponential map is a well defined diffeomorphism for every $x$. We denote by $J_x$ the Jacobian of the exponential map $\exp_x$. Integrals with respect to $d \vol$ can then be written in terms of integrals on $T_x \M$ weighted by the function $J_{x}$. More precisely, for an arbitrary test function $\varphi: \M \rightarrow \R$, 
\[ \int_{B_\M(x, i_0)}  \varphi(\tilde x) d \vol(\tilde x) = \int_{B_{x}(0, i_0)}  \varphi\big( \exp_{x}(v)  \big) J_x(v) dv.  \]

For fixed  $0 < r \leq \min\{ i_0,1/\sqrt{K}\}$ one can obtain bounds on the metric distortion by the exponential map  $\exp_x\colon B(r) \subseteq T_x\M \rightarrow \M$ (\cite[Chapter 10]{docarmo1992riemannian} and \cite[Section 2.2]{burago2013graph}), and thereby guarantee the existence of a universal constant $C$ such that 

%
%
\begin{equation}
(1+C m K |v|^2)^{-1} \leq J_x(v) \leq (1+C m K |v|^2).
\label{eqn:EstimateJacobian}
\end{equation}
An immediate consequence of the previous inequalities is 
\begin{align}\label{eqn:EstimateVolumeBall}
| \text{vol}(B_\M(x,r)) - \alpha_m r^m |& \leq C m K r^{m+2},
\end{align}
where we recall $\alpha_m$ is the volume of the unit ball in $\R^m$. Equations \eqref{eqn:EstimateJacobian} and \eqref{eqn:EstimateVolumeBall} will be used in our geometric and probabilistic arguments and motivate our assumptions on the choice of local regularization parameter $r$ in terms of the injectivity radius and the sectional curvature.

\subsection{Local Distributions}\label{ssec:localdist}
Next we study the local behavior of $(X,Z)$. To characterize its local distribution, it will be convenient to introduce the following family of probability measures. 
\begin{definition}
	\label{def:conddistrib}
Let $y$ be a vector in $\R^d$ whose distance to $\M$ is less than $R$. Let $x$ be the projection of $y$ onto $\M$. We say that the random variable $(\tilde{X}, \tilde{Z})$ has the distribution $\boldsymbol{\mu}_y$ provided that
\[ \Prob\big(  (\tilde X , \tilde Z) \in  A_1 \times A_2  \big)  :=  \Prob\big( (X,Z) \in  A_1\times A_2  |  X+ Z  \in B(y,r)  \big), \]
for all Borel sets $A_1 \subseteq \M $ $A_2 \subseteq \R^d$, where in the above $(X, Z)$ is distributed according to $\boldsymbol{\mu}$.  
\end{definition}

In the remainder we use $\boldsymbol{\mu}_i$ as shorthand notation for $\boldsymbol{\mu}_{y_i}$.
As for the original measure $\boldsymbol{\mu}$, we characterize $\boldsymbol{\mu}_i$ in terms of a marginal and conditional distribution. We introduce the density $\widetilde{p}_i: \M \rightarrow \R$ given by
\begin{equation}
  \widetilde{p}_i(x):=  \frac{\Prob_i \big( X + Z  \in B(y_i, r)| X=x\big) }{\Prob_i \big( X + Z  \in B(y_i, r)\big) }  \cdot p(x),
  \label{eqn:newmarginal}
\end{equation}
and define
\begin{equation}
 \widetilde{p}_i(z|x) = \frac{\mathds{1}_{x+ z \in B(y_i,r)} }{\Prob_i \big( X+Z \in B(y_i, r)  |X=x\big)  }  \cdot p(z |x), 
 \label{eqn:newconditional}
 \end{equation}
 where in the above and in the remainder we use $\E_i$ and $\Prob_i$ to denote conditional expectation and conditional probability given $(x_i, z_i)$. It can be easily shown that these functions correspond to the marginal density of $\tilde{X}_i$ and the conditional density of $\tilde{Z}_i$ given $\tilde{X}_i=x$, where $(\tilde{X}_i, \tilde{Z}_i) \sim \boldsymbol{\mu}_i$. The distribution $\boldsymbol{\mu}_i$ is of relevance  because by definition of $\overline{y}_i$ one has
\[  \E_i [ \overline{y}_i   ] = \E_i[\tilde X _i + \tilde Z _i ]. \]

Now we are ready to introduce the main decomposition of the error between the similarity functions $\delta_{\bar{\mathcal{Y}}_n}$ and $\delta_{\mathcal{X}_n}$.  Using the triangle inequality we can write
\begin{align}
\big| |x_i -x_j | - |\bar{y}_i - \bar{y}_j |\big| 
& \le \big| \E_i[\tilde{X}_i ]-x_i -  ( \E_j[\tilde{X}_j ]-x_j )\big|  \label{eq:geometric} \\
&+ \big|\E_j[\tilde{Z}_j]\big| + \big|\E_i[ \tilde{Z}_i ] \big|  \label{eq:condnoise}\\ 
&+  \big|\E_i[\bar{y}_i] - \bar{y}_i\big| + \big| \E_j[ \bar{y}_j] - \bar{y}_j \big|. \label{eq:concentration}
\end{align} 
 In the next subsections we bound each of the terms \eqref{eq:condnoise} (expected conditional noise), \eqref{eq:geometric} (difference in \textit{geometric} bias),  and \eqref{eq:concentration} (sampling error). As we will see in Subsection \ref{sec:samplerror} we can control both terms in \eqref{eq:concentration} with very high probability using standard concentration inequalities. The other three terms are deterministic quantities that can be written in terms of integrals with respect to the distributions $\tilde{\mu}_i$ and $\tilde{\mu}_j$. To study these integrals it will be convenient to introduce two quantities $r_- < r < r_+$ (independent of $i=1, \dots, n$) satisfying:
\begin{enumerate}
	\item For all $x \in \M$ with $d_\M(x, x_i) > r_+$ we have 
	\[  \Prob_i \big( X+ Z \in B(y_i, r)| X=x \big)=0.  \]
	 In particular, the density $\tilde{p}_i(x)$ is supported in $ \overline{B_\M(x_i, r_+)}$.
	\item For all $x$ with $d_\M(x, x_i)< r_-$ we have
	\[ \Prob_i\big(X+ Z \in B(y_i, r)| X=x \big)=1. \]
\end{enumerate}

In Appendix \ref{sec:r+-} we present the proof of the following lemma giving estimates for $r_+$ and $r_-$.

\begin{lemma}[Bounds for $r_+$ and $r_-$]
Under Assumption \ref{asp:3}, the quantities
\begin{align*}
r_-  &:=  r   \left(   \sqrt{1 + \frac{4\sigma}{R} +  \frac{16\sigma^2 }{r^2} } + \frac{m\sigma}{R}   \right) ^{-1}, \\
r_+ &:= r \left(  \sqrt{ 1 - \frac{8r^2}{R} - \frac{4\sigma}{R}    }  -\frac{m\sigma}{R} \right)^{-1},
\end{align*}
satisfy properties i) and ii). Furthermore,
\[  r_+ - r_- \leq C_{m,R} \left( r^3 + r\sigma  + \frac{\sigma^2}{r} \right),   \; C_{m,R}:=\operatorname{max} \left\{ \frac{8m+32}{R},64\right\} \]
and 
\begin{equation}  \frac{1}{2} r_+\leq  r \leq 2r_-. \label{eq:rprm}
\end{equation}
\end{lemma}

\subsection{Bounding Expected Conditional Noise}\label{ssec:boundingnoise}

\begin{proposition}
	\label{prop:ECN}
	Suppose that Assumptions \ref{asp:1} and \ref{asp:2} hold. Then, 
	\begin{equation*}
	\big|\E_i[ \tilde{Z}_i ]\big| \le C_{m,p} \frac{\sigma}{r} (r_+ - r_-), \; \quad \; C_{m,p}:=\frac{4^{m+1}p_{max}}{mp_{min}}.
	\end{equation*}
\end{proposition}
\begin{proof}
	 Using the definition of $r_+,$ 
	\begin{align*}
	\E_i[ \tilde{Z}_i ]  &= \int_{B_\M(x_i, r_+) } \int z \tilde{p}_i(z|x) dz \,\, \tilde{p}_i(x) \,d\vol(x) \\
	&= \int_{B_\M(x_i, r_{-})} \int z \tilde{p}_i(z|x) dz\,\,  \tilde{p}_i(x) \,d\vol(x) \\ &+ \int_{B_\M(x_i,r_+) \setminus B_\M(x_i,r_-)} \int z \tilde{p}_i (z|x) dz \,\tilde{p}_i(x) d\vol(x).
	\end{align*}
	The first integral is the zero vector because for $x \in B_\M(x_i, r_{-}),$ we have $\tilde{p}(z|x) \propto p(z|x)$ and $p(z|x)$ is assumed to be centered. Therefore,
	\begin{align*}
	\big|\E_i[ \tilde{Z}_i ] \big| &\le \sigma \int_{B_\M(x_i,r_+) \setminus B_\M(x_i,r_-)} \tilde{p}_i(x) d\vol(x) \\
	& =  \frac{\sigma}{\Prob_i\bigl(X+Z \in B(y_i,r)\bigr)}\int_{B_\M(x_i,r_+) \setminus B_\M(x_i,r_-)}  p(x)  d\vol(x) \\
	&\le \frac{\sigma p_{max}}{\Prob_i\bigl(X+Z \in B(y_i,r)\bigr)} \int_{B_\M(x_i,r_+) \setminus B_\M(x_i,r_-)} d\vol(x)  \\
	&\le \frac{\sigma p_{max}}{\Prob_i\bigl(X+Z \in B(y_i,r)\bigr)} \int_{B_{x_i}(0,r_+) \setminus B_{x_i}(0,r_-)} J_{x_i}(v) dv\\
	& \leq  \frac{2 \alpha_m \sigma p_{max}}{\Prob_i\bigl(X+Z \in B(y_i,r)\bigr)} (r_+^m - r_-^m)\\
	& \leq \frac{2 \alpha_m \sigma p_{max}}{m\Prob_i\bigl(X+Z \in B(y_i,r)\bigr)}(r_+ - r_-) r_+^{m-1},
	\end{align*}
	where we have used \eqref{eqn:EstimateJacobian} and the assumptions on $r$ to say (in particular) that $ J_{x_i}(v) \le2$, and also the fact that, for $t>s>0,$ 
	 $$t^m  - s^m = \int_s^t \frac{u^{m-1}}{m}du \leq (t-s)\frac{t^{m-1}}{m}.$$  
	 Finally, notice that
	\[   \Prob_i \big(X+Z \in B(y_i, r)\big) \geq  \Prob_i \big(  X \in B_\M(x_i, r_-) \big) = \int_{B_{x_i}(0,r_-)} p\big(\exp_{x_i}(v)\big) J_{x_i}(v) dv \geq \frac{1}{2}p_{min}\alpha_m r_-^m, \]
	where again we have used \eqref{eqn:EstimateJacobian} to conclude (in particular) that $J_{x_i}(v) \geq 1/2$. The result now follows by \eqref{eq:rprm}. 
	
\end{proof}

\subsection{Bounding Difference in Geometric Bias}\label{ssec:geombias}

In terms of $r_+$ and $r_-$, the difference  $ \E_i[\tilde X_i]-x_i $ (and likewise $ \E_j[\tilde X_j]-x_j $) can be written as:
\begin{align*}
\begin{split}& \E_i[\tilde X_i]-x_i  = \int_{B_\M(x_i, r_+)} (x-x_i )\tilde{p}_i (x) d\vol(x) 
\\& =  \int_{B_{x_i} (0, r_+)} \big( \exp_{x_i}(v)-x_i\big)\tilde{p}_i \big(\exp_x(v)\big)  J_{x_i}(v) dv
\\&= \int_{B_{x_i} (0, r_+)} \big( \exp_{x_i}(v)-x_i\big)\tilde{p}_i \big(\exp_x(v)\big) dv + \int_{B_{x_i} (0, r_+)}\big( \exp_{x_i}(v)-x_i\big)\tilde{p}_i\big(\exp_x(v)\big) \big(J_{x_i}(v) -1 \big) dv
\\&= \frac{1}{\Prob_i\big(X + Z \in B(y_i, r)  \big)} \int_{B_{x_i}(0, r_-)}\big( \exp_{x_i}(v)-x_i\big)\tilde{p}_i \big(\exp_x(v)\big) dv 
\\&+  \int_{B_{x_i}(0, r_+) \setminus B_{x_i} (0 , r_-)}\big( \exp_{x_i}(v)-x_i\big)\tilde{p}_i \big(\exp_x(v)\big) dv+ \int_{B_{x_i}(0, r_+)}\big( \exp_{x_i}(v)-x_i\big)\tilde{p}_i \big(\exp_x(v)\big) \big(J_{x_i}(v) -1 \big) dv
\\&=: \frac{1}{\Prob_i\big(X + Z \in B(y_i, r)  \big)} \int_{B_{x_i}(0, r_-)}\big(\exp_{x_i}(v)-x_i\big)p \big(\exp_x(v)\big) dv+ \xi_i,
\end{split}
\label{eqn:xi-x}
\end{align*}
where the second to last equality 
follows from \eqref{eqn:newmarginal}.
To further simplify the expression for $x_i -  \E_i[\tilde X _i]$ let us define
\[  b_i:=   \int_{B_{x_i}(0, r_-)} \big( \exp_{x_i}(v)-x_i\big)p\big(\exp_x(v)\big) dv.\]
It follows that
\begin{align}
\begin{split}
\big| \E_i[\tilde X_i]-x_i  - (\E_j[ \tilde X_j ]-x_j ) \big| & \leq \Big|\frac{b_i}{P_i} - \frac{b_j}{P_j} \Big|  + \big|\xi_i\big| + \big|\xi_j\big|
\\& \leq \Big|\frac{1}{P_i} - \frac{1}{P_j}\Big| \big|b_i\big| + \frac{1}{P_j} \big|b_i - b_j\big| + \big|\xi_i \big| + \big|\xi_j\big|,
\end{split}
\end{align}
where in the above
\[ P_i := \Prob_i\big(X + Z \in B(y_i, r)  \big), \quad P_j := \Prob_j\big(X + Z \in B(y_j, r)  \big). \]


\begin{lemma}
\label{lem:geo}
	The following hold.
	\begin{enumerate}
		\item The terms $P_i$ satisfy
		\[    \frac{1}{2}p_{min}\alpha_m r_-^m \leq P_i . \]	
		\item The terms $\xi_i$ satisfy:
		\begin{align*}
		\big|\xi_i\big|\leq C_1(r_{+}-r_{-})+C_2r^3,
		\end{align*}
		where, up to universal multiplicative constants, 
		\begin{align*}
		C_1=\frac{4^{m+1} p_{max}}{mp_{min}}, \quad C_2=4^{m+3}mK\frac{p_{max}}{p_{min}}.
		\end{align*}

		\item Suppose that $d_\M(x_i, x_j) \leq r$. Then, 	 
		\[ |P_i - P_j | \leq  C_3 r^{m+1} +  C_4 (r_+ - r_-)r^{m-1} + C_5 r^{m+2},\]
		where, up to universal multiplicative constants,
		\[C_3=C_p\alpha_m, \quad C_4= \frac{2^{m-1}\alpha_m p_{max}}{m}, \quad C_5=mK p_{max}\alpha_m.\]
		and  $C_p$ only depends on bounds on the first derivatives of the density $p$.
	\end{enumerate}
	
\end{lemma}
\nc

\begin{proof}
	The first inequality was already obtained at the end of the proof of Proposition \ref{prop:ECN}. For the second inequality recall that
	\begin{align*}
	\xi_i&=\int_{B_{x_i}(0,r_+) \backslash B_{x_i}(0,r_-)} \big(x_i-\exp_{x_i}(v)\big) \tilde{p_i}\big(\exp_{x_i}(v)\big)  dv +\int_{B_{x_i}(0,r_+)} \big(x_i-\exp_{x_i}(v)\big)  \tilde{p_i}\big(\exp_{x_i}(v)\big)[J_{x_i}(v)-1]dv \\
	&:= I_1+I_2.
	\end{align*}
	For the first term we notice that $|x_i - \exp_{x_i}(v)| \leq d_\M(x_i, \exp_{x_i}(v))\leq r_+$. Thus using i) and the definition of $\tilde{p}_i$ we have
	\begin{align*}
	    |I_1| & \leq \frac{r_+ p_{max}\alpha_m}{\Prob_i\big(X+Z\in B(y_i,r) \big)}  (r_+^m-r_-^m) \leq \frac{4^{m+1}p_{max}}{mp_{min}}(r_+-r_-).
	\end{align*}
	For the second term we use i) and \eqref{eqn:EstimateJacobian} to see that 
	\begin{align*}
	    |I_2|\leq \frac{CmKp_{max}\alpha_m}{\Prob_i\big(X+Z\in B(y_i,r) \big)} r_+^{m+3}\leq C4^{m+3}mK\frac{p_{max}}{p_{min}}r^3 .
	\end{align*}
For iii) we notice that by definition of $r_-$ and $r_+$ we can write 
	\[ \Prob_i\big(X \in B_{x_i}(0, r_-)\big)- \Prob_j\big(X \in B_{x_j}(0, r_+)\big)  \leq P_i - P_j \leq  \Prob_i\big(X \in B_{x_i}(0, r_+)\big) - \Prob_j\big(X \in B_{x_j}(0, r_-)\big),  \]
	and in particular it is enough to bound $H_{ij}:=\big|\Prob_i\big(X \in B_{\M}(x_i, r_+)\big) -  \Prob_j\big(X \in B_{\M}(x_j, r_-)\big)   \big|$. We can expand $H_{ij}$ as follows.
	\begin{align*}
	    H_{ij}&= \int_{B_{x_i}(0,r_-)} p\big(\exp_{x_i}(v) \big) dv - \int_{B_{x_j}(0,r_-)} p\big(\exp_{x_j}(\tilde{v}) \big) d\tilde{v}\\
	    &+ \int_{B_{x_i}(0,r_+)\backslash B_{x_i}(0,r_-)} p\big(\exp_{x_i}(v) \big) dv \\
	    &+\int_{B_{x_i}(0,r_-)} p\big(\exp_{x_i}(v) \big) \big(J_{x_i}(v)-1\big)dv -\int_{B_{x_j}(0,r_-)} p\big(\exp_{x_j}(\tilde{v}) \big) \big(J_{x_j}(\tilde{v})-1\big)d\tilde{v}\\
	    &:=\mathcal{I}_1+\mathcal{I}_2+\mathcal{I}_3.
	\end{align*}
	By a similar argument as above, we can bound $\mathcal{I}_2$ and $\mathcal{I}_3$ by 
	\begin{align*}
	    |\mathcal{I}_2|&\leq p_{max} \alpha_m (r_+^m-r_-^m) \leq \frac{2^{m-1}}{m}\alpha_m p_{max}(r_+-r_-)r^{m-1}, \\
	    |\mathcal{I}_3|&\leq 2CmK p_{max}\alpha_m r_-^{m+2} \leq 2CmK p_{max}\alpha_m r^{m+2}.
	\end{align*}

	Finally, we notice that we can identify $B_{x_i}(0,r_-)$ with $B_{x_j}(0,r_-)$. From the assumed smoothness on $p$ (which in particular is $C^1$) we see that for any $v\in B_{x_i}(0,r_-)$  we have 
	\[ \big| p\big(\exp_{x_i}(v)\big)- p\big(\exp_{x_j}(v)\big)\big| \leq C_p d_\M\big(\exp_{x_i}(v), \exp_{x_j}(v)\big)  \leq 3C_p r. \]
Then it follows that $|\mathcal{I}_1|\leq 3C_p\alpha_m r^{m+1}$ and we get the desired result. 
\end{proof}

We now bound the difference $\big| b_i - b_j \big|$ for nearby points $x_i, x_j$, where we recall that 
\[ b_i :=  \int_{B_{x_i}(0,r_-)} ( \exp_{x_i}(v)-x_i) dv.\]

\begin{proposition}
	\label{prop:diffbias}
Suppose that $x_i$ and $x_j$ are such that $d_\M(x_i, x_j) \leq r$. Then, 
\[ \big|b_i - b_j \big| \leq  C r^{m+3},\]
where the constant $C$ can be written as
\[ C= p_{max} \alpha_m \left( \frac{6 \sqrt{m}}{R^2} + \Big(1+\frac{4}{R}\Big)C_\M  \right)  +  \frac{C_p}{R}\alpha_m , \]
where $C_p$ is a constant that depends on bounds on first and second derivatives of the density $p$,  and $C_\M$ is a constant that depends only on the change in second fundamental form along $\M$ (a third order term).

\end{proposition}
As we will see  Proposition \ref{prop:diffbias} can be proved putting together simple ideas from differential geometry. We present the required auxiliary results as we develop the proof of the proposition. 

We start by conveniently writing $b_i$ and $b_j$ in a way that facilitates their direct comparison. Indeed, for any given $v \in B_{x_i}(0, r_-)$ let us consider the curves 
\[ \gamma_{v,i}(t) := \exp_{x_i}\left( t \frac{v}{|v|}  \right), \quad t \in [0, |v|],  \]
and
\[  t \in [0, |v|] \mapsto x_i + t \in [0, |v|].  \]
$\gamma_{v,i}$ is an arc-length parameterized geodesic on $\M$ that starts at the point $x_i$ and at time $|v|$ passes though the point $\exp_{x_i}(v)$. Its initial velocity $\dot{\gamma}_{v,i}(0)$ is the vector $v/|v|$. On the other hand, while the second curve does not stay in $\M$ for $t>0$, it does have the same starting point and velocity as $\gamma_{v,i}$. We can use the fundamental theorem of calculus to write:
\[ \exp_{x_i}(v)  -  (x_i + v) = \int_{0}^{|v|} \left(\dot{\gamma}_{v,i}- \frac{v}{|v|}\right)dt, \]
as well as
\begin{equation}
\dot{\gamma}_{v,i}(t) - \frac{v}{|v|}  = \int_{0}^t \ddot{\gamma}_{v,i}(s) ds , \quad \forall t \in [0, |v|].  
\label{eqn:velocities}
\end{equation}
In particular, we have the second order representation
\begin{equation}
\exp_{x_i}(v) - x_i   = v+ \int_{0}^{|v|} \int_{0}^t \ddot{\gamma}_{v,i}(s)dsdt.
\label{secondorderrep} 
\end{equation}
As a consequence of the previous formula we can rewrite $b_i$ as 
\begin{align}
\label{eqn:bi}
\begin{split}
b_i & =  \int_{B_{x_i}(0,r_-)} \big( \exp_{x_i}(v)-x_i\big) p\big(\exp_{x_i}(v)\big)dv  
\\&=  \int_{B_{x_i}(0,r_-)}  p\big(\exp_{x_i}(v)\big)\int_{0}^{|v|} \int_{0}^t \ddot{\gamma}_{v,i}(s)dsdtdv + \int_{B_{x_i}(0,r_-)} v p\big(\exp_{x_i}(v)\big) dv .
\end{split}
\end{align}
Completely analogous definitions and statements can be introduced to represent $b_j$. 

To exploit the formula \eqref{eqn:bi} in order to compare $b_i$ and $b_j$ it is useful to relate vectors in $T_{x_i} \M$ with vectors in $T_{x_j}\M$ by using a convenient linear isometry $F_{ij}: T_{x_i} \M  \mapsto T_{x_j }\M $ constructed using parallel transport.

\begin{lemma}
	\label{lem:isometry}
	Suppose that $x_i$ and $x_j$ are such that $d_\M(x_i, x_j) \leq r$. Let $\phi: t \in [0, d_\M(x_i, x_j)  ] \mapsto \phi(t) \in \M,$ be the arc-length parameterized geodesic starting at $x_i$ at time zero and passing through $x_j$ at time $t= d_\M(x_i, x_j)$. For an arbitrary vector $v \in T_{x_i}\M$ let $V_v$ be the (unique) vector field along $\phi$ that solves the ODE
\[\begin{cases} \frac{D}{dt} V_v(t) = 0 , \quad t \in \big(0, d_\M(x_i, x_j)\big), \\ V_v(0)= v, \end{cases}, \]
where $\frac{D}{dt}$ denotes the covariant derivative (on $\M$) along the curve $\phi$. Then, the map $F_{ij}$ defined by 
\[ F_{ij}: v \longmapsto \tilde{v}:= V_v\big( d_\M(x_i, x_j)  \big)\]
is a linear isometry.  Moreover,
\begin{equation}
\label{eqn:FijDiff}
|v - \tilde{v}| \leq \frac{1}{R}|v|d_{\M}(x_i, x_j) , \quad \forall v \in T_{x_i }\M.
\end{equation}
\end{lemma}

\begin{proof}

First note that $F_{ij}$ is a linear isometry since the ODE defining $V_v$ is linear and the vector fields $V_v$ are parallel to the curve $\phi$ by definition. To get the estimate \eqref{eqn:FijDiff} we can use the fundamental theorem of calculus and write
\[ \tilde{v} = v + \int_{0}^t \dot{V}_v(s)ds,    \]
where $t:= d_\M(x_i, x_j)$. The fact that $V_v$ is parallel along the curve $\phi$ implies that $\dot{V}_v(s) \in T_{\phi(s)}\M ^\perp$ and furthermore that for arbitrary unit norm $\eta$ with $\eta \in T_{\phi(s)} \M^\perp$ we have
\[ |\langle \dot{V}_v(s) , \eta \rangle|  = |\langle S_\eta ( V_v(s) ) , \dot{\phi}(s) \rangle|  \leq \lVert S_\eta \rVert |V_v(s)| | \dot{\phi}(s)|  = \lVert S_\eta \rVert |v| ,\] 
where $S_\eta$ is the so called \textit{shape operator} representing the second fundamental form (see Proposition 2.3. Chapter 6 in Do Carmo). The relevance of the previous inequality is that when combined with Proposition 6.1 in \cite{niyogi2008finding} (where the norm of the second fundamental form for an arbitrary normal vector is shown to be bounded by the reciprocal of the reach of the manifold) it implies that 
\[ | \dot{V}_v(s) | \leq \frac{|v|}{R}, \quad \forall s \in [0,t].\] 
Therefore, 
\[ |\tilde{v}- v| \leq \int_{0}^t | \dot{V}_v(s)|ds \leq \frac{|v|}{R}d_\M(x_i, x_j),  \]
establishing in this way the desired bound.
\end{proof}

From now on, for a given $v \in B_{x_i}(0, r_-)$ we let $\tilde{v} \in B_{x_j}(0, r_-)$ be its image under $F_{ij}$. We consider the curve:
\[  \gamma_{\tilde v,j}(t) := \exp_{x_j}\left( t \frac{\tilde v}{|\tilde v|}  \right), \quad t \in [0, |\tilde v|],  \]
where we recall that $|v|= |\tilde{v}|$ because $F_{ij}$ is a linear isometry. We can then make a change of variables and write $ b_j$ as 
\begin{equation}
 b_j =  \int_{B_{x_i}(0,r_-)}  p\big(\exp_{x_j}(\tilde v)\big)\int_{0}^{|v|} \int_{0}^t \ddot{\gamma}_{\tilde v,j}(s)dsdtdv + \int_{B_{x_i}(0,r_-)} \tilde v p\big(\exp_{x_j}(\tilde v)\big) dv . 
 \label{eqn:bj}
 \end{equation}

In the next lemma we find bounds for the norms of accelerations.

\begin{lemma}
\label{lem:boundaccelerations}	
Let $v \in B_{x_i}(0, r_-)$ and let $\tilde v $ be as in Lemma \ref{lem:isometry}. Then, for all $t \in [0, |v|] $ we have
\[ | \ddot{\gamma}_{v,i}(t) | \leq \frac{1}{R}, \] 
and
\begin{align*}
| \dot{\gamma}_{v,i}(t) -  \dot{\gamma}_{\tilde v , j}(t) | \leq 2 \frac{|v|}{R} + \frac{d_\M(x_i,x_j)}{R}.
\end{align*}
\end{lemma}	

\begin{proof}
	The first inequality appears in the proof of Proposition 2 in \cite{niyogi2008finding} and is obtained in a completely analogous way as we obtained the bound for $\dot{V}_v$ in the proof of Lemma  \ref{lem:isometry} (given that unit speed geodesics are auto parallel). 
	
To prove the second estimate, we notice that from the first bound and  \eqref{eqn:velocities} it follows
\[ \left | \dot{\gamma}_{v,i}(t) - \frac{v}{|v|}   \right | \leq \frac{|v|}{R}, \quad \forall t \in [0, |v|].  \]
Naturally, a similar inequality holds for $\gamma_{\tilde v , j}$. Using Lemma \ref{lem:isometry} we conclude that for all $t \in [0, |v|]$ (recall that $|v| = |\tilde{v}|$) 
\begin{align*}
\begin{split}
| \dot{\gamma}_{v,i}(t) -  \dot{\gamma}_{\tilde v , j}(t) | &\leq \left| \frac{v}{|v|} - \frac{\tilde v}{| \tilde v |}\right|  + \left | \dot{\gamma}_{v,i}(t) - \frac{v}{|v|}   \right | +   \left | \dot{\gamma}_{\tilde v,i}(t) - \frac{\tilde v}{|\tilde v|}   \right |
\\& \leq 2\frac{|v|}{R} + \frac{1}{|v|} | v - \tilde{v}| 
\\& \leq 2 \frac{|v|}{R} + \frac{d_\M(x_i,x_j)}{R}.
\end{split}
\end{align*}

\end{proof}

From our assumption that the density $p$ was in $C^2(\M)$ it follows that
\[  p\big(\exp_{x_i}(v)\big) = p(x_i) + \langle  \nabla p (x_i), v \rangle  + R_i(v), \]
\[  p\big(\exp_{x_j}(\tilde v)\big) = p(x_j) + \langle  \nabla p (x_j), \tilde v \rangle  + R_j(\tilde v),  \]
where the remainder terms satisfy 
\[ \max \{ | R_i(v)|, |R_j(\tilde v)| \} \leq C_p |v|^2, \]
for a constant $C_p$ that depends on a uniform bound on second derivatives of $p$. Likewise,  
\[ \max \{  |p(x_i)- p(x_j)|, |\nabla p(x_i) - \nabla p (x_j)| \}  \leq C_p d_\M(x_i, x_j).  \]
Plugging the previous identities in the expressions \eqref{eqn:bi} and \eqref{eqn:bj}, using \eqref{eqn:FijDiff}, using the bound on accelerations from Lemma \ref{lem:boundaccelerations}, and finally, using the fact that by symmetry
\[  \int_{B_{x_i}(0,r_-)}   p(x_i) v dv =0 , \quad \int_{B_{x_i}(0,r_-)} p(x_j) \tilde v  dv=0,     \]
we can conclude that 
\begin{align}
\label{eqn:Aux1acce}
\begin{split}
 \big|b_i - b_j \big| & \leq  \int_{B_{x_i}(0,r_-)}\int_{0}^{|v|} \int_{0}^t \big| p(x_j)\ddot{\gamma}_{\tilde v,j}(s) - p(x_i)  \ddot{\gamma}_{v,i}(s)  \big|  dsdtdv   +  \frac{C_p}{R}\alpha_m r^{m+2}\big(r + d_\M(x_i,x_j)\big)  
 \\& \leq p_{max} \cdot \int_{B_{x_i}(0,r_-)}\int_{0}^{|v|} \int_{0}^t \big| \ddot{\gamma}_{\tilde v,j}(s) -   \ddot{\gamma}_{v,i}(s)  \big|  dsdtdv 
  \\&+ \int_{B_{x_i}(0,r_-)}\int_{0}^{|v|} \int_{0}^t \big| p(x_j) - p(x_i)\big|\big|  \ddot{\gamma}_{v,i}(s)  \big|  dsdtdv
 + \frac{C_p}{R}\alpha_m r^{m+2}\big(r + d_\M(x_i,x_j)\big)
 \\&\leq  p_{max} \cdot \int_{B_{x_i}(0,r_-)}\int_{0}^{|v|} \int_{0}^t \big| \ddot{\gamma}_{\tilde v,j}(s) -   \ddot{\gamma}_{v,i}(s)  \big|  dsdtdv +  \frac{C_p}{R}\alpha_m r^{m+2}\big(r + d_\M(x_i,x_j)\big).
\end{split}
 \end{align}
In the above $C_p $ is a constant that depends on derivatives of $p$ of order $1$ and order $2$ (and in particular is equal to zero when $p$ is constant) and $\alpha_m$ is the volume of the $m$-dimensional unit ball.  

Proposition \ref{prop:diffbias} now follows from the next lemma where we bound the difference of accelerations.
\begin{lemma}\label{lemma5}
Let $v \in B_{x_i}(0, r_-)$ and let $\tilde v $ be as in Lemma \ref{lem:isometry}. Then, for all $t\in [0,|v|]$ we have 
\[ \big| \ddot{\gamma}_{v,i}(t) - \ddot{\gamma}_{\tilde v , j}(t)\big | \leq \left(2\frac{\sqrt m}{R^2} + C_\M\right)\big(2 |v| + d_{\M}(x_i, x_j)\big) + 2C_\M \left( \frac{|v|}{R}+ \frac{d_\M(x_i, x_j)}{R} \right), \]
where $C_\M$ is a constant that depends only on the change in second fundamental form along $\M$ (a third order term).  
\end{lemma}

\begin{proof}
For a fixed $t \in [0, |v|]$ we let
\[   x:= \gamma_{v, i}(t) , \quad \tilde{x} := \gamma_{\tilde v , j}(t).   \]
We start by constructing a convenient linear map
\[   \eta \in T_{x}\M^\perp \mapsto \tilde{\eta} \in T_{\tilde x} \M ^\perp.   \]
For this purpose we use a \textit{frame} $E_1, \dots, E_m$ on a neighborhood (in $\M$) of $x$ containing the geodesic connecting $x$ and $\tilde{x}$ . The frame is constructed by parallel transporting an orthonormal basis $E_1(x), \dots, E_m(x)$ of $T_{x}\M$ along geodesics emanating from $x$. Now, associated to $\eta \in T_{x}\M ^\perp$ we define the (normal) vector field $N_\eta$ by
\[  N_\eta: = \eta - \sum_{l=1}^m \langle E_l , \eta \rangle E_l. \]
Let $\phi_{x \tilde{x}}$ be the arc-length parameterized geodesic with $\phi_{x\tilde x}(0)=x $ and $\phi_{x\tilde x}(\tilde t)= \tilde x$. We restrict the vector field $N_\eta$ to the curve $\phi_{x, \tilde x}$ and abuse notation slightly to write $N_\eta(s)$ and $E_j(s)$ for the value of the vector fields at the point $\phi_{x\tilde x }(s)$.  We let $\tilde{\eta} := N_{\eta}(\tilde{t}) $ and notice that
\begin{align}
\label{eqn:etaetatilde}
\begin{split}
| \eta - \tilde{\eta}|& = \left( \sum_{l=1}^m \langle E_l(\tilde t) , \eta \rangle^2   \right)^{1/2} = \left( \sum_{l=1}^m \langle E_l(\tilde t)- E_l(0) , \eta \rangle^2   \right)^{1/2} \leq \frac{\sqrt {m} d_\M(x, \tilde x) |\eta| }{R},
\end{split}
\end{align}
where in the last line we have used the fact that $|E_l(\tilde t)- E_l(0)| \leq \frac{\tilde t}{R}$ (proved in the exact same way as \eqref{eqn:FijDiff}).

Let $\eta \in T_{x}\M ^\perp$ be a unit norm vector and let $\tilde{\eta}$ be as constructed before. Since $N_\eta$ is a normal vector field which locally extends  $\eta$ we can follow the characterization for the shape operator in Proposition 2.3 Chapter 6 in \cite{docarmo1992riemannian} and deduce that:
\[  \langle \ddot{\gamma}_{v,i}(t) , \eta \rangle   =   \langle  S_\eta(\dot{\gamma}_{v,i}) , \dot{\gamma}_{v,i}(t) \rangle  =   \langle  \frac{d}{dt}N_\eta(\gamma_{v,i}(t) ) , \dot{\gamma}_{v,i}(t) \rangle.\]
Moreover, the smoothness of the manifold $\M$ allows us to extend $N_\eta$ smoothly to a neighborhood in $\R^d$ of $x$ and $\tilde{x}$ (we also use $N_\eta$ to represent the extension), and in particular
\[ \lVert \nabla N_\eta (x) - \nabla N_\eta (\tilde{x}) \Vert \leq C_\M |x-\tilde{x}| \leq C_\M d_\M(x, \tilde x),  \]
\[ \lVert \nabla N_\eta(\tilde{x})\rVert \leq C_\M,  \] 
where $C_\M$ is a constant that depends only on the change in second fundamental form along $\M$ (which can be controlled uniformly given that $\M$ is smooth and compact). We can then use the chain rule and write:
\[  \langle \ddot{\gamma}_{v,i}(t) , \eta \rangle   =  \langle \nabla N (x) \dot{\gamma}_{v,i}  , \dot{\gamma}_{v,i} \rangle,   \]
and in a similar fashion
\[ \langle \ddot{\gamma}_{\tilde v,j}(t) ,  \eta \rangle  = \langle \ddot{\gamma}_{\tilde v,j}(t) , \eta - \tilde {\eta} \rangle +  \langle \ddot{\gamma}_{\tilde v,j}(t) , \tilde \eta \rangle = \langle \ddot{\gamma}_{\tilde v,j}(t) , \eta - \tilde {\eta} \rangle +   \langle -\nabla N (\tilde x) \dot{\gamma}_{\tilde v,j}(t) , \dot{\gamma}_{\tilde v,j}(t) \rangle. \]
Using the triangle and Cauchy-Schwartz inequalities we obtain:
\begin{align*}
\begin{split}
\lvert \langle  \ddot{\gamma}_{v,i}(t) - \ddot{\gamma}_{\tilde v, j}(t) , \eta \rangle \rvert & \leq |\ddot{\gamma}_{\tilde{v},j}| |\eta - \tilde \eta | + \lVert \nabla N (x) - \nabla N (\tilde x) \rVert|\dot{\gamma}_{v,i}|^2+  \lVert \nabla N (\tilde{x}) \rVert|\dot{\gamma}_{v,i} - \dot{\gamma}_{\tilde v , j} | (|\dot{\gamma}_{\tilde{v},j}| + |\dot{\gamma}_{v,i}| )
\\& \leq \frac{\sqrt m}{R^2}d_\M (x, \tilde{x}) + C_\M d_\M(x,\tilde x) + 2C_\M \left( \frac{|v|}{R}+ \frac{d_\M(x_i, x_j)}{R} \right). 
\end{split} 
\end{align*}
Since the above inequality holds for all $\eta \in T_{x}\M ^\perp$ with norm one, we conclude that
\[  | \Pi_x(\ddot{\gamma}_{v,i}(t)) - \Pi_{x}(\ddot{\gamma}_{\tilde v , j}(t))  |\leq \frac{\sqrt m}{R^2}d_\M (x, \tilde{x}) + C_\M d_\M(x, \tilde x) + 2C_\M \left( \frac{|v|}{R}+ \frac{d_\M(x_i, x_j)}{R} \right),   \]
where $\Pi_x $ represents the projection onto $T_{x}\M^{\perp}$. Moreover, since $\ddot{\gamma}_{v,i}(t)$ is the acceleration of a unit speed geodesic passing through $x$, we know that $\ddot{\gamma}_{v,i}(t) \in T_{x}\M^\perp$, so that $\Pi_x(\ddot{\gamma}_{v,i})= \ddot{\gamma}_{v,i} $.  Similarly we have  $\Pi_{\tilde x}(\ddot{\gamma}_{\tilde v,j})= \ddot{\gamma}_{\tilde v,j} $ (where $\Pi_{\tilde x}$ represents projection onto $T_{\tilde x}\M^{\perp}$ ) . Hence
\begin{align}
\begin{split}
|\ddot{\gamma}_{v,i}(t) - \ddot{\gamma}_{\tilde v,j}(t)    | & \leq |\Pi_{x} \ddot{\gamma}_{v,i}(t) - \Pi_{x} \ddot{\gamma}_{\tilde v,j}(t)   | + |  \Pi_x \ddot{\gamma}_{\tilde  v,j}(t) -  \ddot{\gamma}_{\tilde  v,j}(t) |,
\end{split}
\end{align}
and so it remains to find a bound for $ |  \Pi_x \ddot{\gamma}_{\tilde  v,j}(t) -  \ddot{\gamma}_{\tilde  v,j}(t) |$. We can write
\[ \Pi_x \ddot{\gamma}_{\tilde v , j}= \ddot{\gamma}_{\tilde v , j} - \sum_{l=1}^m \langle  \ddot{\gamma}_{\tilde v , j} ,   E_l(0) \rangle E_l(0).  \]
Therefore, 
\[ |  \ddot{\gamma}_{\tilde v , j}  -  \Pi_x \ddot{\gamma}_{\tilde v , j}  |  = \left(\sum_{l=1}^m \langle \ddot{\gamma}_{\tilde v , j} , E_l(0) \rangle ^2\right)^{1/2}  = \left( \sum_{l=1}^m \langle \ddot{\gamma}_{\tilde v , j} , E_l(0)- E_l(\tilde{t}) \rangle^2 \right)^{1/2} =  \sqrt{m}\frac{d_\M(x, \tilde x)}{R^2}.\]
Putting everything together we deduce that 
\begin{align*}
|\ddot{\gamma}_{v,i}- \ddot{\gamma}_{\tilde v, j}|  & \leq (2\frac{\sqrt m}{R^2} + C_\M)d_\M (x, \tilde{x})  + 2C_\M \left( \frac{|v|}{R}+ \frac{d_\M(x_i, x_j)}{R} \right) 
\\& \leq (2\frac{\sqrt m}{R^2} + C_\M)(2 |v| + d_{\M}(x_i, x_j)) + 2C_\M \left( \frac{|v|}{R}+ \frac{d_\M(x_i, x_j)}{R} \right),
\end{align*}
where in the last step we have used the triangle inequality 
\[d_\M(x,\tilde x) \leq d_\M(x, x_i) + d_\M(x_i, x_j)+ d_\M(x_j, \tilde x) \leq 2 |v| + d_\M(x_i,x_j).\]
\end{proof}

\nc

\begin{remark}
Notice that the computations in the proof of Proposition \ref{prop:diffbias} also show that
\[ |b_i | \leq C r^{m+2}, \quad \quad  i=1, \dots, n. \]
Indeed, this can be seen directly from \eqref{eqn:bi}, Lemma \ref{lem:boundaccelerations} (which bounds the acceleration term), and the fact that the first term on the right-hand side of the following expression drops by symmetry:
\[ \int_{B_{x_i}(0, r_-)}  p(\exp_{x_i}(v))  vdv = p(x_i)\int_{B_{x_i}(0, r_-)} v  dv    +  \int_{B_{x_i}(0,r_-)} (\langle \nabla p (x_i) ,v \rangle + R_i(v)) v dv.  \]
Note that while the terms $b_i$ are $O(r^{m+2})$, their difference (for nearby points) is $O(r^{m+3})$. This gain in order is directly connected to what was discussed in Remark \ref{rem:choiceofsigma}.
\end{remark}

\subsection{Bounding Sampling Error} 
\label{sec:samplerror}
We will make use of two concentration inequalities to bound the sampling error. We first recall Hoeffding's inequality. 
\begin{lemma}[Hoeffding's inequality]\label{lemma:concentration1}
Let $w_1,\ldots,w_n$ be i.i.d samples from a random variable $w$ taking values in the interval [0,1] and let $\overline{w}$ be the sample average. Then, 
\begin{equation*}
    \Prob \left(  |\overline{w}-\E[\overline{w}]|>t  \right)\leq 2e^{-2nt^2}.
\end{equation*}
\end{lemma}
The next is a generalization for random vectors that follows directly from the simple and elegant work \cite{Pinelis} (more precisely, Theorem 3).
\begin{lemma} \label{lemma:concentration2}
Let $W_1,\ldots,W_n$ be i.i.d samples from a random vector $W$ such that $|W|\leq  M$ for some constant $M$, and $\E[W]=0$. Let $\overline{W}$ be the sample average. Then,
\begin{equation*}
\Prob \biggl( \Big|\overline{W} -    \E\big[\overline{W}\big]  \Big| > \sqrt{\frac{M^2}{n}}t   \biggr) \le 2 e^{-t^2/16}.
\end{equation*}
\end{lemma}

\begin{proposition}
\label{prop:SE}
Suppose Assumption \ref{asp:3} holds. Then, 
\begin{equation*}
\Prob\left( \big| \overline{y}_i -\E_i[ \overline{y}_i] \big| > \sqrt{\frac{2^{m+4}}{\alpha_m p_{min}}}r^{3} \right) \leq4e^{-cnr^{\operatorname{max}\{2m,m+4 \}}}, \; where \; c=\operatorname{min}\left\{ \frac{\alpha_m^2p_{min}^2}{4^{m+2}},  \frac{1}{16}\right\}.
\end{equation*}
In particular, if $nr^{\operatorname{max}\{2m,m+4 \}}\gg 1$, then $\big| \overline{y}_i -\E_i[ \overline{y}_i]\big| \leq \sqrt{\frac{2^{m+4}}{\alpha_m p_{min}}}r^{3}$ with high probability. 
\end{proposition}
\begin{proof}
Let $N_i$ be the number of points in $B(y_i,r)$. Notice that $\tilde{x}_i+\tilde{z}_i-\E_i[\tilde{X}_i+\tilde{Z}_i]$ is centered and bounded by $2r$ in norm, and $\overline{y}_i=\overline{\tilde{x}_i+\tilde{z}_i}$.   Then Lemma \ref{lemma:concentration2} implies 
\begin{equation*}
    \Prob_i\left( \big|\overline{y}_i-\E_i[\overline{y}_i] \big|  > \sqrt{\frac{4r^2}{N_i}}t \Bigg| N_i\right) \leq 2e^{-  t^2/16}.
\end{equation*}
By the law of iterated expectations it follows that 
\[ \Prob\left( \big|\overline {y}_i - \E_i[\overline y_i] \big| > \sqrt{\frac{4r^2}{N_i}} t  \right) \leq 2 e^{- t^2/16}.  \]
Next note that $N_i$, the number of points $y_j$ in $B(y_i,r),$ can be bounded below by $\widetilde{N}_i$, the number of points $x_j$ that lie in the ball $B_{\mathcal{M}}(x_i,r_-).$ Thus, 
\begin{equation}\label{eq:conbound1}
\Prob \biggl( \big|\bar{y}_i -    \E_i[\bar{y}_i]  \big| > \sqrt{\frac{4r^2}{\widetilde N_i }}t \biggr) \le 2 e^{-t^2/16}.
\end{equation}
Now we find probabilistic bound for $\widetilde{N}_i$. Let $w_j=\mathds{1}\{x_j \in B(x_i,r_-) \}$. Then given $x_i$, the $w_j$ are i.i.d samples from Bernoulli($q_i$), where $q_i=\mu\big(B_{\mathcal{M}}(x_i,r_-)\big)$. Lemma \ref{lemma:concentration1} implies
\begin{equation*}
    \Prob_i \left( \big|\widetilde{N}_i-nq_i \big|>nt \big|x_i \right) \leq 2e^{-2nt^2}.
\end{equation*}
Again by the law of iterated expectation and rearranging terms, we have 
\begin{equation} \label{eq:conbound2}
        \Prob \left( \widetilde{N}_i<n(q_i-t) \right) \leq 2e^{-2nt^2}.
\end{equation}
Combining \eqref{eq:conbound1} and \eqref{eq:conbound2}, we obtain 
\begin{align*}
    \Prob\left( \big| \overline{y}_i -\E_i[ \overline{y}_i] \big| > \sqrt{\frac{4r^2}{n(q_i-s)}}t \right)
    &=\Prob\left( \big| \overline{y}_i -\E_i[ \overline{y}_i] \big| > \sqrt{\frac{4r^2}{n(q_i-s)}} t, \widetilde{N}_i <n(q_i-s)\right)\\
    &+\Prob\left( \big| \overline{y}_i -\E_i[ \overline{y}_i] \big| > \sqrt{\frac{4r^2}{n(q_i-s)}}t ,\widetilde{N}_i\geq n(q_i-s)\right)\\
    &\leq \Prob \left(\widetilde{N}_i <n(q_i-s) \right) 
    + \Prob\left( \big| \overline{y}_i -\E_i[ \overline{y}_i] \big| > \sqrt{\frac{4r^2}{\widetilde{N}_i}}t \right)\\
    &\leq 2e^{-2ns^2}+2e^{- t^2/16}.
\end{align*}
Under Assumption \ref{asp:3}, \eqref{eqn:EstimateVolumeBall} implies $q_i\geq \frac{\alpha_m p_{min}}{2^{m+1}} r^m$. Taking $s=\frac{\alpha_m p_{min}}{2^{m+2}} r^m$ and $t=\sqrt{nr^{m+4}}$, we see that 
\begin{align*}
    \Prob\left( \big| \overline{y}_i -\E_i[ \overline{y}_i] \big| >\sqrt{\frac{2^{m+4}}{\alpha_m p_{min}}}r^{3} \right) \leq 2e^{-\frac{\alpha_m^2p_{min}^2}{4^{m+2}} nr^{2m}}+2e^{- nr^{m+4}/16} \leq 4e^{-cnr^{\operatorname{max}\{2m,m+4 \}}},
\end{align*}
where $c=\operatorname{min}\left\{\frac{\alpha_m^2p_{min}^2}{4^{m+2}} ,\frac{1}{16}\right\}$. 
The result then follows. 
\end{proof}
Theorem \ref{thm1} now follows by combining Lemma \ref{lem:geo}, Propositions \ref{prop:ECN}, \ref{prop:diffbias}, and Proposition \ref{prop:SE} together with a union bound.



\section{From Local Regularization to Global Estimates}
\label{sec:them2}

In this section we use the local estimates \eqref{eq:distbound} to show spectral convergence of $\Delta_{\bar{\Y}_n,\veps}$ towards the continuum Laplace-Beltrami operator. We first make some definitions. Recall that the graph $\Gamma_{\delta,\veps}=([n],W)$ has weights
\begin{align*}
    W(i,j)=\frac{2(m+2)}{\alpha_m \veps^{m+2}n}\mathds{1}\{\delta(i,j)<\veps\},
\end{align*}
where $m$ is the dimension of $\mathcal{M}$ and $\alpha_m$ is the volume of the $m-$dimensional Euclidean unit ball. For a function $u:[n]\to \R$, we denote its value on the $i$-th node as $u(i)$.  
We then define the discrete Dirichlet energy  of $u$  as \begin{align*}
    E_{\delta,\veps}[u]=\frac{m+2}{\alpha_m \veps^{m+2}n} \sum_{i=1}^n\sum_{j=1}^n  \mathds{1}\{ \delta(i,j)<\veps\} |u(i)-u(j)|^2
\end{align*}
and the $L^2$ norm of $u$ as 
\begin{align*}
    \|u\|^2=\frac{1}{n}\sum_{i=1}^n  |u(i)|^2.
\end{align*}
Given that $\Delta_{\delta, \veps} $ is a positive semi-definite operator, we can use the minimax principle to write
\begin{align*}
    \lambda_{\ell}(\Gamma_{\delta,\veps})= \underset{L}{\operatorname{min}} \underset{u\in L\backslash \{0\}}{\operatorname{max}} \frac{E_{\delta,\veps} [u]}{\|u\|^2} ,
\end{align*}
where $\lambda_{\ell}(\Gamma_{\delta,\veps})$ is the $\ell$-th smallest eigenvalue of $\Delta_{\Gamma_{\delta,\veps}}$ and the minimum is taken over all subspaces $L$ of dimension $\ell$.
The following lemma compares the eigenvalues of the discrete graphs constructed using $\delta_{\mathcal{X}_n}$ and $\delta_{\bar{\Y}_n}$. 
\begin{lemma}\label{lemma:specbound}
Let $\eta$ be the bound in \eqref{eq:distbound} so that for all $i,j$ with $d_\M(x_i, x_j) \leq r$ we have
\begin{equation*}
\big| \delta_{\mathcal{X}_n}(i,j)-\delta_{\bar{\Y}_n}(i,j) \big| \leq \eta.
\end{equation*}
Suppose that $\veps$ is chosen so that $\veps \geq2 Cm \eta$, for some universal constant $C$. Then,
\begin{align}\label{eq:specbound}
    \Big( 1-Cm\frac{\eta}{\veps}\Big) \lambda_{\ell}(\Gamma_{\mathcal{X}_n,\veps-\eta}) \leq \lambda_{\ell}(\Gamma_{\bar{\Y}_n,\veps})\leq \Big( 1+Cm\frac{\eta}{\veps}\Big) \lambda_{\ell}(\Gamma_{\mathcal{X}_n,\veps+\eta}).
\end{align}
\end{lemma}
\begin{proof}
We first compare the Dirichlet energies. Since $\delta_{\mathcal{X}_n}(i,j)<\delta_{\bar{\Y}_n}(i,j)+\eta$, we have
\begin{align}
    E_{\bar{\Y}_n,\veps}[u]&=\frac{m+2}{\alpha_m \veps^{m+2}n} \sum_i\sum_{j} \mathds{1}\{\delta_{\bar{\Y}_n}(i,j)<\veps\} |u_i-u_j|^2 \nonumber \\
    &\leq \frac{m+2}{\alpha_m \veps^{m+2}n} \sum_i\sum_{j} \mathds{1}\{\delta_{\mathcal{X}_n}(i,j)<\veps+\eta \} |u_i-u_j|^2 \nonumber \\
    &= \Big(\frac{\veps+\eta}{\veps}\Big)^{m+2} E_{\mathcal{X}_n,\veps+\eta}[u] \nonumber \\
    &\leq \Big(1+Cm\frac{\eta}{\veps} \Big)E_{\mathcal{X}_n,\veps+\eta}[u]. \label{eq:specdiff}
\end{align}
Now we use the minimax principle to show the upper-bound on \eqref{eq:specbound}. Let $u_1,\ldots,u_{\ell}$ be the first $l$ eigenvectors of $\Delta_{\mathcal{X}_n,\veps+\eta}$ and let $L=$span$\{u_1,\ldots,u_k\}$. Then dim$L=\ell$ and for any $u\in L$, $E_{\mathcal{X}_n,\veps+\eta}[u]\leq \lambda_{\ell }(\Gamma_{\mathcal{X}_n,\veps+\eta })\|u\|^2 $.  Then by \eqref{eq:specdiff}, we have
\begin{align*}
    \lambda_{\ell}(\Gamma_{\bar{\Y}_n,\veps})\leq \underset{L\backslash0}{\operatorname{max}} \frac{E_{\bar{\Y}_n,\veps} [u] }{\|u\|^2} 
    \leq \Big(1+Cm\frac{\eta}{\veps} \Big)  \underset{L\backslash0}{\operatorname{max}} \frac{E_{\mathcal{X}_n, \veps+\eta }[u] }{\|u\|^2} 
    \leq  \Big( 1+Cm\frac{\eta}{\veps}\Big) \lambda_{\ell}(\Gamma_{\mathcal{X}_n,\veps+\eta}).
\end{align*}
By a similar argument applied to $\Gamma_{\mathcal{X}_n,\veps-\eta}$ and $\Gamma_{\bar{\Y}_n,\veps}$, we get the lower-bound in \eqref{eq:specbound}. 
\end{proof} 

\begin{remark}
	With the convergence of eigenvalues and the relationship between the Dirichlet energies it is also possible to make statements about convergence of eigenvectors (or better yet, spectral projections). 
\end{remark}

The spectral convergence towards the continuum (Theorem \ref{thm2}) is a consequence of the following theorem, proved in \cite[Corollary 2]{SpecRatesTrillos}.
\begin{theorem} \label{thm:speccv}
Let $d_{\infty}$ be the $\infty$-OT distance between $\mu_n$ and $\mu$. Suppose $\veps$ satisfies the conditions in Equation \eqref{conditionseps} and that Assumptions \ref{asp:1} and \ref{asp:2} hold. Then 
\begin{align*}
    \frac{|\lambda_{\ell}(\Gamma_{\mathcal{X}_n,\veps})-\lambda_{\ell}(\mathcal{M})|}{\lambda_{\ell}(\mathcal{M})}  \leq \tilde{C} \left(\frac{d_{\infty}}{\veps}+\big(1+\sqrt{\lambda_{\ell}(\mathcal{M})}\big)\veps+\Big( K+\frac{1}{R^2}\Big)\veps^2 \right),
\end{align*}
where $\tilde{C}$ only depends on $m$ and the regularity of $p$. 
\end{theorem}
Combining Lemma \ref{lemma:specbound} and Theorem \ref{thm:speccv} gives Theorem \ref{thm2}.

\section{Numerical Experiments}
\label{sec:numerics}
In this section we present a series of numerical experiments where we conduct local regularization on three different data sets. In Subsection \ref{sec:d&s} we consider a toy example with artificial data generated by perturbing points sampled uniformly from the unit, two-dimensional sphere embedded in $\R^d$ with $d=100$. We show that the approximation of the hidden Euclidean distances between unperturbed points is significantly improved by locally regularizing the data, and that this improvement translates into better spectral approximation of the spherical Laplacian. Our numerical findings corroborate the theory developed in the previous two sections. In Subsection \ref{sec:classification} we consider the two-moon and MNIST data sets and show that graphs constructed with locally regularized data can be used to improve the performance of a simple graph-based optimization method for semi-supervised classification. 

\subsection{Distance \& Spectrum} 
\label{sec:d&s}
In this subsection we study the effect of local regularization on distance approximation and spectral convergence, as an illustration of the results from Sections \ref{sec:them1} and \ref{sec:them2}. In our toy model we consider uniform samples from the unit two-dimensional sphere $\M = \mathcal{S}$ embedded in $\mathbb{R}^d$, with $d=100$. The motivation for such a choice is that the eigenvalues of the associated Laplace-Beltrami operator on $\mathcal{S}$ are known explicitly. Indeed, after appropriate normalization, $\Delta_{\mathcal{S}}$ admits eigenvalues $\ell(\ell+1), \ell \in \mathbb{N}$, with corresponding multiplicity $2\ell+1$. 

The data set is generated by sampling $n=3000$ points $x_i$ uniformly from the sphere and adding uniform noise $z_i$ that is bounded by $\sigma$ in norm. Local regularization is performed by taking $r\propto \sqrt{\sigma}$ and the graph is constructed with $\veps=2n^{-1/4}$. The actual proportion constant in $r$ is not obvious from our theory and in the experiments below we choose $r=\sqrt{\sigma}/3$ for $\sigma=0.1$ and $r=\sqrt{\sigma}$ for the rest of the $\sigma$'s. We first show that the $\bar{y}_i$ give a better approximation of the pairwise distances of the $x_i$ than the $y_i$ do. We only consider those nodes $i,j$ such that $\delta_{\mathcal{X}_n}(i,j)<\veps$ (i.e. the nodes that are relevant for the construction of the graph Laplacians). More precisely, let $D_{\mathcal{X}_n}$ be the matrix whose $ij$th entry is $\delta_{\mathcal{X}_n}(i,j)\mathds{1}\{\delta_{\mathcal{X}_n}(i,j)<\veps\}$.  Similarly, we define $[D_{\Y_n}]_{ij}=\delta_{\mathcal{Y}_n}(i,j)\mathds{1}\{\delta_{\mathcal{X}_n}(i,j)<\veps\}$ and $[D_{\bar{\Y}_n}]_{ij}=\delta_{\bar{\Y}_n}(i,j)\mathds{1}\{\delta_{\mathcal{X}_n}(i,j)<\veps\}$. In Table \ref{table:1} we compare the Frobenius norm of the $D_{\mathcal{X}_n}-D_{\Y_n}$ and $D_{\mathcal{X}_n}-D_{\bar{\Y}_n}$ for different values of $\sigma$. We see that the improvement is substantial. 

\begin{table}[h!]
\centering
\begin{tabular}{ |c|c|c|c|c|c|c|c|c|c| } 
 \hline
 &$\sigma=0.1$ & $\sigma=0.2$ & $ \sigma=0.3$ & $\sigma=0.4$ & $\sigma=0.5$ &  $ \sigma=0.6$ &
 $\sigma=0.7$ & $\sigma=0.8$  & $\sigma=0.9$ \\
 \hline
 $\|D_{\mathcal{X}_n}-D_{\Y_n}\|_F$ &1.33&4.18&7.97&12.11&16.59&21.15&25.78&30.47&35.53 \\ 
 \hline
 $\|D_{\mathcal{X}_n}-D_{\bar{\Y}_n}\|_F$ & 0.73& 1.49 & 1.44 & 1.70 & 1.74 & 1.85 & 1.86 & 2.01&2.16\\
\hline
\end{tabular}
\caption{Frobenius norm of $D_{\mathcal{X}_n}-D_{\Y_n}$ and $D_{\mathcal{X}_n}-D_{\bar{\Y}_n}$ on $\mathcal{S}$ for several values of $\sigma.$}
\label{table:1}
\end{table}

Next we study the spectral approximation of Laplacians by comparing the spectra of $\Delta_{\mathcal{X}_n,\veps}$, $\Delta_{\Y_n,\veps}$ with that of $\Delta_{\bar{\Y}_n,\veps}$. Note that since the $x_i$ are uniformly distributed, the density $p$ on $\mathcal{S}$ that they are sampled from is constant and equal to $\frac{1}{\text{vol}{\mathcal{M}}}$. So for the spectra of the graph Laplacians to match in scale with that of $\Delta_{\mathcal{S}}$, the weights should be rescaled according to 
\begin{align*}
    W(i,j) =\frac{2(m+2)\text{vol}(\mathcal{M})}{\alpha_m\veps^{m+2 }n},
\end{align*}
where vol$(\mathcal{M})$ is the volume of the manifold and equals $4\pi$ in this case. In Figure \ref{figure:1} we compare the first 100 eigenvalues of $\Delta_{\mathcal{X}_n,\veps}$, $\Delta_{\Y_n,\veps}$, and $\Delta_{\bar{\Y}_n,\veps}$ with the continuum spectrum.  We see that when the noise size is large, the Euclidean graph Laplacian $\Delta_{\Y_n,\veps}$ does not give a meaningful approximation of the continuum spectrum, while the locally regularized version $\Delta_{\bar{\Y}_n,\veps}$ still performs well. 

\begin{figure}[!htb]
\minipage{0.32\textwidth}
  \includegraphics[width=\linewidth]{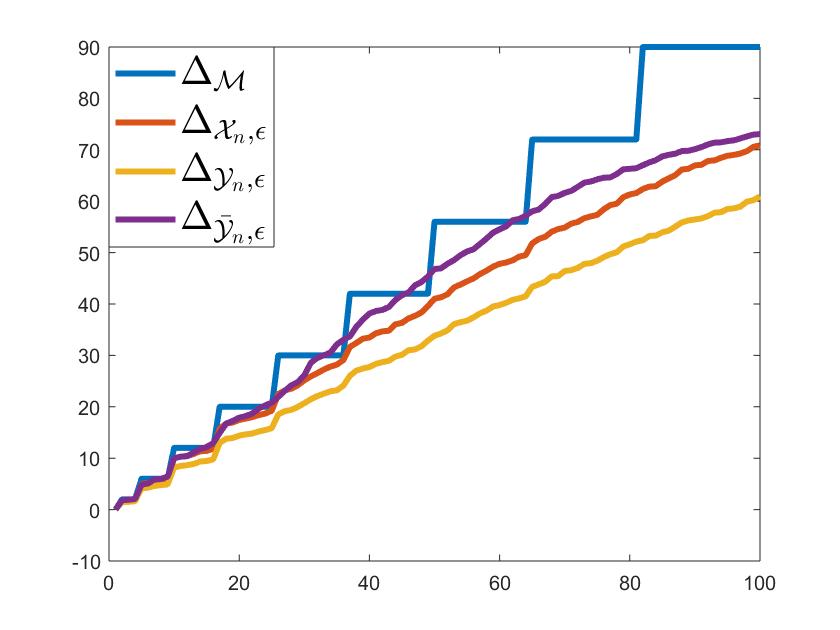}
\subcaption{$\sigma=0.1$}\label{fig:awesome_image1}
\endminipage\hfill
\minipage{0.32\textwidth}
  \includegraphics[width=\linewidth]{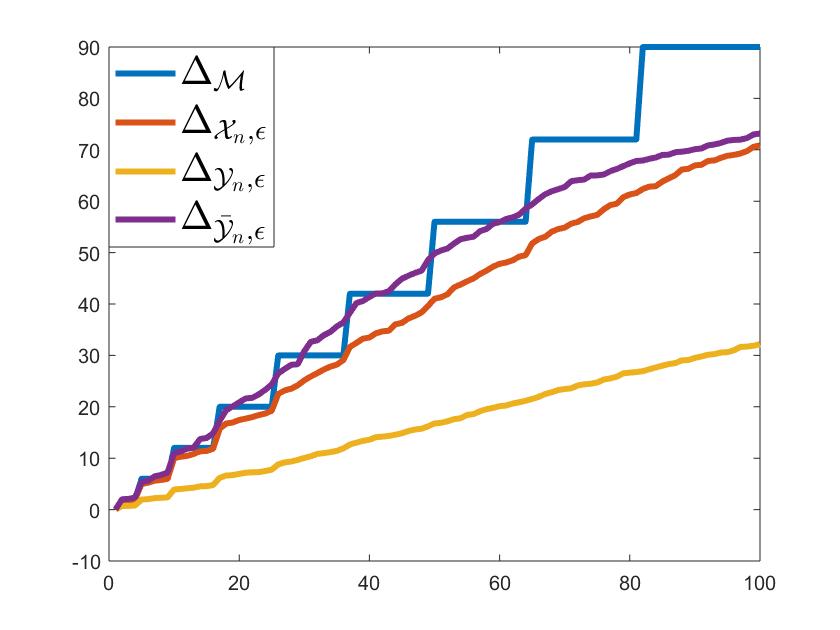}
\subcaption{$\sigma=0.2$}\label{fig:awesome_image2}
\endminipage\hfill
\minipage{0.32\textwidth}
  \includegraphics[width=\linewidth]{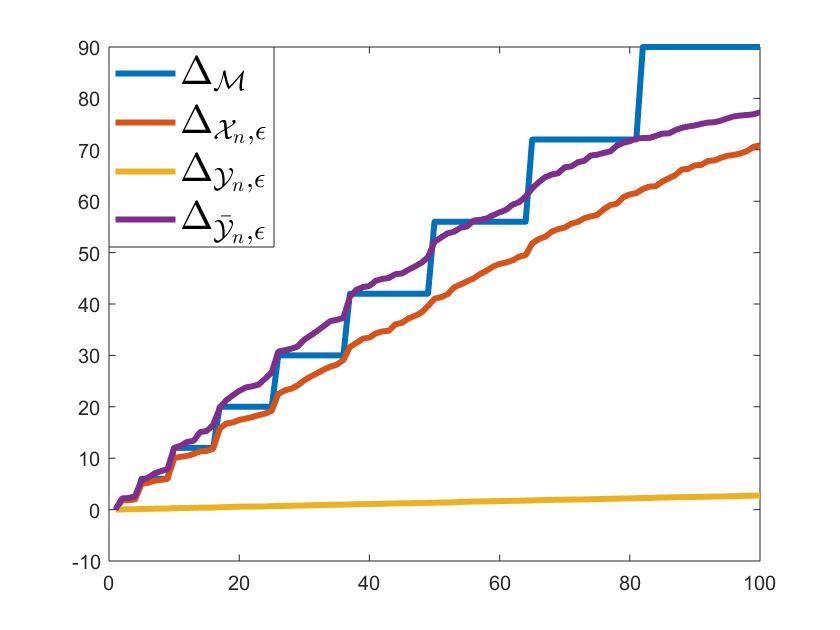}
\subcaption{$\sigma=0.3$}\label{fig:awesome_image2}
\endminipage\hfill
\minipage{0.32\textwidth}
  \includegraphics[width=\linewidth]{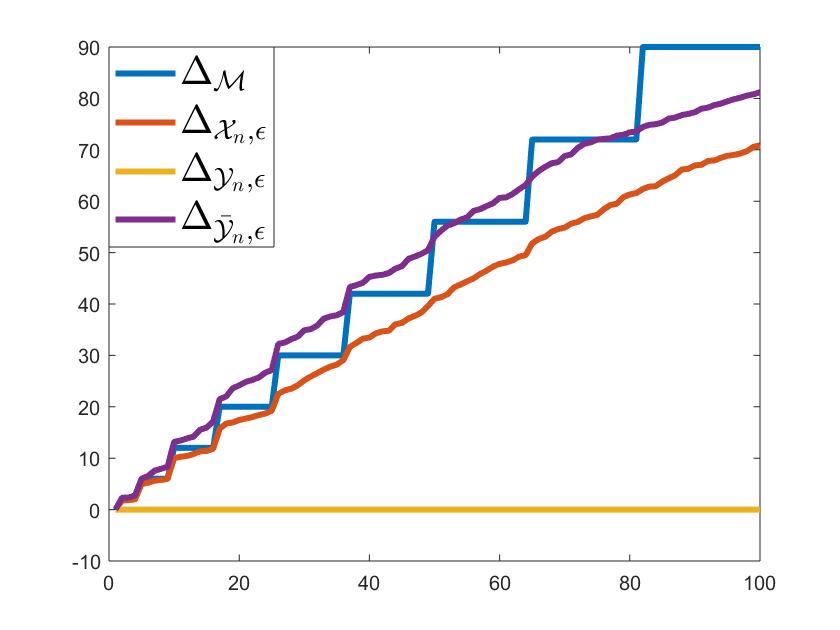}
\subcaption{$\sigma=0.4$}\label{fig:awesome_image2}
\endminipage\hfill
\minipage{0.32\textwidth}
  \includegraphics[width=\linewidth]{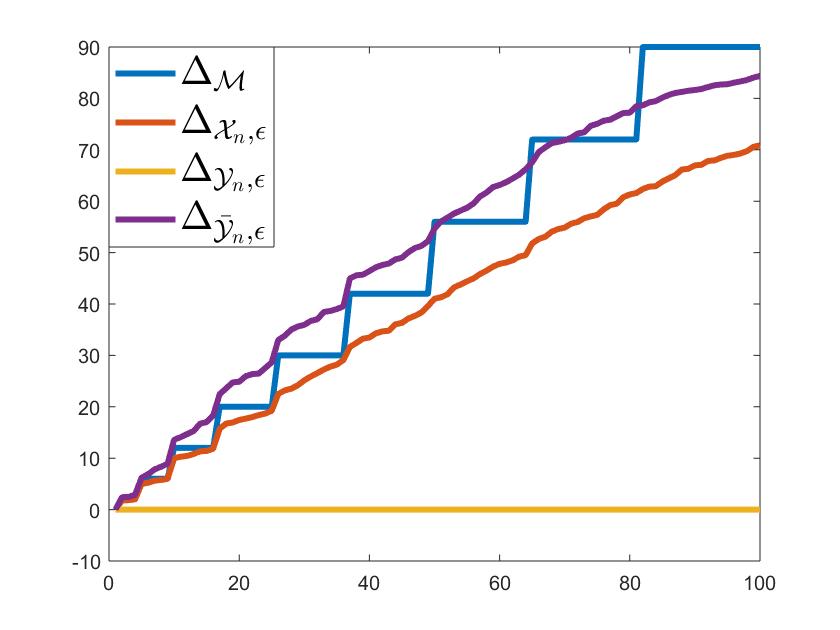}
\subcaption{$\sigma=0.5$}\label{fig:awesome_image2}
\endminipage\hfill
\minipage{0.32\textwidth}
  \includegraphics[width=\linewidth]{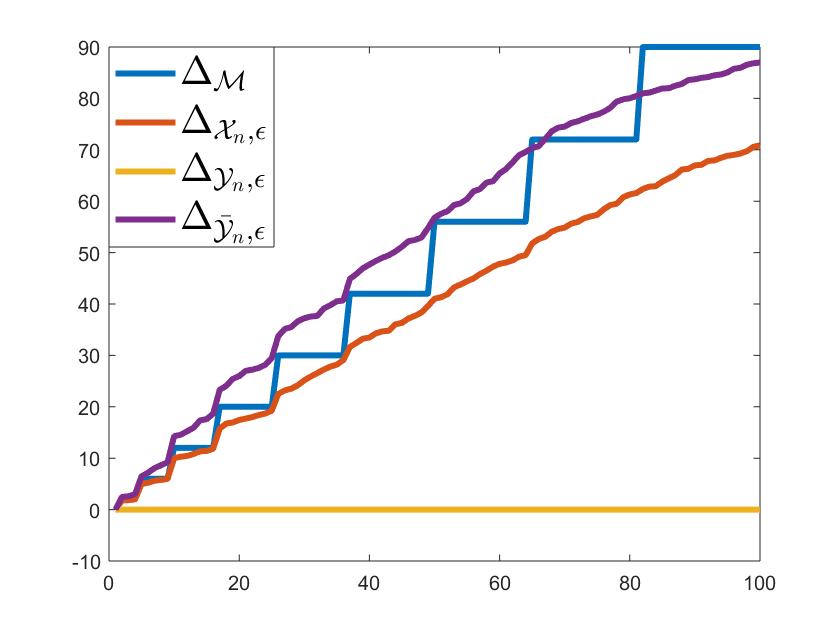}
\subcaption{$\sigma=0.6$}\label{fig:awesome_image2}
\endminipage\hfill
\minipage{0.32\textwidth}
  \includegraphics[width=\linewidth]{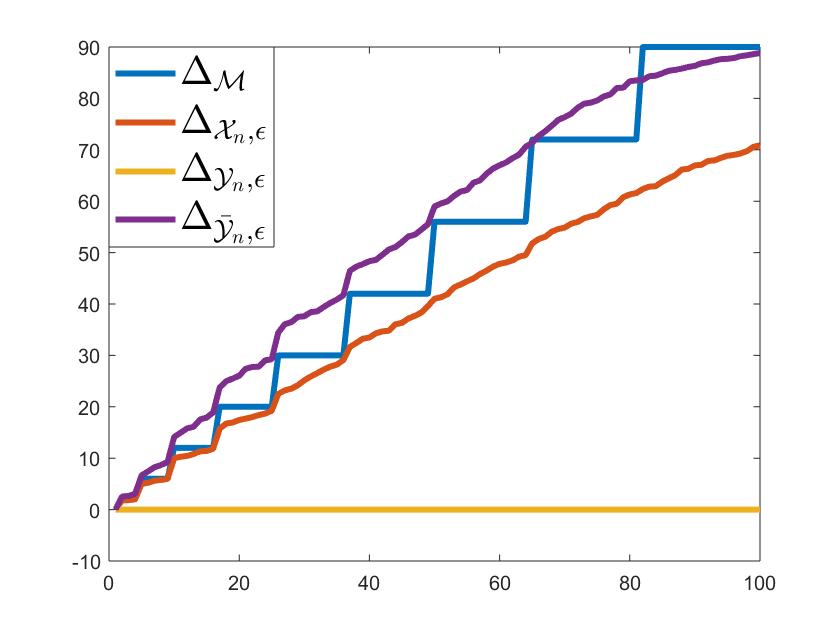}
\subcaption{$\sigma=0.7$}\label{fig:awesome_image2}
\endminipage\hfill
\minipage{0.32\textwidth}
  \includegraphics[width=\linewidth]{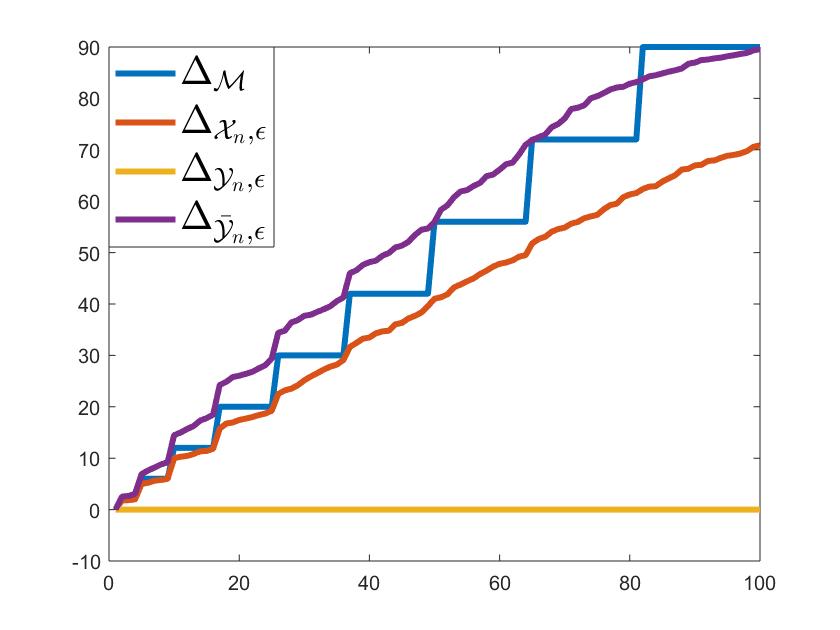}
\subcaption{$\sigma=0.8$}\label{fig:awesome_image2}
\endminipage\hfill
\minipage{0.32\textwidth}
  \includegraphics[width=\linewidth]{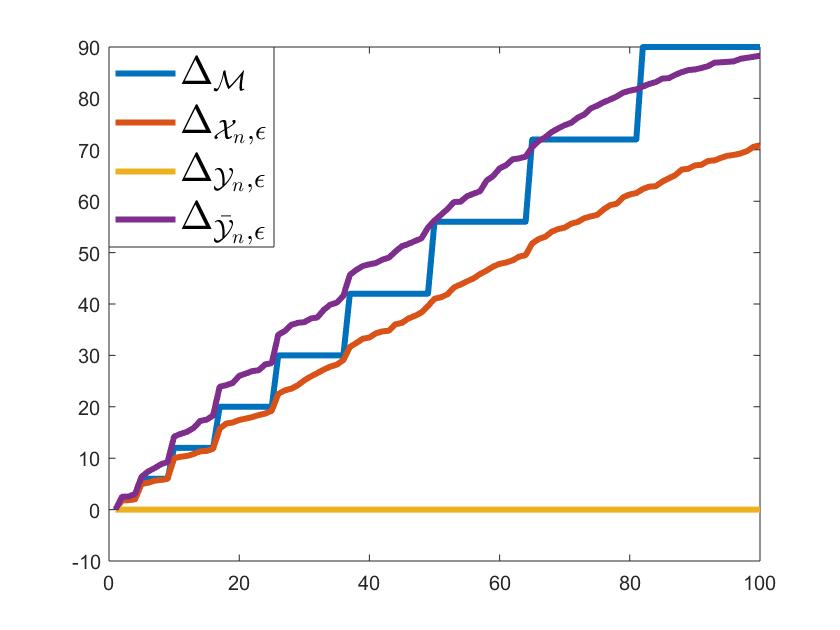}
\subcaption{$\sigma=0.9$}\label{fig:awesome_image2}
\endminipage
\caption{Comparison of spectra of continuum Laplacian, $\Delta_{\mathcal{X}_n,\veps}$, $\Delta_{\Y_n,\veps}$ and $\Delta_{\bar{\Y}_n,\veps}$ for different values of $\sigma$.}
\label{figure:1}
\end{figure}

\begin{remark} \label{rmk:r}
While our theory in Section \ref{sec:them1} suggests the choice that $r\propto \sqrt{\sigma}$ in the small $r$ and large $n$ limit, for practical purposes some other scalings may give better results. Indeed we want to remark that for the above $\sigma$'s, choosing $r=\sigma$ seems to give better  spectral approximation.  The choice of the local-regularization parameter will be further investigated in the subsection \ref{sec:MNIST} in the context of a classification task, where a data-driven (cross-validation) approach can be used.    
\end{remark}

\subsection{Classification} 
\label{sec:classification}
In this subsection we  demonstrate the practical use of local regularization by applying it to classification problems. To show the potential benefits, we consider synthetic and real data sets, namely the two moons and MNIST data sets.  Since in one of our experiments we study a  real data set, where in general the connectivity parameter in an $\veps$ graph is hard to tune, we instead consider fully connected graphs with self-tuning weights. Precisely, given a similarity $\delta: [n] \times [n] \to [0,\infty)$ we define, following \cite{zelnik2005self}, the weights by 
\begin{align*}
    W(i,j)=\exp\left(-\frac{\delta(i,j)^2}{2\tau(i)\tau(j)}\right),
\end{align*}
where $\tau(i)$ is the similarity between the $i$-th data point and its $K$-th nearest neighbor with respect to the distance $\delta$. As before, we denote by $\Gamma_{\mathcal{X}_n},$ $\Gamma_{\Y_n}$ and $\Gamma_{\bar{\Y}_n}$ the graphs constructed with similarities $\delta_{\mathcal{X}_n}, \delta_{\Y_n}$, and $\delta_{\bar{\Y}_n}$.  Instead of specifying a universal $\veps$ representing the connectivity length-scale, the neighborhood for each point is selected from using the local geometry which varies in space. It amounts to choosing different values of $\veps$ adaptively depending on the local scale, as proposed in \cite{zelnik2005self}. Since the $\tau(i)$ are defined by considering $K$-nearest neighbors, a natural variant of the above fully connected graph is to set the weights to be $0$ whenever $x_i$ and $x_j$ are not among the $K$-nearest neighbors of each other. In other words, we can construct a (symmetrized) $K$-NN graph with the same $K$ as in the definition of $\tau(i)$ and the nonzero weights are the same as above. It turns out that empirically this $K$-NN version can improve the classification performance substantially, but to illustrate the local regularization idea, we will present results for both graph constructions. We shall denote these two types of graphs as fully-connected and $K$-NN variants for brevity. \par

In the following, we focus on the semi-supervised learning setting where we are given $n$ data points with the first $J$ being labeled. The classification is done by minimizing a probit functional as explained below. Let $\Delta_{\delta}$ be a normalized graph Laplacian constructed on the data set, which will be constructed using $\mathcal{X}_n$, $\Y_n$ and $\bar{\Y}_n$ and $\Delta_{\delta}=I-D^{-1/2}WD^{-1/2}$ as compared with \eqref{eq:graphlaplacian}. Let $(\lambda_i,q_i)$, $i=1,\ldots,n$ be the associated eigenvalue-eigenvector pairs, and let $U=\text{span}\{q_2,\ldots,q_n\}$. The classifier is set to be the sign of the minimizer $u$ of the functional 
\begin{align*}
   \mathcal{J}(u):=\frac{1}{2c}\langle u,\Delta_{\delta}u\rangle -\sum_{j=1}^J \log \Big( \Phi ( y(j)u(j) ;\gamma) \Big), \; \text{with} \;  c:=n\Big(\sum_{i=2}^n \lambda_i^{-1}\Big)^{-1}, 
\end{align*}
where $\{y(j)\}_{j =1}^J$ is the vector of labels and $\Phi$ is the cdf of $\mathcal{N}(0,\gamma^2)$.  The functional $\mathcal{J}$ can be interpreted as the negative log posterior in a Bayesian setting, as discussed in \cite{bertozzi2018uncertainty}.  Throughout our experiments we set $\gamma=0.1.$ 

\subsubsection{Two Moons} \label{sec:twomoons}
We first study the two moons data set, which is generated by sampling points uniformly from two semi-circles of unit radius centered at $(0,0)$ and $(1,0.5)$ and then embedding the data set in $\mathbb{R}^d$, with $d=100$. We then perturb the data by adding uniform noise with norm bounded by $\sigma$. In addition to the semi-supervised setting, we also examine the unsupervised case.

We consider $n=1000$ points $1\%$ of which have labels and we set $K=10$. As pointed out in Remark \ref{rmk:r}, we choose the regularization parameter $r$ to be equal to $\sigma$. 
We compare the approximation of distance matrix and classification performance on $\mathcal{X}_n,\Y_n$, and $\bar{\Y}_n$'s, as in Table \ref{table:2} and \ref{table:3}. Instead of comparing nodes that are within $\delta_{\mathcal{X}_n}$-distance $\veps$, we consider nodes that are $K$-nearest neighbors of each other with respect to $\delta_{\mathcal{X}_n}$.  As before, the regularized points $\bar{\Y}_n$ approximate the pairwise distances better. Moreover, in terms of classification, we see that $\Gamma_{\bar{\Y}_n}$ is able to capture the exact correct labeling as the clean data $\mathcal{X}_n$ does.

\begin{table}[h!]
\centering
\begin{tabular}{ |c|c|c|c|c|c|c|c| } 
 \hline
  & $\sigma=0.1 $  &$\sigma=0.2$& $\sigma=0.3$& $\sigma=0.4$ & $\sigma=0.5$ & $\sigma=0.6$ & $\sigma=0.7$\\
 \hline
 $\|D_{\mathcal{X}_n}-D_{\Y_n}\|_F$ & 0.78&	1.64&	2.49&	3.42&	4.26&	5.15&6.05\\ 
 \hline
 $\|D_{\mathcal{X}_n}-D_{\bar{\Y}_n}\|_F$& 0.13& 0.23&	0.31&	0.37&	0.47&	0.63 &	0.65	\\ 
 \hline
\end{tabular}
\caption{Frobenius norm of $D_{\mathcal{X}_n}-D_{\Y_n}$ and $D_{\mathcal{X}_n}-D_{\bar{\Y}_n}$ on two moons for different values of $\sigma$.} 
\label{table:2}
\end{table}
\begin{table}[h!]
\centering
\begin{tabular}{ |c|c|c|c|c|c|c|c| } 
 \hline
  & $\sigma=0.1 $  &$\sigma=0.2$& $\sigma=0.3$& $\sigma=0.4$ & $\sigma=0.5$ & $\sigma=0.6$ & $\sigma=0.7$\\
 \hline
 $\Gamma_{\mathcal{X}_n}$& 0 & 0 & 0& 0&0 &0 &0  \\ 
 \hline
 $\Gamma_{\Y_n}$ & 0 & 0 & 60& 137& 183& 198 & 218 \\ 
 \hline
 $\Gamma_{\bar{\Y}_n}$ & 0 & 0 & 0 & 0&0 &0 &0\\ 
 \hline
\end{tabular}
\caption{Classification error of $\Gamma_{\mathcal{X}_n}$, $\Gamma_{\Y_n}$ and $\Gamma_{\bar{\Y}_n}$ on two moons for different values of $\sigma$.}
\label{table:3}
\end{table}

\begin{figure}[!htb]
\minipage{1\textwidth}
\minipage{0.32\textwidth}
  \includegraphics[width=\linewidth]{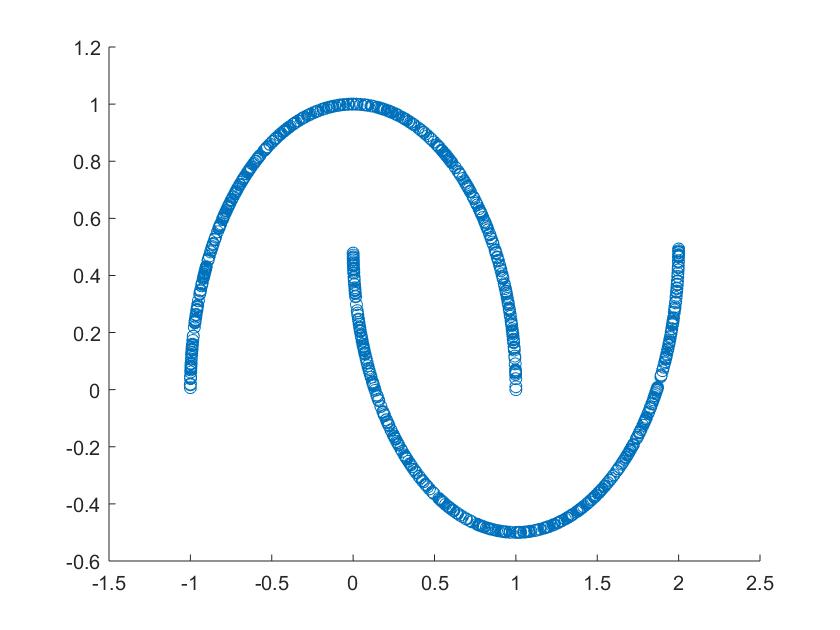}
\label{fig:awesome_image1}
\endminipage\hfill
\minipage{0.32\textwidth}
  \includegraphics[width=\linewidth]{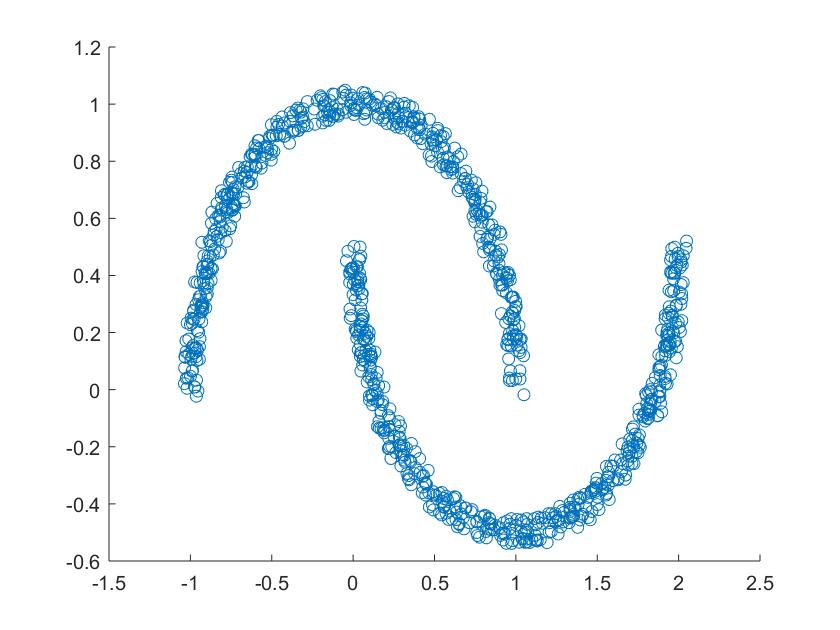}
\label{fig:awesome_image2}
\endminipage\hfill
\minipage{0.32\textwidth}
  \includegraphics[width=\linewidth]{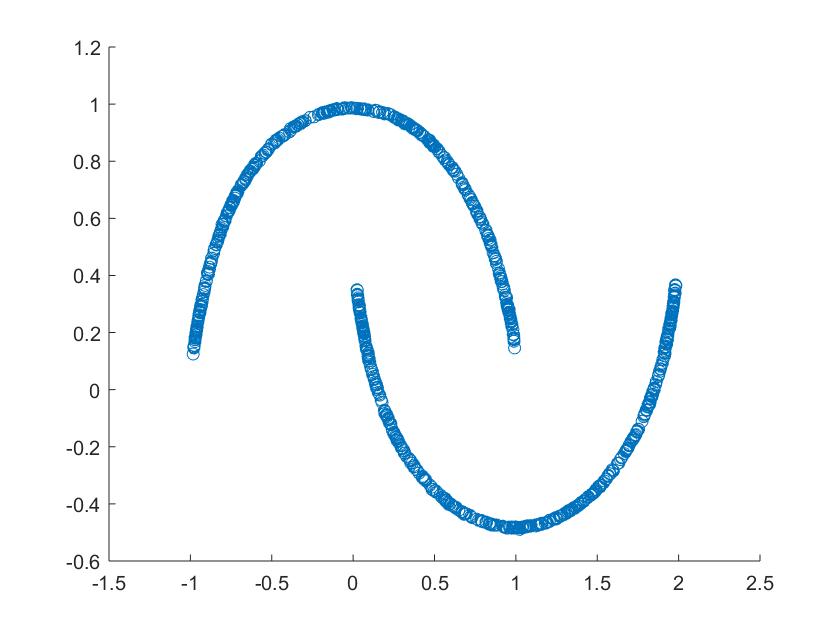}
\label{fig:awesome_image2}
\endminipage
\subcaption{$\sigma=0.5$.}
\endminipage\hfill
\minipage{1\textwidth}
\minipage{0.32\textwidth}
  \includegraphics[width=\linewidth]{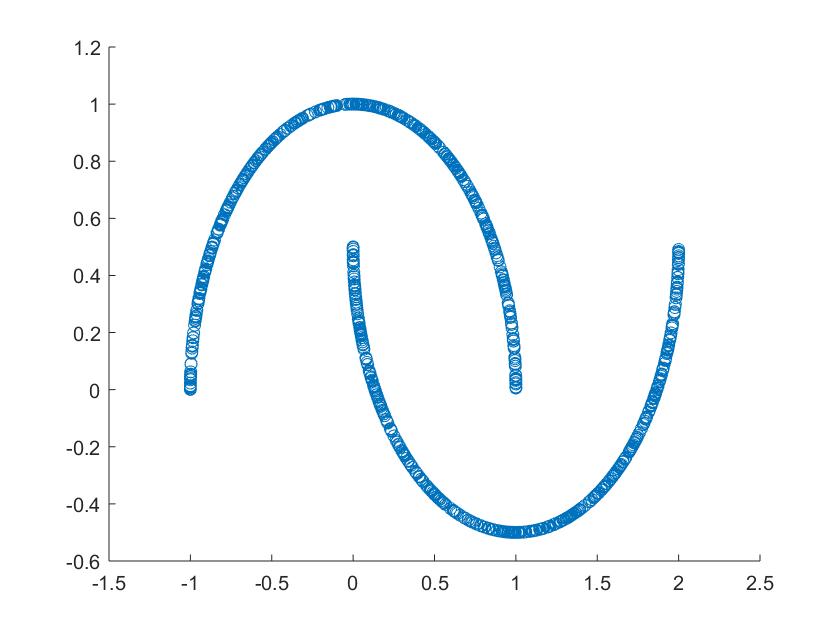}
\label{fig:awesome_image2}
\endminipage\hfill
\minipage{0.32\textwidth}
  \includegraphics[width=\linewidth]{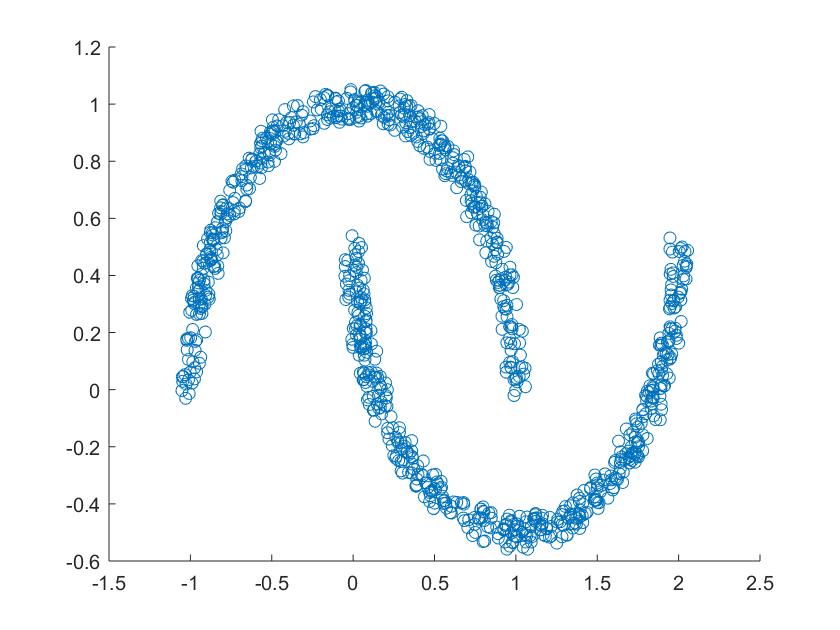}
\label{fig:awesome_image2}
\endminipage\hfill
\minipage{0.32\textwidth}
  \includegraphics[width=\linewidth]{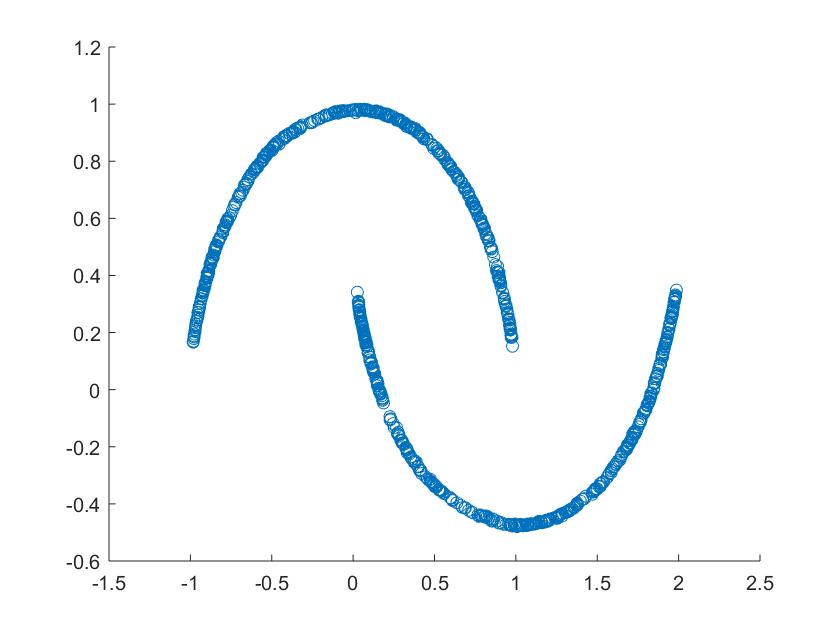}
\label{fig:awesome_image2}
\endminipage
\subcaption{$\sigma=0.6$.}
\endminipage\hfill
\minipage{1\textwidth}
\minipage{0.32\textwidth}
  \includegraphics[width=\linewidth]{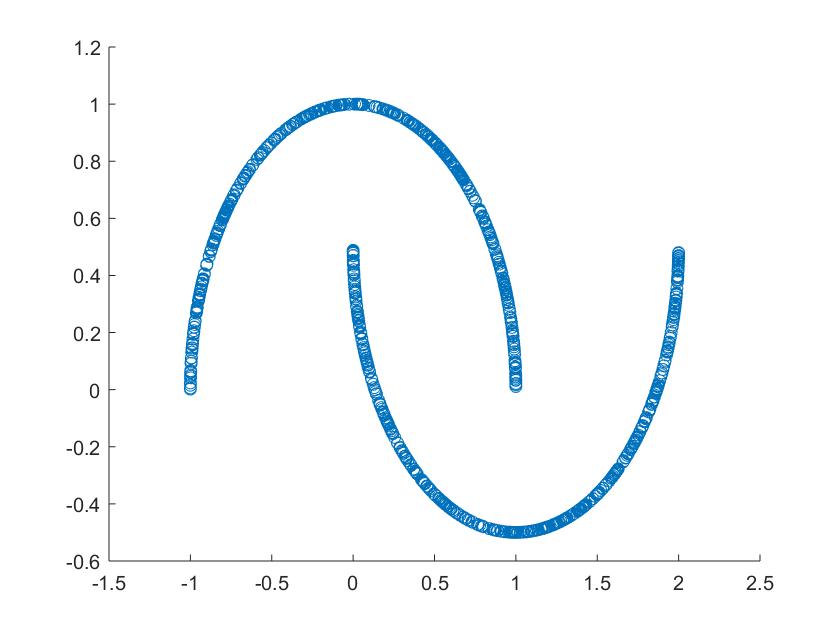}
\label{fig:awesome_image2}
\endminipage\hfill
\minipage{0.32\textwidth}
  \includegraphics[width=\linewidth]{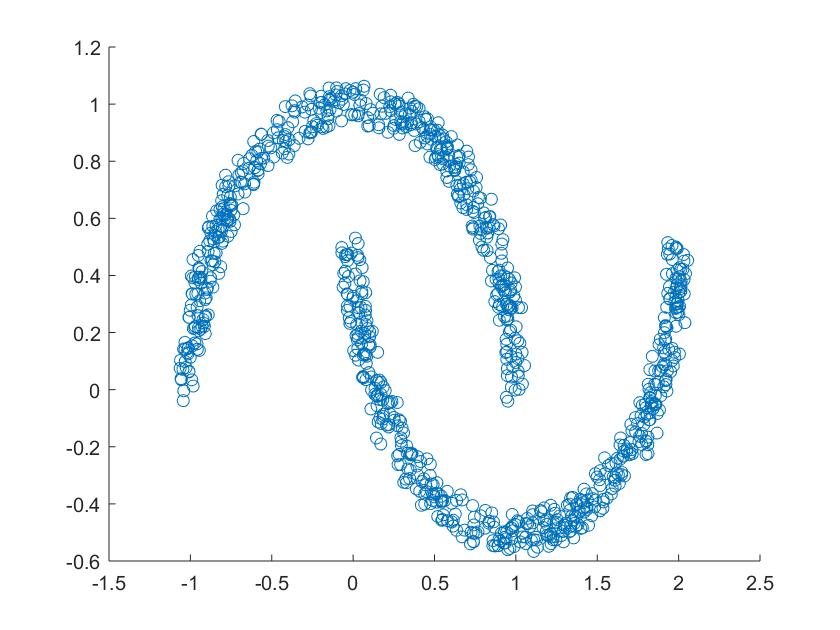}
\label{fig:awesome_image2}
\endminipage\hfill
\minipage{0.32\textwidth}
  \includegraphics[width=\linewidth]{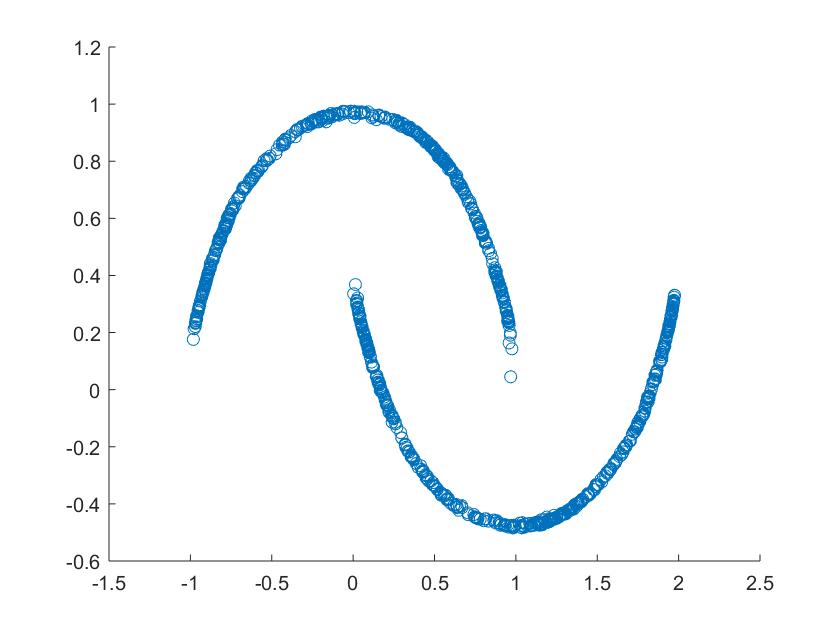}
\label{fig:awesome_image2}
\endminipage
\subcaption{$\sigma=0.7$.}
\endminipage\hfill 
\caption{Visualization of the point clouds $\mathcal{X}_n$, $\Y_n$, and $\bar{\Y}_n$. Each row contains scatter plots of the first two coordinates of the points in the data sets $\mathcal{X}_n$, $\Y_n$, and $\bar{\Y}_n$. }
\label{figure:2}
\end{figure}

For further understanding, in Figure \ref{figure:2} we plot the first two coordinates of the points in $\mathcal{X}_n$, $\Y_n$ and $\bar{\Y}_n$ for large values of $\sigma$. We  see that after local regularization, the first two coordinates of $\bar{\Y}_n$ lie almost on the underlying manifold. The denoising effect of local regularization is apparent. Furthermore, we observe that the semicircles for $\bar{\Y}_n$ are ``shorter" than those of $\mathcal{X}_n$. In other words, points near the ends are pulled away from the boundaries.  Moreover, if one looks carefully at the plots for $\bar{\Y}_n$, points are denser near the top and bottom. This illustrates that local regularization not only reduces noise, but also moves points to regions of high probability. We refer to \cite{chen2016comprehensive,fukunaga1975estimation} and the references therein for some discussion on mean-shift and mode-seeking type algorithms. 

\begin{remark}
The above results were obtained using the fully-connected graph. When instead its $K$-NN variant was used, we also obtained 100\% classification accuracy for $\Gamma_{\Y_n}$. Moreover, when we removed the labels and simply did spectral clustering with the $K$-NN variant, we still got 100\% correctness for all $\Gamma_{\mathcal{X}_n}$, $\Gamma_{\Y_n}$ and $\Gamma_{\bar{\Y}_n}$.
\end{remark}

\subsubsection{MNIST} \label{sec:MNIST}
In this subsection we apply local regularization on a real data set, namely MNIST. Unlike the previous examples, we do not have access to an underlying manifold. Instead of adding additional noise to the data set, we directly apply local regularization and show that by doing so we get better classification performance. \par
Each image is seen as a noisy data point in $\mathbb{R}^{784}$.  Here we only focus on the binary classification of pairs of digits. Since we have no prior knowledge on the noise size, choosing the localization parameter $r$ becomes difficult. In these experiments, we perform $2$-fold cross validation on the label sets. When there are few labels, we repeatedly  generate holdout sets and compare the overall error. Due to this practical difficulty of tuning $r$, we propose two variants of $\Gamma_{\bar{\Y}_n}$ that can serve as alternatives in practice. 

We study the classification performance of $\Gamma_{\bar{\Y}_n}$ for different pairs of digits. We consider $n=1000$ images and $K=20$. Since this is a semi-supervised setting, it is also of interest to see how the number of labels affects the classification. We first consider $4\%$ labels on four different pairs of digits and then examine the pair $4\&9$ more closely by adding more labels. Table \ref{table:4} and \ref{table:5} show the corresponding results.

\begin{table}[h!]
\centering
\begin{tabular}{ |c|c|c|c|c|c|c|c|c|c| } 
 \hline
  Fully-connected& 3\&8  & 5\&8 & 4\&9 & 7\&9  & $K$-NN variant & 3\&8 &5\&8 & 4\&9& 7\&9\\
 \hline
 $\Gamma_{\Y_n}$ &277 &480 & 480&480 &$\Gamma_{\Y_n}$ & 76& 55& 133& 73\\ 
 \hline
 $\Gamma_{\bar{\Y}_n}$ & 134 &174 &300 & 153& $\Gamma_{\bar{\Y}_n}$ & 60 &36 & 96& 54\\ 
 \hline
\end{tabular}
\caption{Classification error for different pairs of digits 3\&8, 5\&8, 4\&9, and 7\&9.}
\label{table:4}
\end{table}

     \begin{table}[h!]
\centering
\begin{tabular}{ |c|c|c|c|c|c|c|c|c|c| } 
 \hline
  Fully-connected& 4\%  & 8\% &  12\% & 16\%  & $K$-NN variant  & 4\%&8\% &12\% & 16\% \\
 \hline
 $\Gamma_{\Y_n}$ & 480&427 & 388& 294& $\Gamma_{\Y_n}$&133 &109 &76 &51\\ 
 \hline
 $\Gamma_{\bar{\Y}_n}$ & 300&261 & 219& 182& $\Gamma_{\bar{\Y}_n}$& 96&64 &60 &45\\ 
 \hline
\end{tabular}
\caption{Classification error for 4\&9 with different number of labels.}
\label{table:5}
\end{table}

\begin{figure}[h!]  
    \centering
    \minipage{0.48\textwidth}
    \minipage{0.09\textwidth}
  \includegraphics[width=\linewidth]{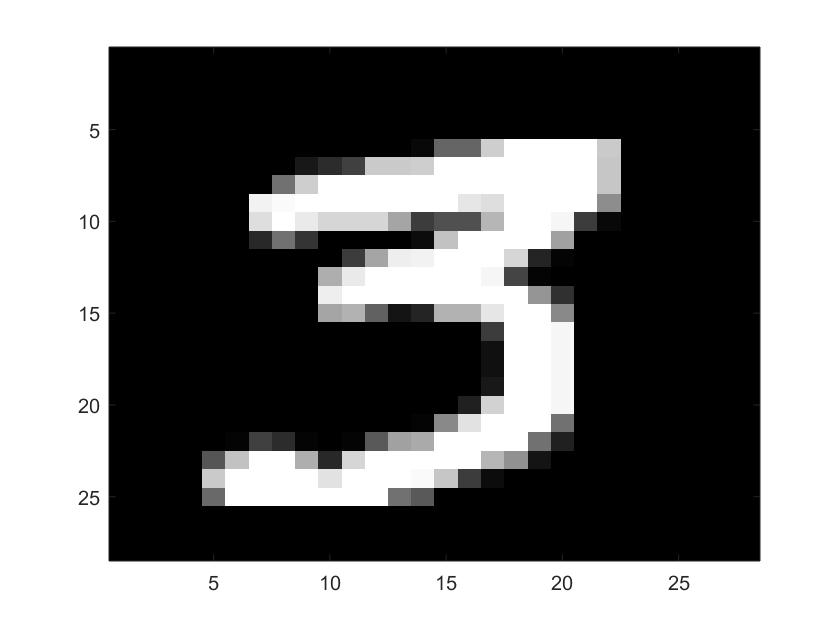}
\endminipage\hfill
\minipage{0.09\textwidth}
  \includegraphics[width=\linewidth]{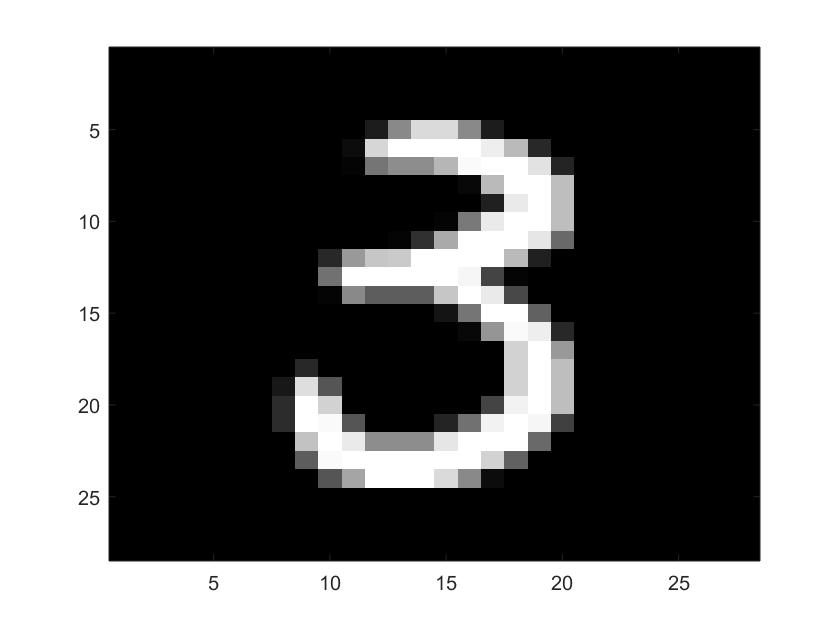}
\endminipage\hfill
\minipage{0.09\textwidth}
  \includegraphics[width=\linewidth]{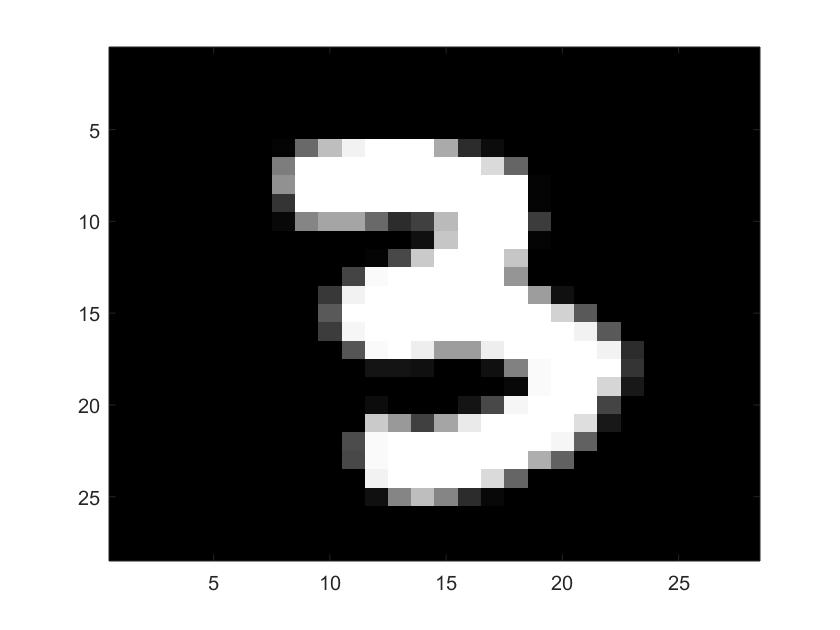}
\endminipage\hfill
\minipage{0.09\textwidth}
  \includegraphics[width=\linewidth]{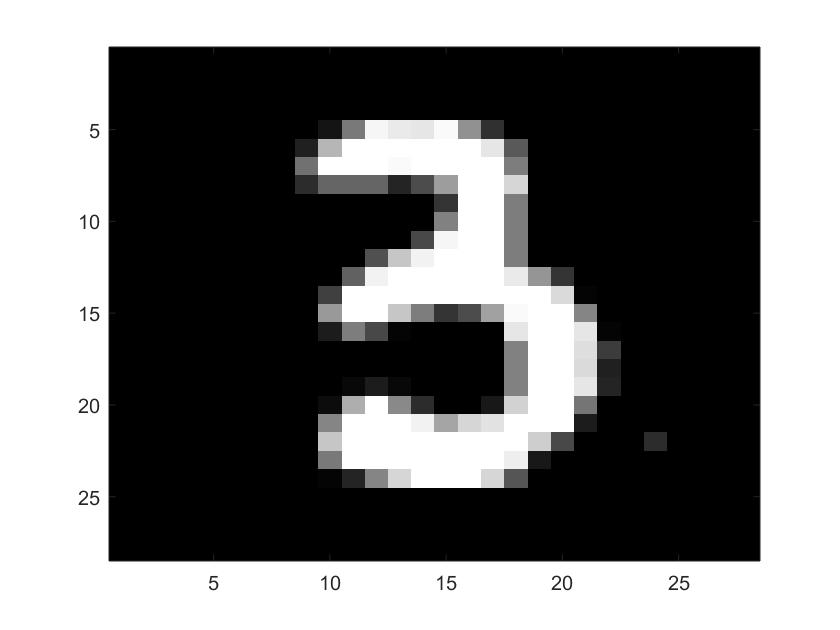}
\endminipage\hfill
\minipage{0.09\textwidth}
  \includegraphics[width=\linewidth]{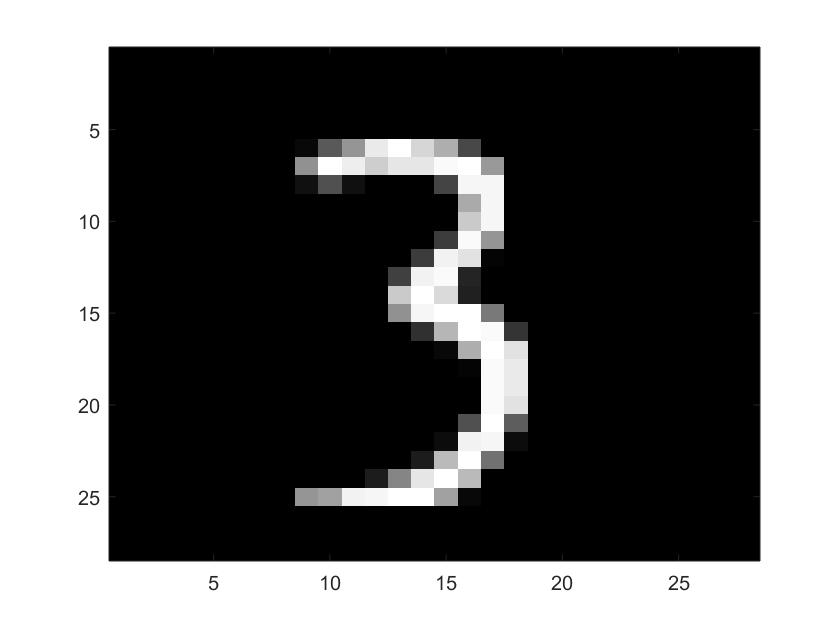}
\endminipage\hfill
\minipage{0.09\textwidth}
  \includegraphics[width=\linewidth]{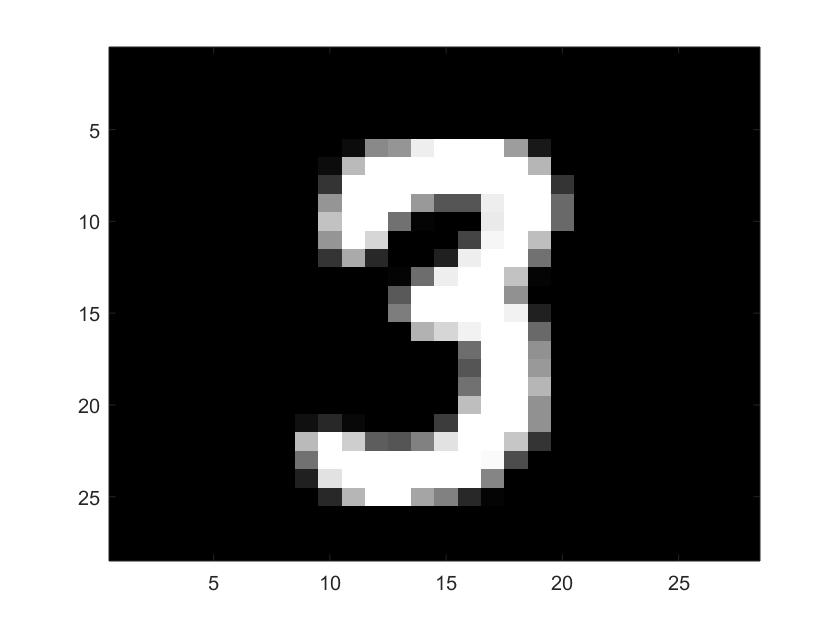}
\endminipage\hfill
\minipage{0.09\textwidth}
  \includegraphics[width=\linewidth]{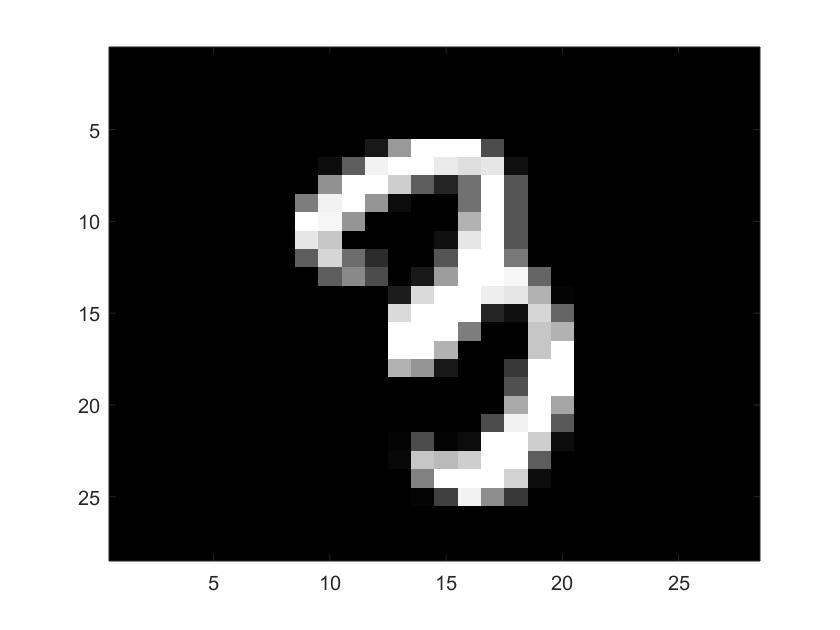}
\endminipage\hfill
\minipage{0.09\textwidth}
  \includegraphics[width=\linewidth]{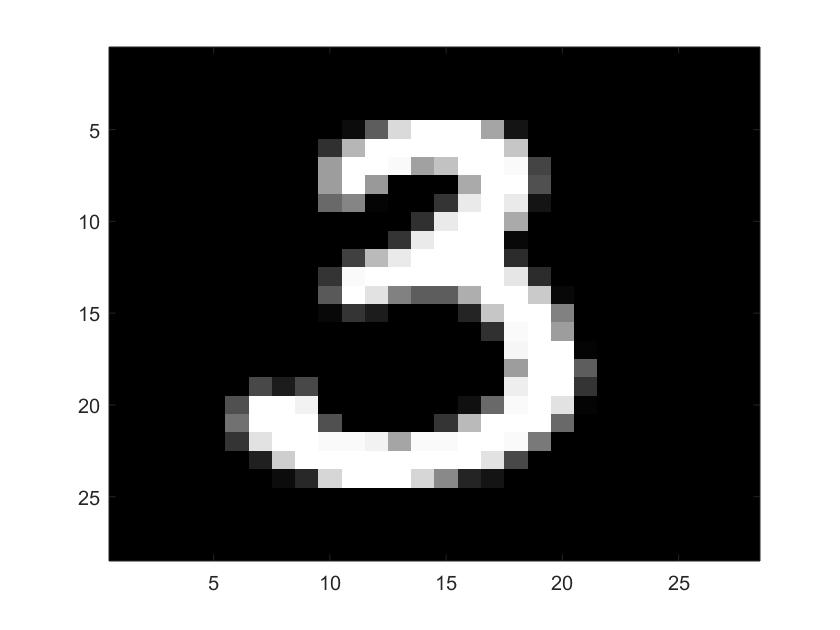}
\endminipage\hfill
\minipage{0.09\textwidth}
  \includegraphics[width=\linewidth]{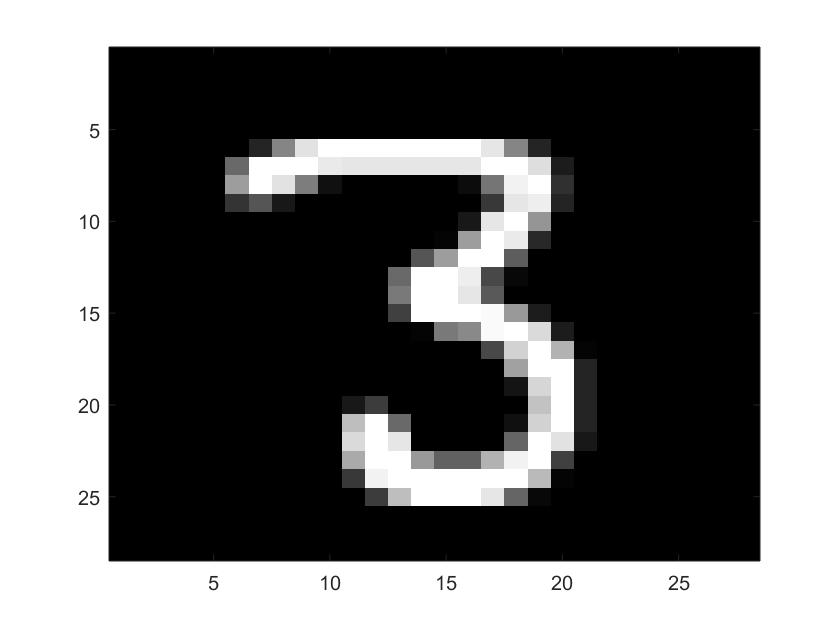}
\endminipage\hfill
\minipage{0.09\textwidth}
  \includegraphics[width=\linewidth]{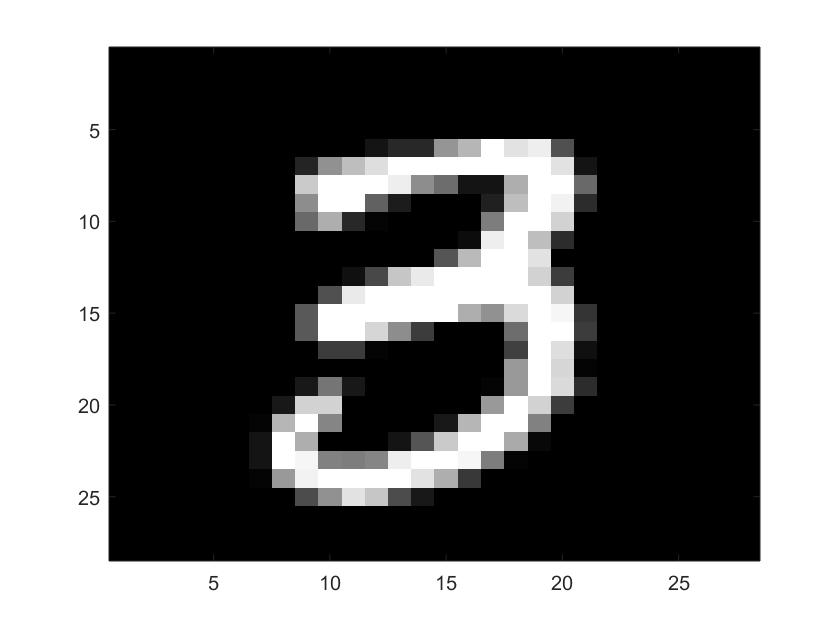}
\endminipage
\endminipage\hfill
 \minipage{0.48\textwidth}
    \minipage{0.09\textwidth}
  \includegraphics[width=\linewidth]{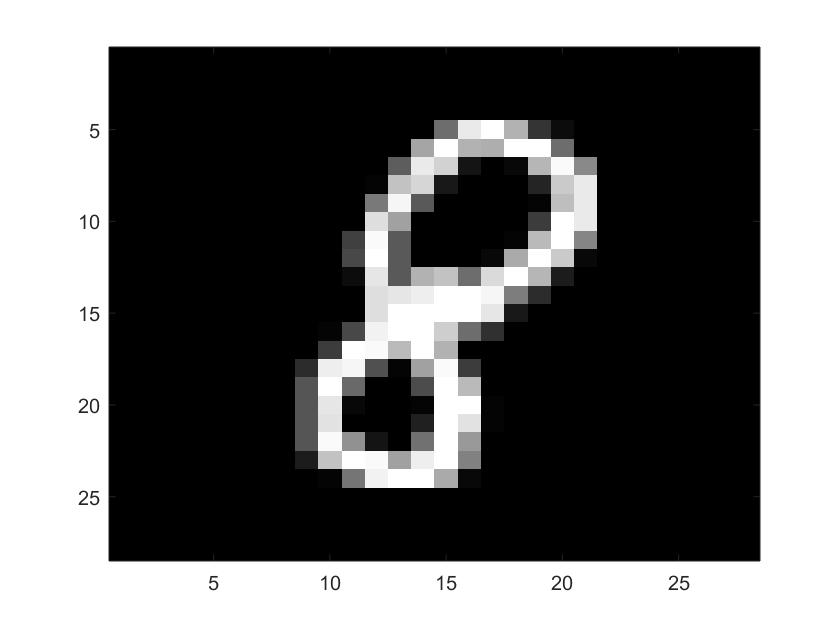}
\endminipage\hfill
\minipage{0.09\textwidth}
  \includegraphics[width=\linewidth]{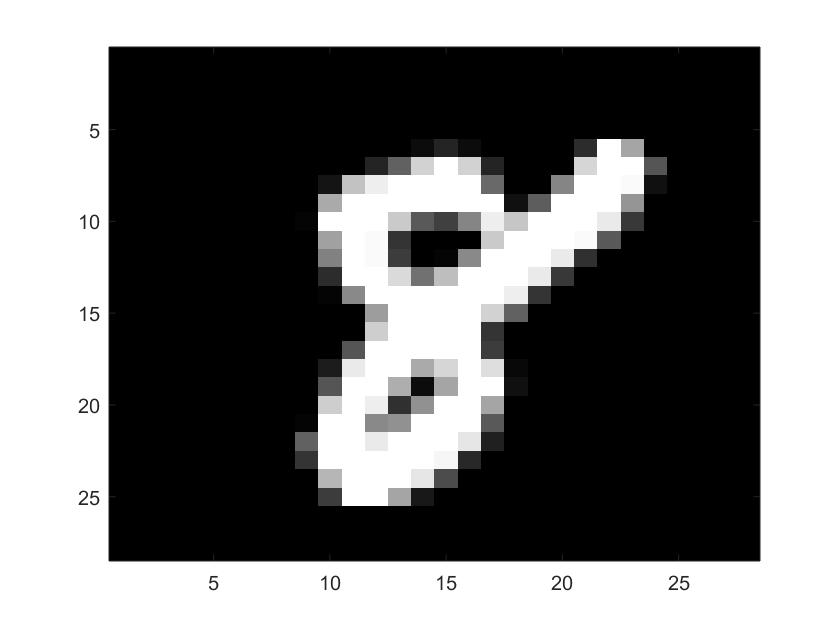}
\endminipage\hfill
\minipage{0.09\textwidth}
  \includegraphics[width=\linewidth]{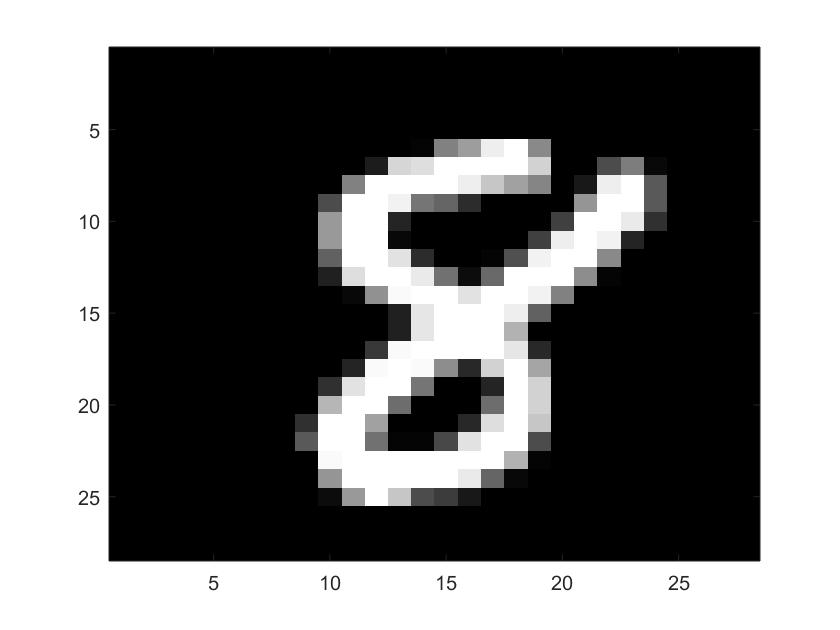}
\endminipage\hfill
\minipage{0.09\textwidth}
  \includegraphics[width=\linewidth]{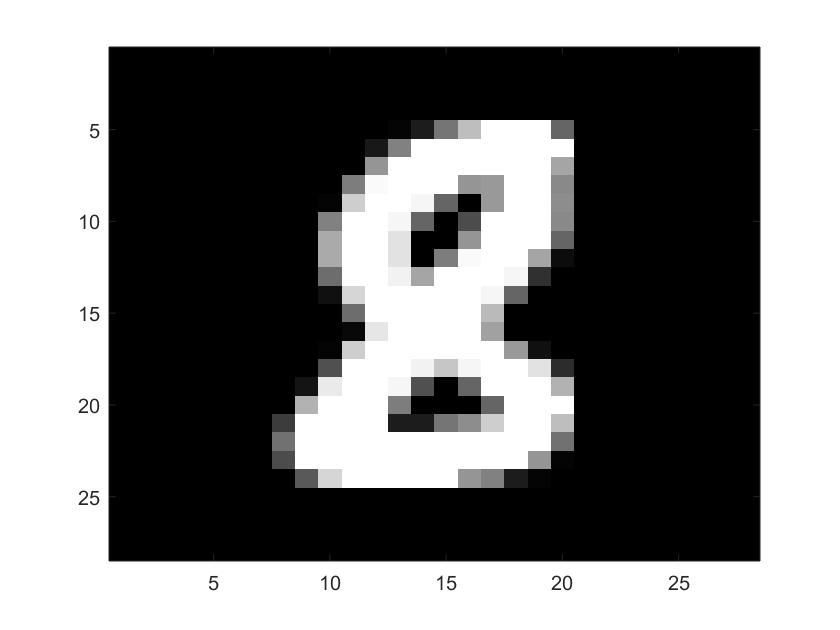}
\endminipage\hfill
\minipage{0.09\textwidth}
  \includegraphics[width=\linewidth]{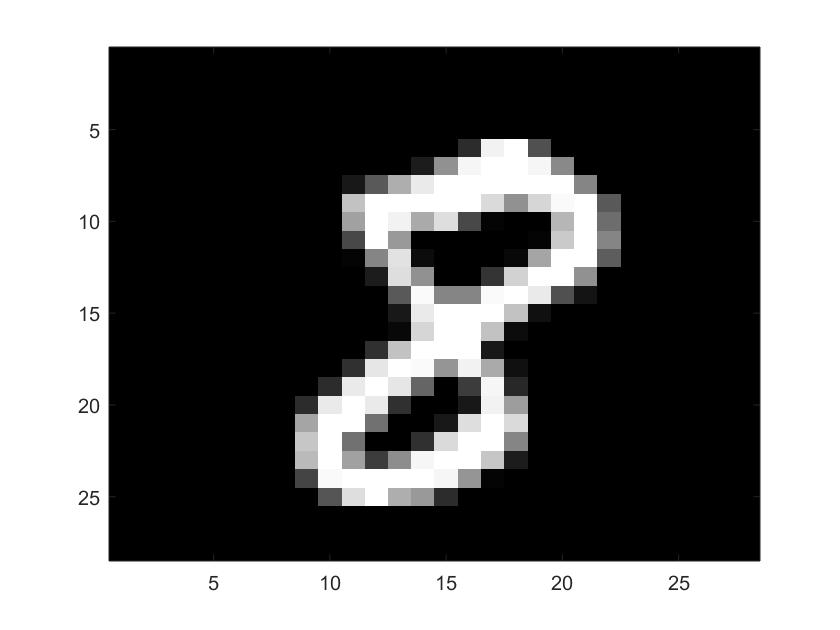}
\endminipage\hfill
\minipage{0.09\textwidth}
  \includegraphics[width=\linewidth]{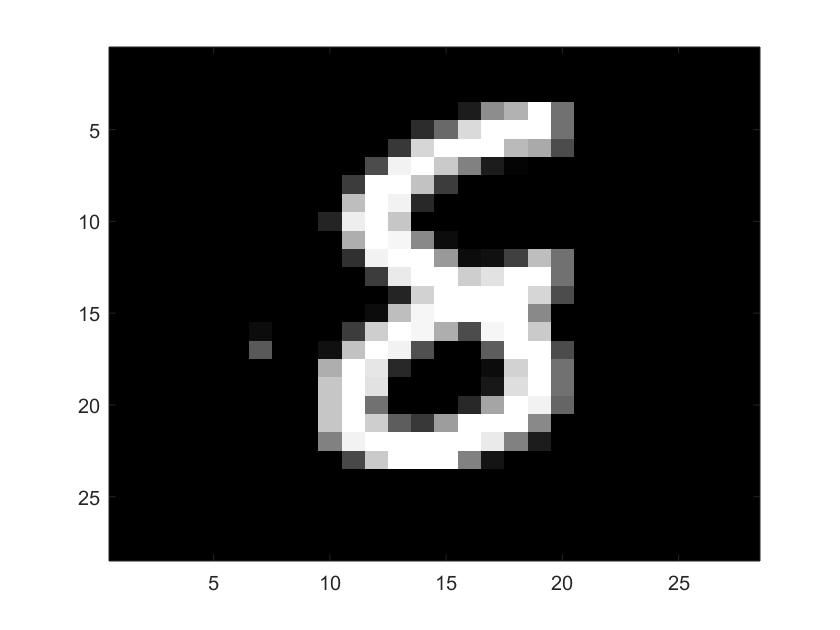}
\endminipage\hfill
\minipage{0.09\textwidth}
  \includegraphics[width=\linewidth]{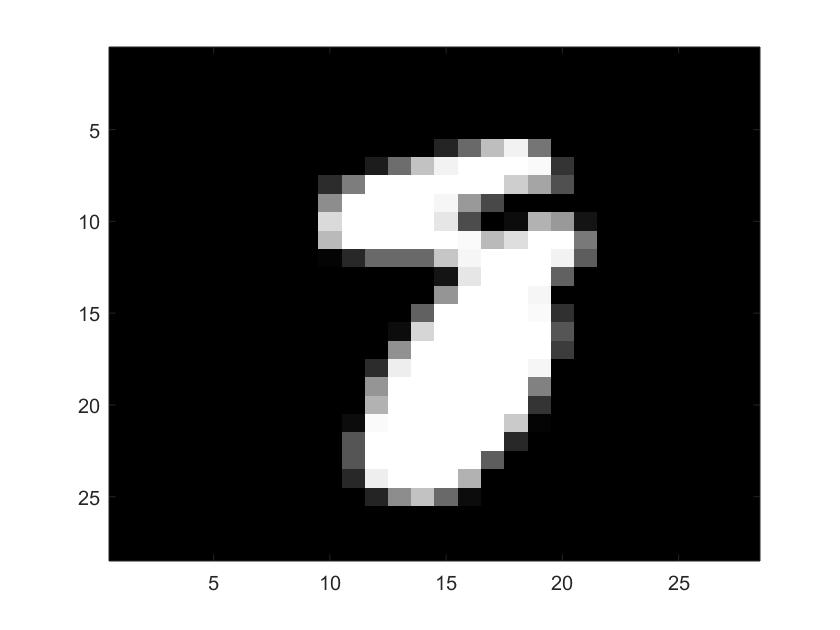}
\endminipage\hfill
\minipage{0.09\textwidth}
  \includegraphics[width=\linewidth]{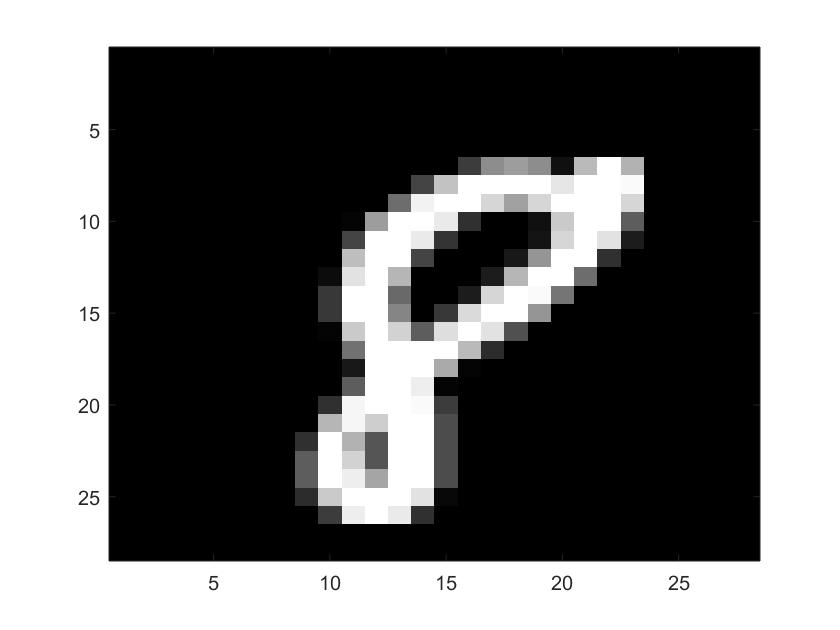}
\endminipage\hfill
\minipage{0.09\textwidth}
  \includegraphics[width=\linewidth]{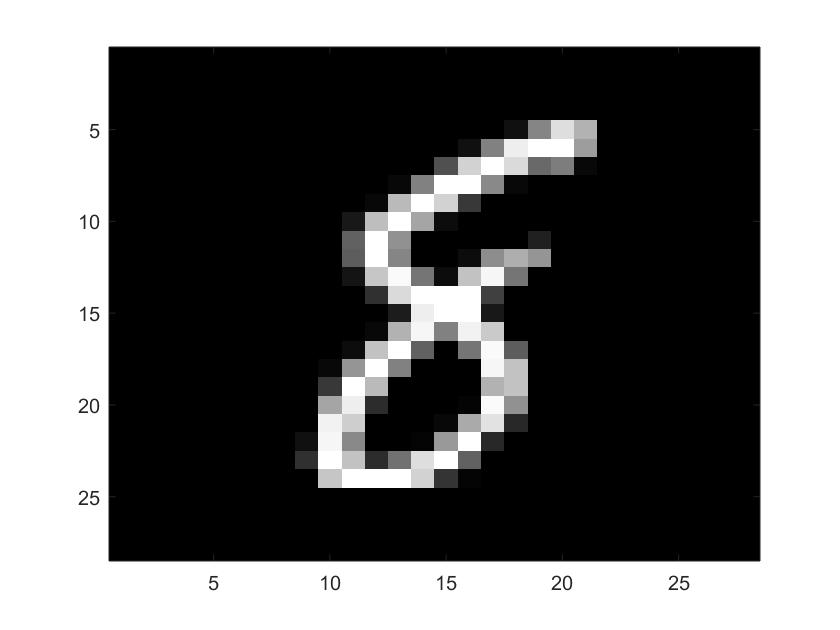}
\endminipage\hfill
\minipage{0.09\textwidth}
  \includegraphics[width=\linewidth]{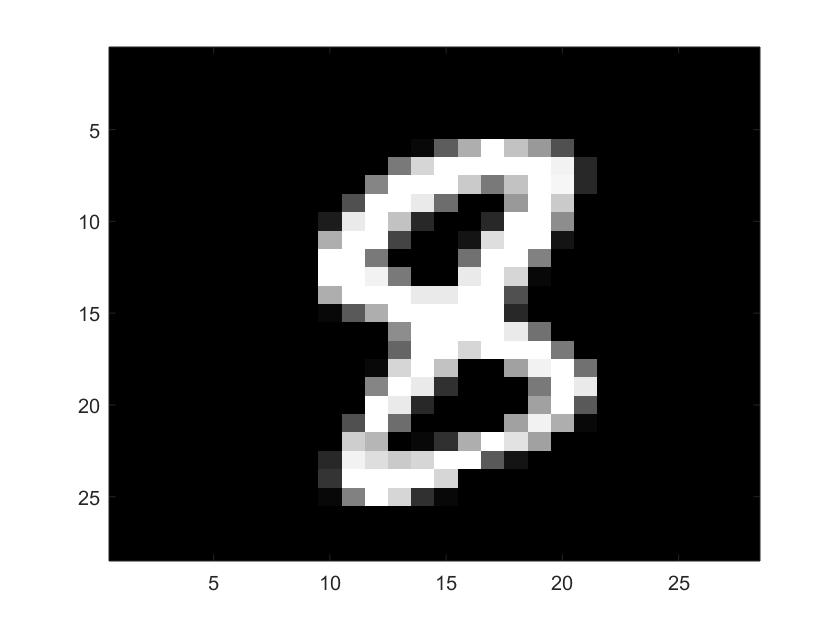}
\endminipage
\endminipage\hfill
   \minipage{0.48\textwidth}
    \minipage{0.09\textwidth}
  \includegraphics[width=\linewidth]{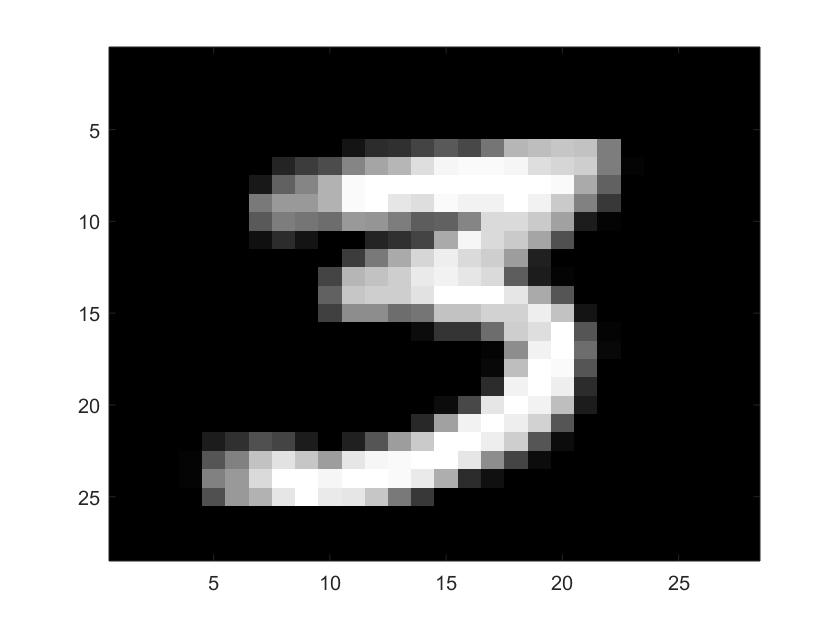}
\endminipage\hfill
\minipage{0.09\textwidth}
  \includegraphics[width=\linewidth]{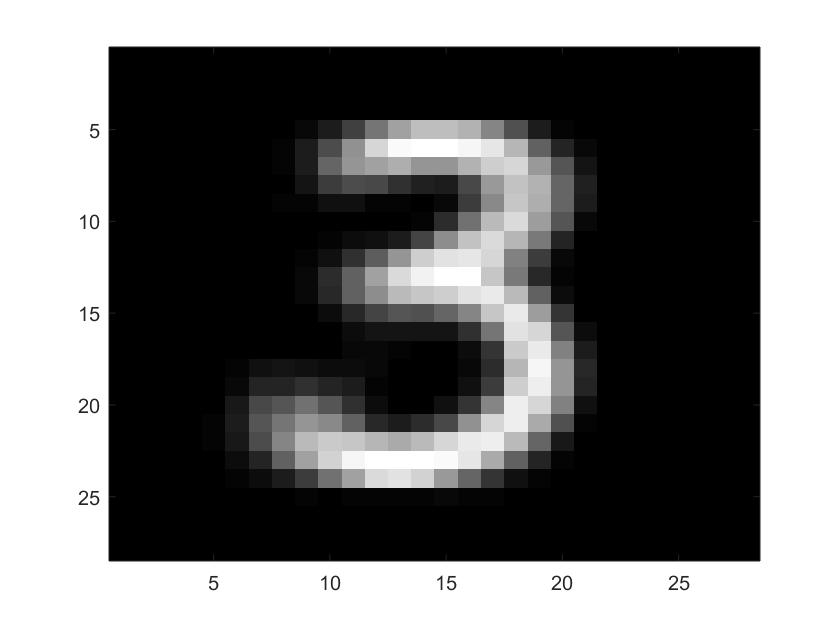}
\endminipage\hfill
\minipage{0.09\textwidth}
  \includegraphics[width=\linewidth]{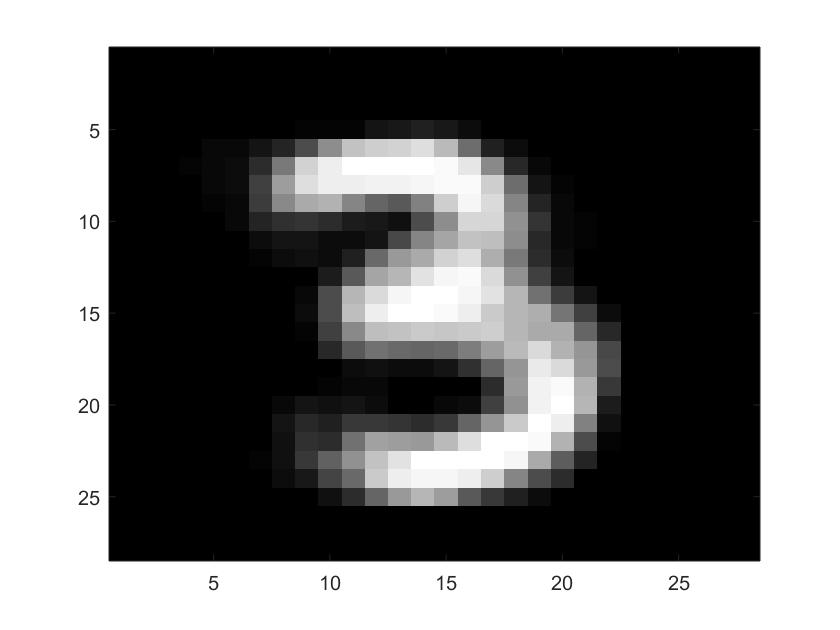}
\endminipage\hfill
\minipage{0.09\textwidth}
  \includegraphics[width=\linewidth]{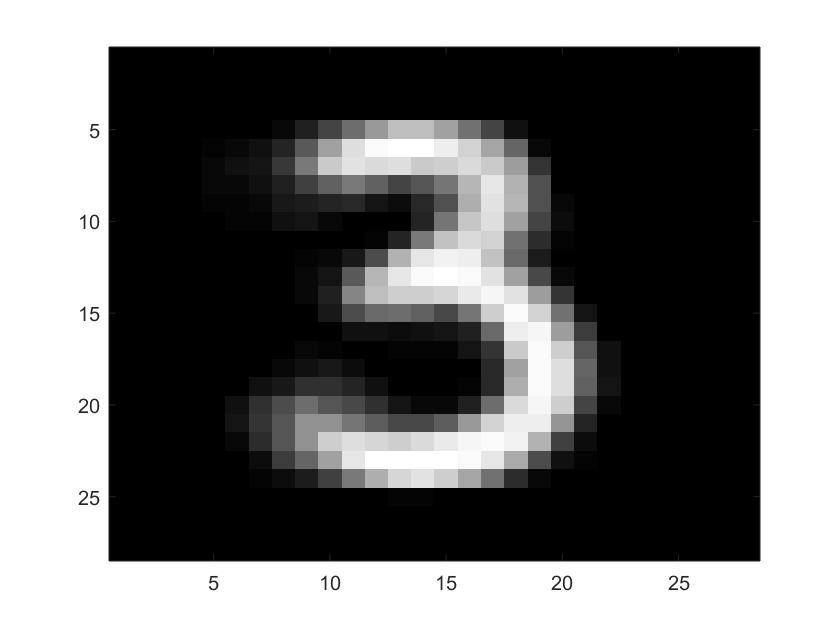}
\endminipage\hfill
\minipage{0.09\textwidth}
  \includegraphics[width=\linewidth]{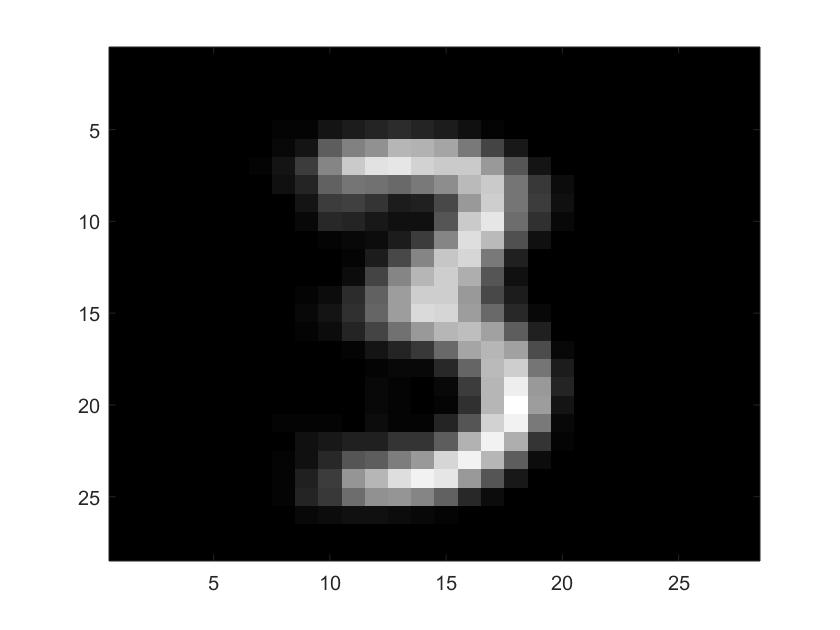}
\endminipage\hfill
\minipage{0.09\textwidth}
  \includegraphics[width=\linewidth]{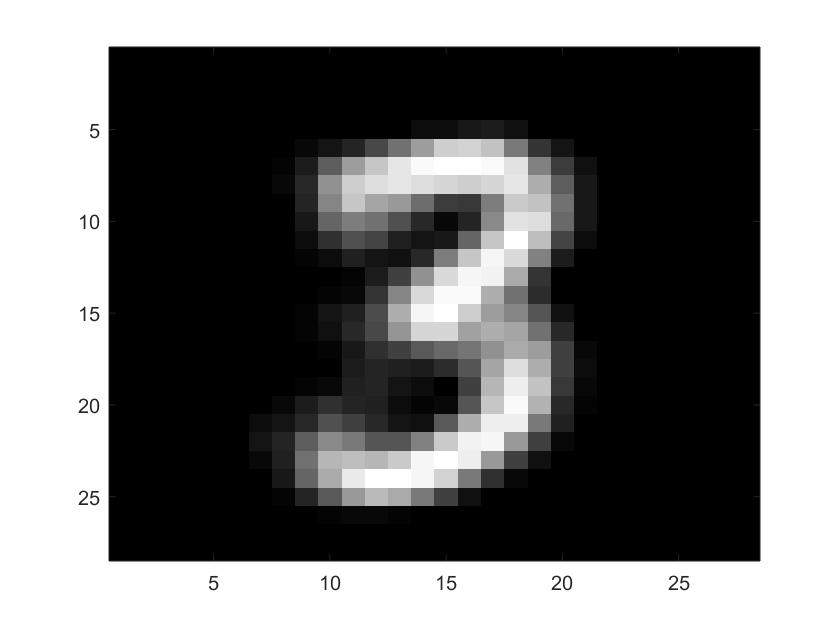}
\endminipage\hfill
\minipage{0.09\textwidth}
  \includegraphics[width=\linewidth]{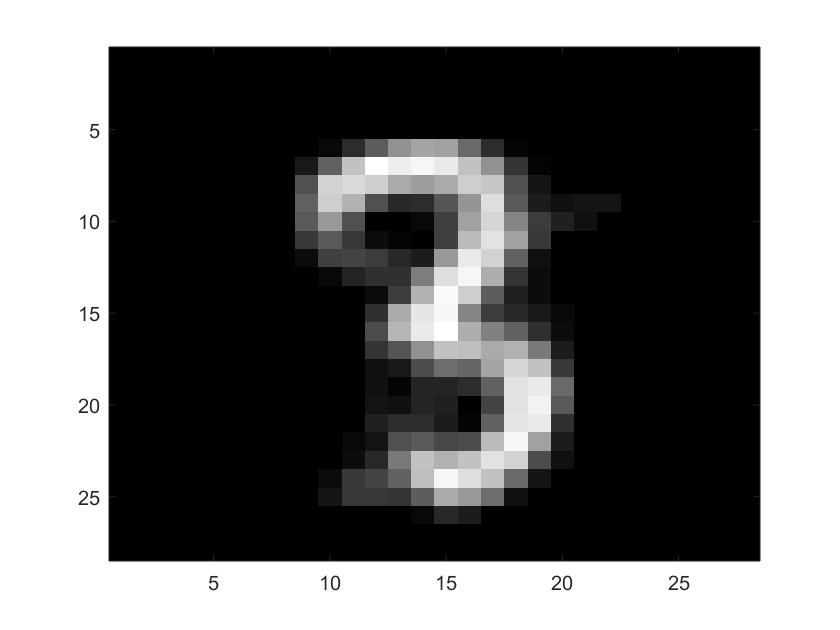}
\endminipage\hfill
\minipage{0.09\textwidth}
  \includegraphics[width=\linewidth]{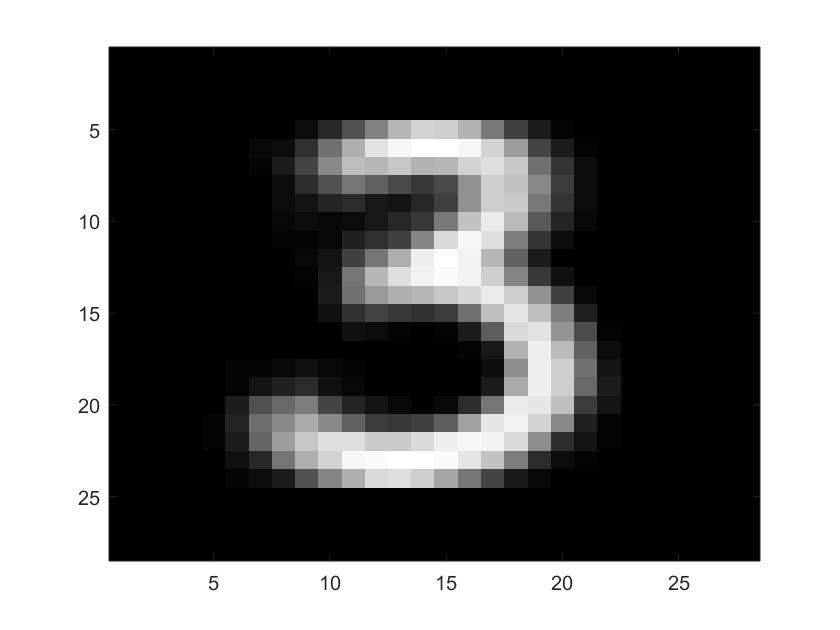}
\endminipage\hfill
\minipage{0.09\textwidth}
  \includegraphics[width=\linewidth]{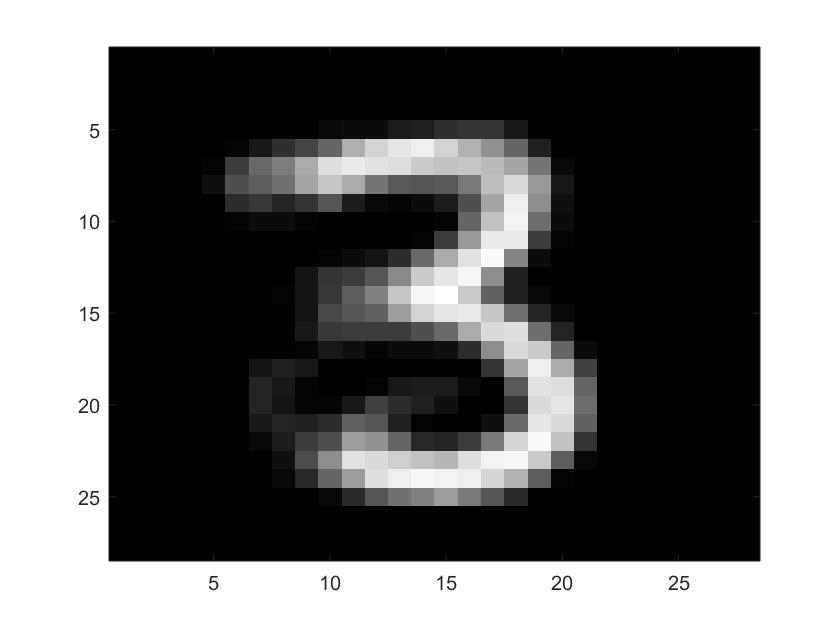}
\endminipage\hfill
\minipage{0.09\textwidth}
  \includegraphics[width=\linewidth]{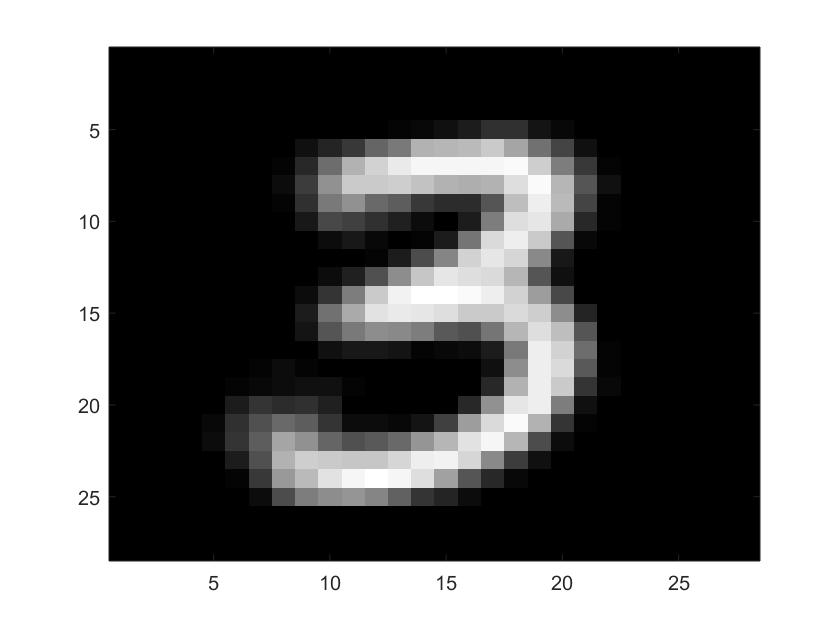}
\endminipage
\subcaption{Threes in MNIST.}
\endminipage\hfill
    \minipage{0.48\textwidth}
    \minipage{0.09\textwidth}
  \includegraphics[width=\linewidth]{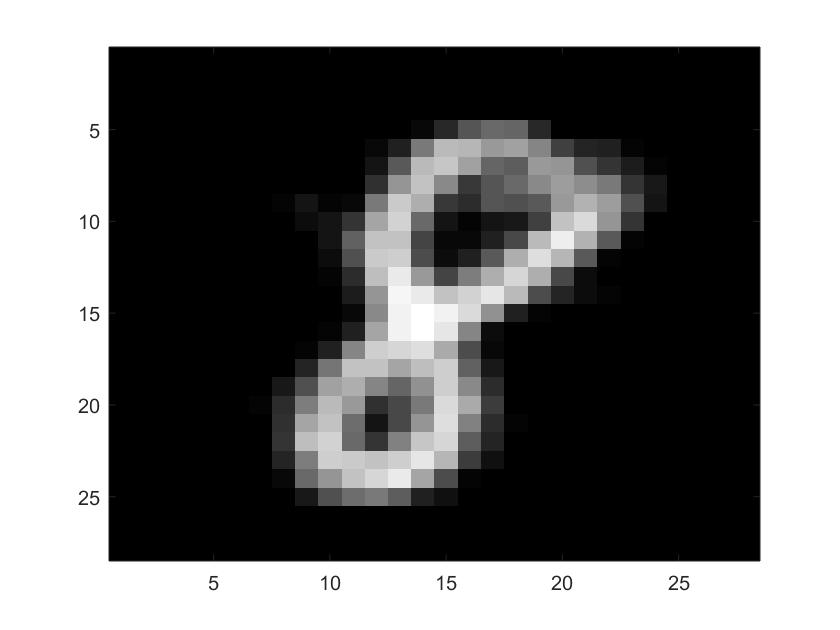}
\endminipage\hfill
\minipage{0.09\textwidth}
  \includegraphics[width=\linewidth]{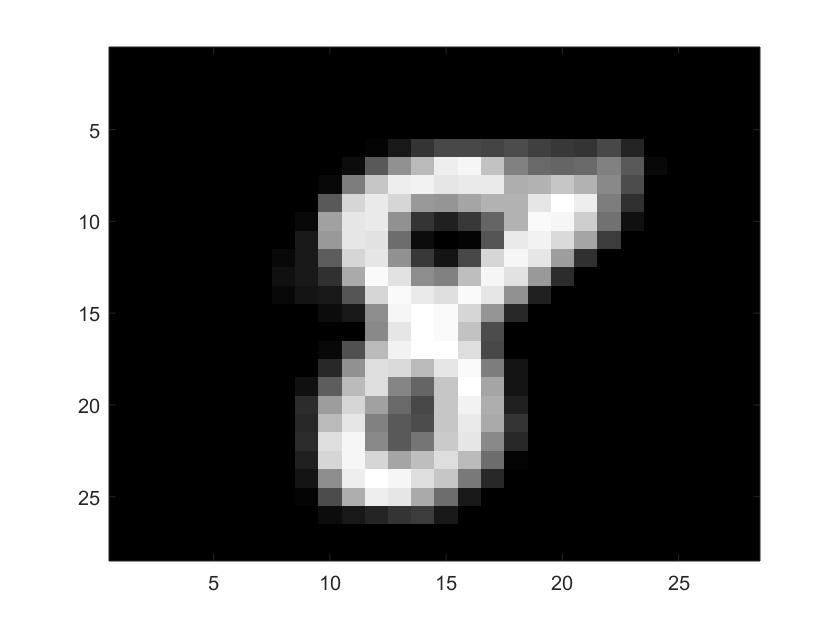}
\endminipage\hfill
\minipage{0.09\textwidth}
  \includegraphics[width=\linewidth]{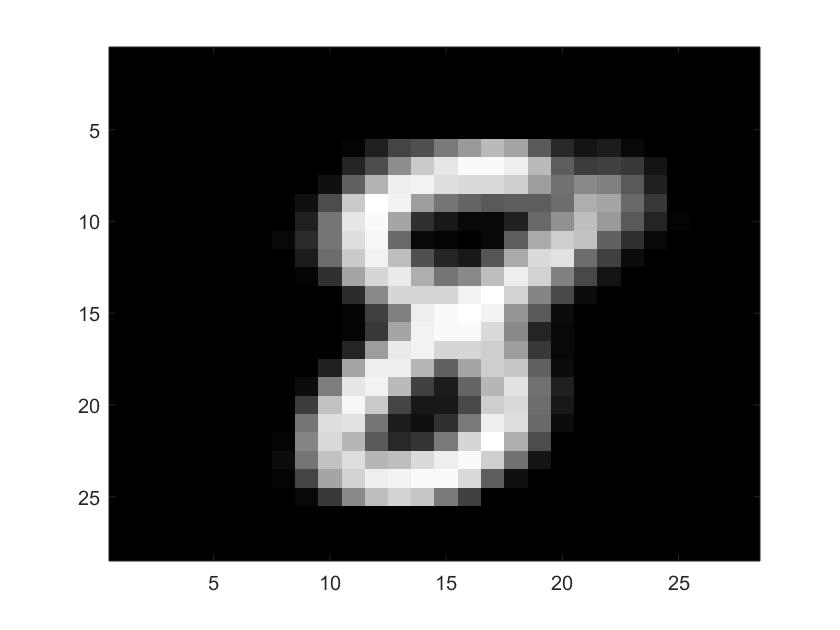}
\endminipage\hfill
\minipage{0.09\textwidth}
  \includegraphics[width=\linewidth]{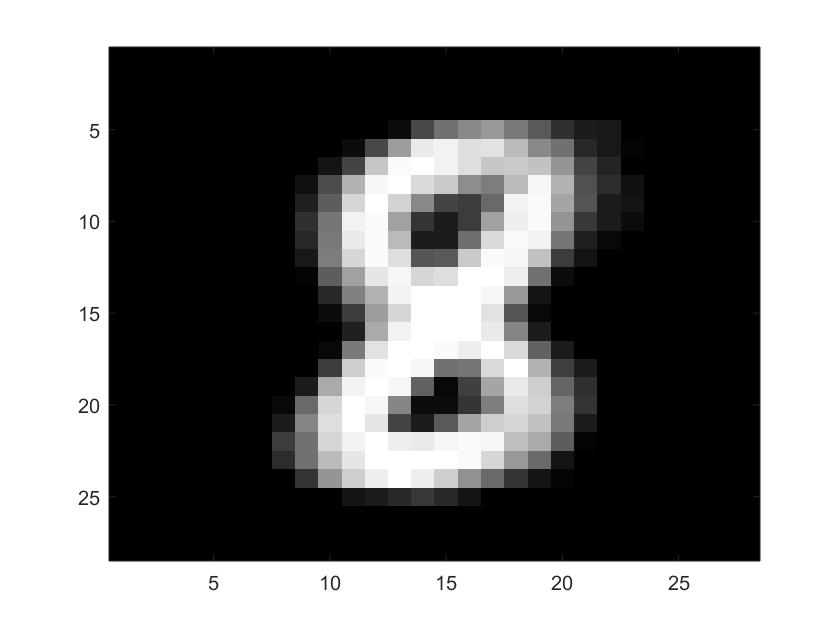}
\endminipage\hfill
\minipage{0.09\textwidth}
  \includegraphics[width=\linewidth]{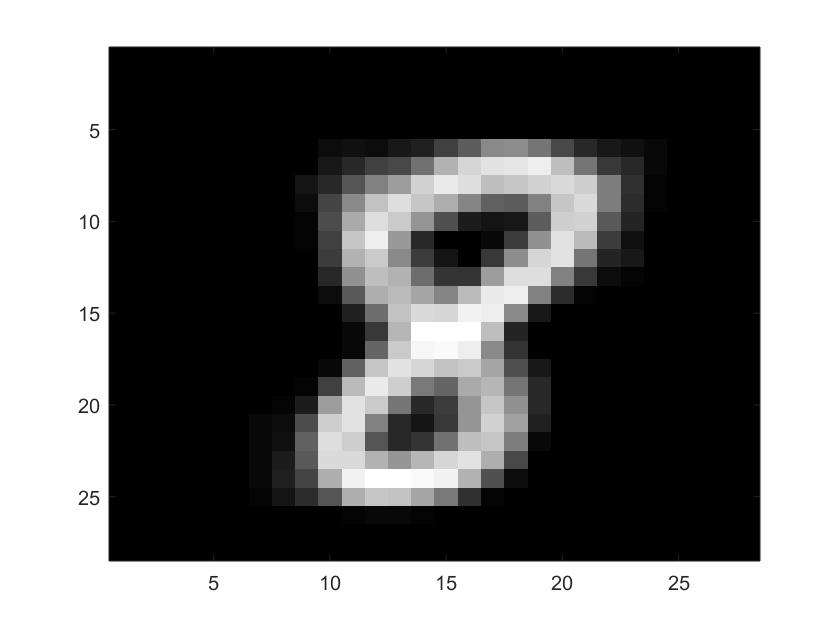}
\endminipage\hfill
\minipage{0.09\textwidth}
  \includegraphics[width=\linewidth]{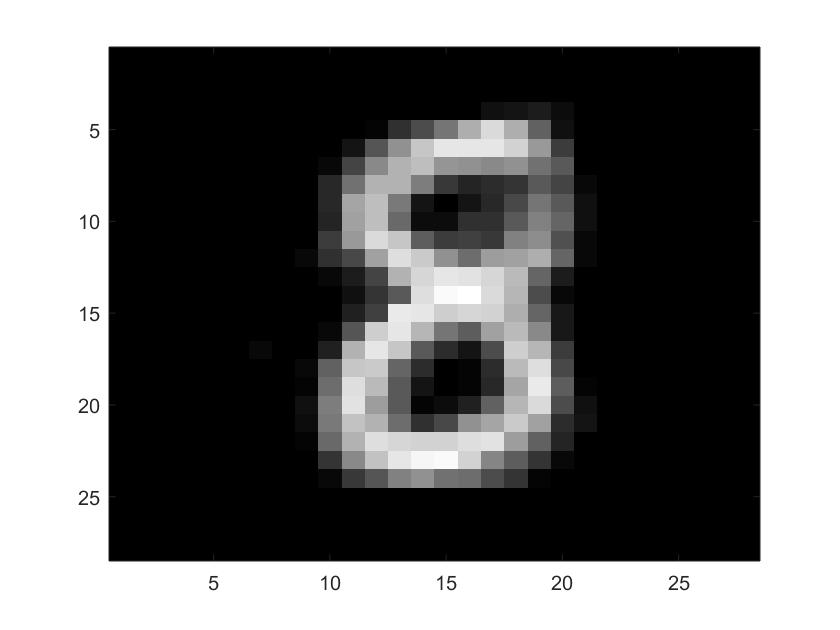}
\endminipage\hfill
\minipage{0.09\textwidth}
  \includegraphics[width=\linewidth]{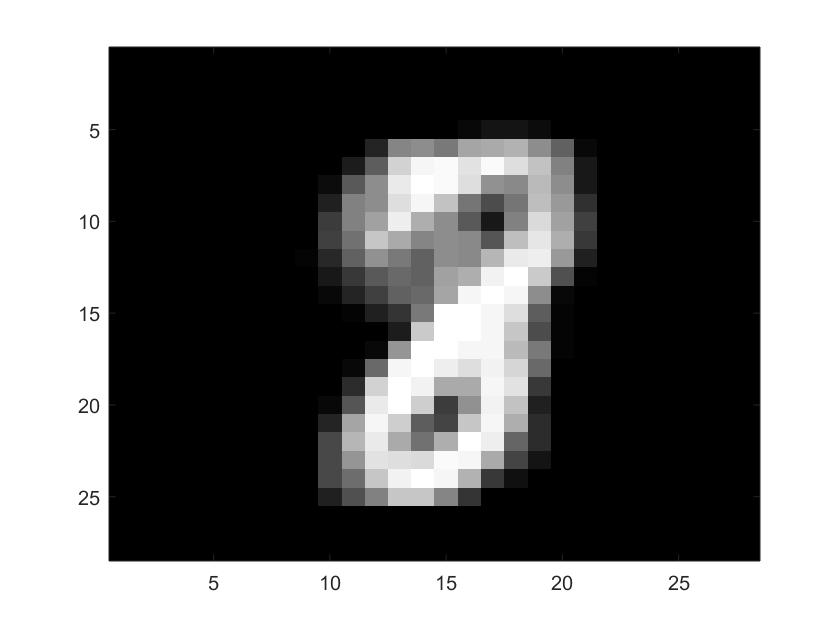}
\endminipage\hfill
\minipage{0.09\textwidth}
  \includegraphics[width=\linewidth]{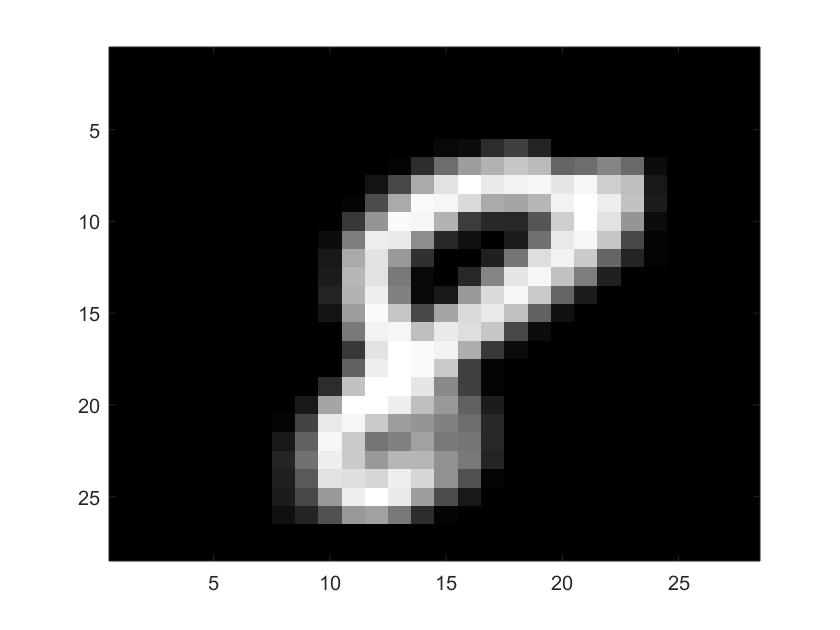}
\endminipage\hfill
\minipage{0.09\textwidth}
  \includegraphics[width=\linewidth]{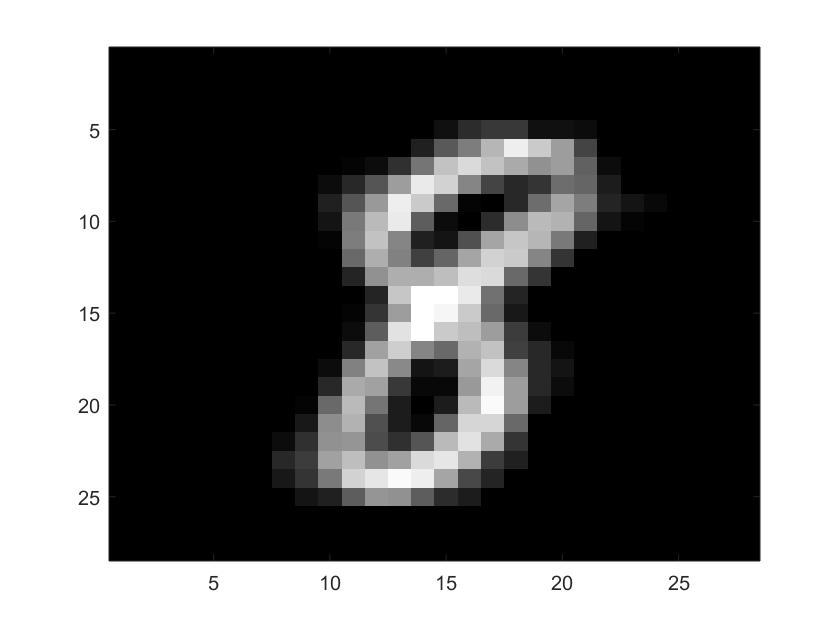}
\endminipage\hfill
\minipage{0.09\textwidth}%
  \includegraphics[width=\linewidth]{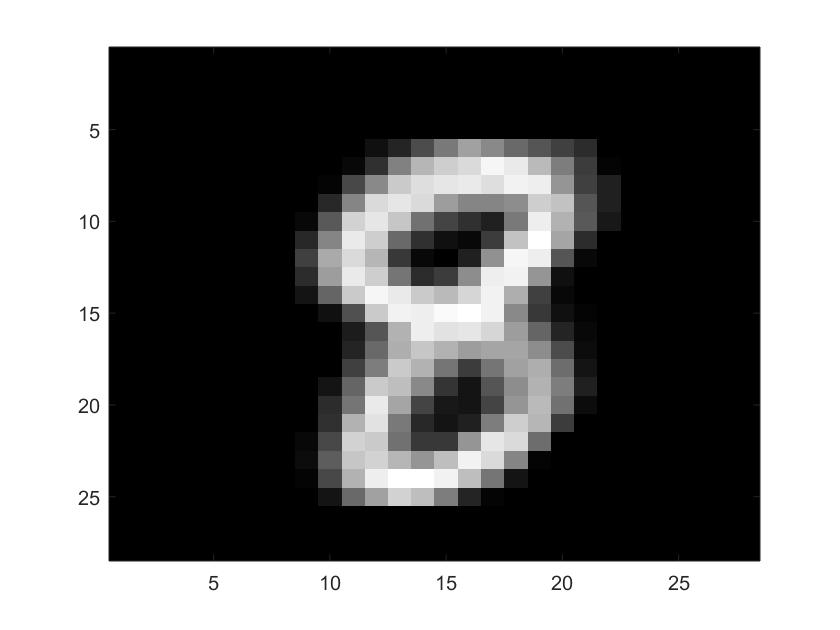}
\endminipage
\subcaption{Eights in MNIST.}
\endminipage\hfill
\caption{Visualization of the regularization effects. The second row is the regularized version of the corresponding image in the first row. While arguably more blurred, the digits in the second row are more homogeneous within each group, making classification easier.}
\label{figure:3}
\end{figure}

 As before, the $K$-NN variant performs much better than the fully-connected graph. As in Table \ref{table:4}, we see that except for the pair 3\&8, the classification error for the other three pairs with $\Gamma_{\Y_n}$ is 480: after respecting the 40 labels, the other 960 images are classified as part of the same group. However, after regularization, the classification error is greatly reduced with $\Gamma_{\bar{\Y}_n}$.  The same is true when we use the $K$-NN variant, but the improvement is smaller. Similarly as in Table \ref{table:5}, the improvement for local regularization becomes less dramatic as we go from the fully-connect graph to its $K$-NN variant and as the number of labels increases.  On the one hand, this implies that there is certainly a limit for the improvement that local regularization can provide. On the other hand, it also suggests that local regularization is most effective when information is limited. More specifically, we can interpret the fully-connected graph as having less geometric information than its $K$-NN variant because the similarity between $x_i$ and $x_j$ being $0$ tells us more than if the similarity is small but nonzero.  Also, in the case of more labels, we certainly have more information on the point cloud. So when there is little prior knowledge of labels and similarities, one has to rely on the geometry of the data set. Our theory and our experiments show that local regularization improves the recovery of geometric information and thereby boosts the classification performance in that scenario. We present a visualization of the effect of local regularization in Figure 3. The two rows represent the image before and after local regularization respectively. We can see that especially for the eights, many of the images get ``fixed" after regularization. Moreover, at a high level, images within each group in the second row look more similar among themselves than those in the first row. Because of this we expect the classification to be better.

\begin{remark}
The four chosen pairs of digits are the most difficult pairs to classify. Local regularization can improve the performance for other pairs too, but since the overall error rates for the other pairs is smaller, cross-validation can be harder. For unsupervised spectral clustering, local regularization still gives improvement, but using cross validation to choose $r$ is no longer possible. 
\end{remark}
   
  As mentioned  above, the practical choice of $r$ can be challenging.  We propose two alternatives that may be easier to work with and investigate their competence on the MNIST data set. \par  
     \begin{itemize}
         \item $k$-NN regularization: This is a natural variant of $\Gamma_{\bar{\Y}_n}$ based on $k$-nearest neighbor regularization.  Instead of specifying a neighborhood of $y_i$ of radius $r$, we simply regress the data by averaging over its $k$ nearest neighbors. Here $k$ is not necessarily the same as $K$ (the number of neighbors used to construct a similarity graph). Conceptually, choosing $k$ amounts to setting different values of $r$ at different points in such a way that the resulting neighborhoods contain roughly the same number of points. This construction is easier to work with since $k$ is in general easier to tune than $r$.    \par
        
        \end{itemize}
        \begin{itemize}
            \item Self-tuning: This is a global regularization variant that does not require hyper parameters. Instead of averaging over a neighborhood of radius $r$, we take a global weighted average of the whole point cloud, where the weights are proportional to the similarities between the $y_i$. More specifically, we define a new distance in terms of the points $\hat{y}_i$, where  
        \begin{align*}
            \hat{y}_i=\sum_{j=1}^n W(i,j) y_j. 
        \end{align*}
        We see that points far from $y_i$ have small contribution in the definition of $\hat{y}_i$ and so essentially one ends up summing over points in a neighborhood that is implicitly specified by the similarities. For points close to $y_i$, the weights are roughly on the same order. Hence $\hat{y}_i$ can be seen approximately as $\bar{y}_i$ plus a small contribution from points that are far from $y_i$. We expect this construction to behave a little worse than the $\Gamma_{\bar{\Y}_n}$ with optimal $r$. However, the fact that this construction does not require the tuning of any hyper-parameter makes it an appealing choice. Table 6 compares the classification performance of all graphs mentioned above (with the four different choices of distance function, and the two alternatives to build similarity graphs). 
         \end{itemize}

\begin{table}[h!]
\centering
\begin{tabular}{ |c|c|c|c|c|c|c|c|c|c|c| } 
 \hline
  Fully-connected & 3\&8  & 5\&8 &  4\&9 & 7\&9   & $K$-NN variant & 3\&8 & 5\&8 & 4\&9 & 7\&9\\
 \hline
 $\Gamma_{\Y_n}$        & 277&480 & 480& 480&   $\Gamma_{\Y_n}$   & 76& 55 &128 &73\\ 
 \hline
  $\Gamma_{\bar{\Y}_n}$ &134 &174 & 369& 153&  $\Gamma_{\bar{\Y}_n}$&69 &36 &97 &54\\ 
 \hline
 $k$-NN regularization  & 115& 74&431 & 183& $k$-NN regularization &53 &59 &96 &61\\ 
 \hline
  self-tuning graph &161 & 139& 334& 263&   self-tuning graph& 76&31 &88 &56\\ 
 \hline
\end{tabular}
\caption{Comparison of classification errors with 4\% labeled data.}
\label{table:6}
\end{table}

\begin{remark}
The idea of using labels to learn $r$ (or $k$) can be understood as a specific instance of a more general idea: to use labels to better inform the learning of the underlying geometry of a data set. What is more, one can try to simultaneously learn the geometry of the input space with the learning of the labeling function, instead of looking at these two problems in sequential form. This will be the topic of future research.     
\end{remark}

\bibliographystyle{plain}
	\bibliography{isbib}

\begin{thebibliography}{10}

\bibitem{agapiou2015importance}
S.~Agapiou, O.~Papaspiliopoulos, D.~Sanz-Alonso, and A.~M. Stuart.
\newblock Importance sampling: Intrinsic dimension and computational cost.
\newblock {\em Statistical Science}, 32(3):405--431, 2017.

\bibitem{belkin2004semi}
M.~Belkin and P.~Niyogi.
\newblock {Semi-supervised learning on Riemannian manifolds}.
\newblock {\em Machine learning}, 56(1-3):209--239, 2004.

\bibitem{belkin2006manifold}
M.~Belkin, P.~Niyogi, and V.~Sindhwani.
\newblock {Manifold regularization: A geometric framework for learning from
  labeled and unlabeled examples}.
\newblock {\em Journal of machine learning research}, 7(Nov):2399--2434, 2006.

\bibitem{bertozzi2018uncertainty}
A.~L. Bertozzi, X.~Luo, A.~M. Stuart, and K.~C. Zygalakis.
\newblock Uncertainty quantification in graph-based classification of high
  dimensional data.
\newblock {\em SIAM/ASA Journal on Uncertainty Quantification}, 6(2):568--595,
  2018.

\bibitem{burago2013graph}
D.~Burago, S.~Ivanov, and Y.~Kurylev.
\newblock {A graph discretization of the Laplace-Beltrami operator}.
\newblock {\em arXiv preprint arXiv:1301.2222}, 2013.

\bibitem{chen2016comprehensive}
Y.-C. Chen, C.~R. Genovese, L.~Wasserman, et~al.
\newblock A comprehensive approach to mode clustering.
\newblock {\em Electronic Journal of Statistics}, 10(1):210--241, 2016.

\bibitem{coifman2006diffusion}
R.~R. Coifman and S.~Lafon.
\newblock Diffusion maps.
\newblock {\em Applied and computational harmonic analysis}, 21(1):5--30, 2006.

\bibitem{docarmo1992riemannian}
M.~P. do~Carmo.
\newblock {\em Riemannian geometry}.
\newblock Mathematics: Theory \& Applications. Birkh\"{a}user Boston, Inc.,
  Boston, MA, 1992.
\newblock Translated from the second Portuguese edition by Francis Flaherty.

\bibitem{fukunaga1975estimation}
K.~Fukunaga and L.~Hostetler.
\newblock The estimation of the gradient of a density function, with
  applications in pattern recognition.
\newblock {\em IEEE Transactions on information theory}, 21(1):32--40, 1975.

\bibitem{SpecRatesTrillos}
N.~Garc{\'\i}a~Trillos, M.~Gerlach, M.~Hein, and D.~Slep{\v{c}}ev.
\newblock {Spectral convergence of empirical graph Laplacians}.
\newblock {\em Preprinr}, 2018.

\bibitem{trillos2017consistency}
N.~Garcia~Trillos, Z.~Kaplan, T.~Samakhoana, and D.~Sanz-Alonso.
\newblock {On the consistency of graph-based Bayesian learning and the
  scalability of sampling algorithms}.
\newblock {\em arXiv preprint arXiv:1710.07702}, 2017.

\bibitem{garcia2018continuum}
N.~Garcia~Trillos and D.~Sanz-Alonso.
\newblock Continuum limits of posteriors in graph bayesian inverse problems.
\newblock {\em SIAM Journal on Mathematical Analysis}, 50(4):4020--4040, 2018.

\bibitem{trillos}
N.~Garc{\'\i}a~Trillos and D.~Slep{\v{c}}ev.
\newblock Continuum limit of total variation on point clouds.
\newblock {\em Archive for rational mechanics and analysis}, 220(1):193--241,
  2016.

\bibitem{haddad2014texture}
A.~Haddad, D.~Kushnir, and R.~R. Coifman.
\newblock Texture separation via a reference set.
\newblock {\em Applied and Computational Harmonic Analysis}, 36(2):335--347,
  2014.

\bibitem{little2017multiscale}
A.~V. Little and L.~Maggioni, M.and~Rosasco.
\newblock Multiscale geometric methods for data sets i: Multiscale svd, noise
  and curvature.
\newblock {\em Applied and Computational Harmonic Analysis}, 43(3):504--567,
  2017.

\bibitem{ng2002spectral}
A.~Y. Ng, M.~I. Jordan, and Y.~Weiss.
\newblock On spectral clustering: Analysis and an algorithm.
\newblock In {\em Advances in neural information processing systems}, pages
  849--856, 2002.

\bibitem{niyogi2008finding}
P.~Niyogi, S.~Smale, and S.~Weinberger.
\newblock Finding the homology of submanifolds with high confidence from random
  samples.
\newblock {\em Discrete \& Computational Geometry}, 39(1-3):419--441, 2008.

\bibitem{Pinelis}
I.~Pinelis.
\newblock An approach to inequalities for the distributions of
  infinite-dimensional martingales.
\newblock In {\em Probability in {B}anach spaces, 8 ({B}runswick, {ME}, 1991)},
  volume~30 of {\em Progr. Probab.}, pages 128--134. Birkh\"{a}user Boston,
  Boston, MA, 1992.

\bibitem{shi2000normalized}
J.~Shi and J.~Malik.
\newblock Normalized cuts and image segmentation.
\newblock {\em Departmental Papers (CIS)}, page 107, 2000.

\bibitem{singer2006graph}
A.~Singer.
\newblock From graph to manifold laplacian: The convergence rate.
\newblock {\em Applied and Computational Harmonic Analysis}, 21(1):128--134,
  2006.

\bibitem{spielman2007spectral}
D.~A Spielman and S.-H. Teng.
\newblock Spectral partitioning works: Planar graphs and finite element meshes.
\newblock {\em Linear Algebra and its Applications}, 421(2-3):284--305, 2007.

\bibitem{tukey1988computer}
J.~W. Tukey and P.~A. Tukey.
\newblock Computer graphics and exploratory data analysis: An introduction.
\newblock {\em The Collected Works of John W. Tukey: Graphics: 1965-1985},
  5:419, 1988.

\bibitem{von2007tutorial}
U.~Von~Luxburg.
\newblock A tutorial on spectral clustering.
\newblock {\em Statistics and computing}, 17(4):395--416, 2007.

\bibitem{von2008consistency}
U.~Von~Luxburg, M.~Belkin, and O.~Bousquet.
\newblock Consistency of spectral clustering.
\newblock {\em The Annals of Statistics}, pages 555--586, 2008.

\bibitem{weed2017sharp}
J.~Weed and F.~Bach.
\newblock {Sharp asymptotic and finite-sample rates of convergence of empirical
  measures in Wasserstein distance}.
\newblock {\em arXiv preprint arXiv:1707.00087}, 2017.

\bibitem{zelnik2005self}
L.~Zelnik-Manor and P.~Perona.
\newblock Self-tuning spectral clustering.
\newblock In {\em Advances in neural information processing systems}, pages
  1601--1608, 2005.

\bibitem{zhou2005regularization}
D.~Zhou and B.~Sch{\"o}lkopf.
\newblock Regularization on discrete spaces.
\newblock In {\em Joint Pattern Recognition Symposium}, pages 361--368.
  Springer, 2005.

\bibitem{zhu2005semi}
X.~Zhu.
\newblock Semi-supervised learning literature survey.
\newblock 2005.

\end{thebibliography}

\appendix
\label{appendix:1}

\nc

\section{Estimating $r_- $ and $r_+$}
\label{sec:r+-}

\subsection{Estimating $r_-$} 

We want to find values of $t>\frac{r}{2}$ for which for all $v \in T_{x_i } \M$ with $|v| \leq t$, and for all $\eta \in T_{\exp_{x_i}(v) } \M ^\perp$ with $|\eta | \leq \sigma$ we have  
\[   \lvert \exp_{x_i}(v) +\eta - y_i   \rvert < r. \]
We will later take the maximum value of $t$ for which this holds and set $r_-$ to be this maximum value. 

Let $x= \exp_{x_i}(v)$. First, with the parallel transport map used in the proof of the geometric bias estimates (as in \eqref{eqn:etaetatilde}) we can associate a vector $\tilde \eta \in T_{x_i}\M ^\perp$ to a vector $\eta \in T_{x}\M ^\perp$ with norm less than $\sigma$, for which 
\[ |\eta - \tilde{\eta} | \leq \frac{m}{R}\sigma t.\]
Now, 
\begin{align*}
\begin{split}
 \lvert x + \eta - y_i   \rvert & \leq  \lvert x - x_i  + \hat{\eta} - z_i  \rvert +  |\eta - \hat{\eta}  |
 \\ &  =  \left( \lvert x - x_i \rvert^2   +        2 \langle x - x_i    ,\hat{\eta} - z_i  \rangle +  |\hat{\eta}- z_i|^2    \rvert \right)^{1/2} +  |\eta - \hat{\eta}  |
 \\& \leq \left( \lvert x - x_i \rvert^2   +        2 \langle x - x_i    ,\hat{\eta} - z_i  \rangle +  |\hat{\eta}- z_i|^2    \rvert \right)^{1/2} + \frac{m}{R}\sigma t.
 \end{split}
\end{align*}
We have 
\[ |x- x_i | \leq d_\M( x,x_i)  = |v| \leq t,  \]
and also
\[\langle x - x_i    ,\hat{\eta} - z_i  \rangle  =  \langle x- (x_i+ v)    , \hat{\eta} - z_i  \rangle,  \]
as it follows from the fact that $\eta, z_i \in T_{x_i}\M^\perp$ and $v \in T_{x_i} \M$. Using this, Cauchy-Schwartz, and \eqref{secondorderrep} we conclude that
\[  |  \langle x - (x_i+ v)    ,\hat{\eta} - z_i  \rangle   |  \leq  2\sigma |x - (x_i+ v)  |  \leq 2\sigma  \frac{|v|^2}{R}, \]
 and hence
 \begin{align*}
 \begin{split}
 \lvert x + \eta - y_i   \rvert & \leq \left( t^2 +\frac{4}{R}\sigma t^2 + 4\sigma^2  \right)^{1/2} +  \frac{m\sigma t }{R}
 \\& = t \left(   \sqrt{1 + \frac{4\sigma}{R}  +  \frac{4\sigma^2}{ t^2 }} + \frac{m\sigma  }{R} \right) 
 \\& \leq  t \left(   \sqrt{1 + \frac{4 \sigma }{R}+ \frac{16 \sigma^2}{ r^2} } + \frac{m\sigma  }{R} \right).
\end{split}
\end{align*}

From the above it follows that $r_-$ defined as
\begin{equation}
  r_-  :=  r   \left(   \sqrt{1 + \frac{4 \sigma}{R} + \frac{16 \sigma^2}{ r^2 }} + \frac{m\sigma  }{R} \right) ^{-1}  ,
\end{equation}
satisfies the desired properties and moreover
\begin{align}
\begin{split}
r - r_- & \leq r\left(1 -  \left(   \sqrt{1 + \frac{4}{R} \sigma + \frac{16 \sigma^2 }{r^2 }} + \frac{m\sigma}{R}   \right) ^{-1}    \right) .
\end{split}
  \label{eqn:rmenos} 
\end{align}

\subsection{Estimating $r_+$}
To estimate $r_+$, we need the following lemma proved in \cite{SpecRatesTrillos}.
\begin{lemma}
Suppose $x, \tilde x \in \M$ are such that $|x- \tilde{x}| \le R/2.$ Then
\begin{equation}
    |x- \tilde x| \leq d_\M(x, \tilde x) \leq |x- \tilde x| + \frac{8}{R} |x- \tilde x |^3. 
    \label{eqn:metriccomparisson} 
\end{equation}
\end{lemma}
To construct $r_+$ we find values of $t$ with $2r \geq r+\sigma >t>0$ such that if $|v|>t$ then
\[  | \exp_{x_i}(v) + \sigma \eta -y_i|  \geq r    \]
for all $\eta \in T_{\exp_{x_i}(v) } \M^\perp  $ of norm no larger than $\sigma$.

As in the construction of $r_-$ let $x:= \exp_{x_i}(v) $. Similar computations give
\begin{align*}
\begin{split}
\lvert x +  \eta - y_i   \rvert & \geq  \lvert x - x_i  + \hat{\eta} - z_i  \rvert -  |\eta - \hat{\eta}  |
\\ &  =  \left( \lvert x - x_i \rvert^2   +        2 \langle x - x_i    ,\hat{\eta} - z_i  \rangle +  |\hat{\eta}- z_i|^2    \rvert \right)^{1/2} -  |\eta - \hat{\eta}  |
\\& \geq \left( \lvert x  - x_i \rvert^2   +        2 \langle x - x_i    ,\hat{\eta} - z_i  \rangle +  |\hat{\eta}- z_i|^2    \rvert \right)^{1/2} - \frac{m}{R}\sigma |v|
\\&\geq \left( \Big(|v|- \frac{1}{R}|v|^3\Big)^2 -\frac{4}{R}\sigma |v|^2  \right)^{1/2} -  \frac{m}{R}\sigma |v| 
\\& \geq |v| \left(  \sqrt{ 1 - \frac{2|v|^2}{R} - \frac{4\sigma }{R}   }  -\frac{m\sigma }{R} \right)
\\& \geq |v|  \left(  \sqrt{ 1 - \frac{8r^2 }{R}- \frac{4\sigma }{R}   }  -\frac{m\sigma }{R}\right)
\\&\geq t  \left(  \sqrt{ 1 - \frac{8r^2 }{R}- \frac{4\sigma}{R}    }  -\frac{m\sigma}{R} \right),
\end{split}
\end{align*}
where in the third inequality we have used \eqref{eqn:metriccomparisson} to conclude that
\[ |\exp_{x_i}(v)- x_i | \geq d_\M( x,x_i) - C ( d_\M(x, x_i) )^3  = |v| - C |v|^3. \]
We can then take $t$ to be such that the right hand side of \eqref{eqn:rmenos} is equal to $r$. In other words we can take $r_+$ to be equal to 
\[  r_+  :=  r \left(  \sqrt{ 1 - \frac{8r^2 }{R}- \frac{4\sigma }{R}   }  -\frac{m\sigma}{R} \right)^{-1}. \]

From these estimates we see that
\begin{align*}
\begin{split}
 r_+ - r_-  & \le c \left( r^3 + r\sigma  + r \frac{\sigma^2}{r^2} \right).
\end{split}
\end{align*}

\end{document}